\documentclass[11pt]{article}
\pdfoutput=1

\usepackage{graphicx}
\usepackage{float}
\usepackage{hyperref}       
\usepackage{url}            
\usepackage{booktabs}       
\usepackage{amsfonts}       
\usepackage{nicefrac}       
\usepackage{microtype}      
\usepackage{xcolor}         
\usepackage{hyperref}       
\usepackage{url} 
\usepackage{diy}
\usepackage{enumitem}
\usepackage{subcaption}
\usepackage[english]{babel}

\usepackage{romannum}
\usepackage{cleveref}

\usepackage[margin=1.1in]{geometry}
\theoremstyle{plain}
\newtheorem{theorem}{Theorem}
\theoremstyle{plain}
\newtheorem{proposition}{Proposition}
\newtheorem{lemma}{Lemma}
\newtheorem{corollary}{Corollary}

\theoremstyle{definition}
\newtheorem{definition}{Definition}
\theoremstyle{plain}
\newtheorem{assumption}{Assumption}
\newtheorem{property}[theorem]{Property}
\theoremstyle{remark}
\newtheorem{remark}{Remark}

\newcommand{\Htr}{\mathbf{H}_{n_1,\btr}}
\newcommand{\Hte}{\mathbf{H}_{m,\bte}}
\newcommand{\rhob}{\overline{\brho}}
\newcommand{\Sb}{\mathbf{S}}

\definecolor{myorange}{RGB}{245,156,74}
\hypersetup{
	colorlinks=true,
	linkcolor=red,
	filecolor=blue,      
	urlcolor=red,
	citecolor=-myorange,
}
\usepackage{tcolorbox}
\usepackage[page, header]{appendix}

\title{Provable Generalization of Overparameterized Meta-learning Trained with SGD}
\author{Yu Huang\thanks{IIIS, Tsinghua University; e-mail: {\tt y-huang20@mails.tsinghua.edu.cn}.}
\and
Yingbin Liang\thanks{Department of ECE, The Ohio State University; e-mail: {\tt  liang889@osu.edu}.}
\and
Longbo Huang\thanks{IIIS, Tsinghua University; e-mail: {\tt longbohuang@tsinghua.edu.cn}.}
}
\date{
}
\begin{document}
\maketitle
\pagenumbering{arabic}
\begin{abstract}
   Despite the superior empirical success of deep meta-learning, theoretical understanding of overparameterized meta-learning is still limited. This paper studies the generalization of a widely used meta-learning approach, Model-Agnostic Meta-Learning (MAML), which aims to find a good initialization for fast adaptation to new tasks. 
 Under a mixed linear regression model, we analyze the generalization properties of MAML trained with SGD in the overparameterized regime. 
 We provide both upper and lower bounds for the excess risk of MAML, which captures how SGD dynamics affect these generalization bounds. With such sharp characterizations, we further explore how various learning parameters impact the generalization capability of  overparameterized MAML, including explicitly identifying typical data and task distributions that can achieve diminishing generalization error with overparameterization, and  characterizing the impact of adaptation learning rate on both excess risk and the early stopping time. Our theoretical findings are further validated by experiments. 
\end{abstract}

\section{Introduction} 
Meta-learning~\cite{hospedales2020meta} is a learning paradigm which aims to design algorithms that are capable of gaining knowledge from many previous tasks and then using it to improve the performance on future tasks efficiently. It has exhibited great power in various machine learning applications spanning over few-shot image classification~\cite{ren2018meta,rusu2018meta}, reinforcement learning~\cite{gupta2018meta} and intelligent medicine~\cite{gu2018meta}. 
 
 One prominent type of meta-learning approaches is an optimization-based method,  Model-Agnostic Meta-Learning (MAML)~\cite{finn2017model}, which achieves impressive results in different tasks~\cite{obamuyide2019model,bao2019few,antoniou2018train}. The idea of MAML is to learn a good initialization $\boldsymbol{\omega}^{*}$,  such that for a new task we can adapt quickly to 
 a good task parameter starting from $\boldsymbol{\omega}^{*}$. MAML takes a bi-level implementation: the inner-level initializes at the meta parameter and takes task-specific updates using a few steps of gradient descent (GD), and the outer-level optimizes the meta parameter across all tasks. 
 
 With the superior empirical success, theoretical justifications have been provided for MAML and its variants over the past few years from both optimization~\cite{finn2019online,wang2020globala,fallah2020convergence,ji2022theoretical} and generalization perspectives~\cite{amit2018meta,denevi2019learning,fallah2021generalization,chen2021generalization}. However, most existing analyses 
 did not take overparameterization into consideration, which we deem  as crucial to demystify the remarkable generalization  ability of deep meta-learning~\cite{zhang2021understanding, hospedales2020meta}. More recently, \cite{wang2020globalb} studied the MAML with overparameterized deep neural nets and derived a complexity-based bound to quantify the difference between the empirical and population loss functions at their optimal solutions. However, complexity-based generalization bounds tend to be weak in the high dimensional,  especially in the overparameterized regime. Recent  works~\cite{bernacchia2021meta,zou2021unraveling} developed more precise bounds for overparameterized setting under a mixed linear regression model, and identified the effect of adaptation learning rate on the generalization. Yet, they considered only the simple isotropic covariance for data and tasks, and did not explicitly capture how the generalization performance of MAML depends on the data and task distributions. 
 Therefore, the following important problem still remains largely open:
 \begin{center}
   \emph{  Can \textbf{overparameterized} MAML generalize well to a new task, under general data and task distributions?}
 \end{center}
 In this work, we utilize the mixed linear regression, which is widely adopted in theoretical studies for meta-learning~\cite{kong2020meta,bernacchia2021meta,denevi2018learning,bai2021important}, as a proxy to address the above question. In particular, we assume that each task $\tau$ is a noisy linear regression and the associated weight vector is sampled from a common distribution. Under this model, we consider one-step MAML meta-trained with stochastic gradient descent (SGD), where  
 we minimize the loss evaluated at single GD step  further ahead for each task.  
 Such settings correspond to real-world implementations of MAML~\cite{finn2017meta,li2017meta,hospedales2020meta} and are extensively considered in theoretical analysis~\cite{fallah2020convergence,chen2022bayesian,fallah2021generalization}.
 The focus of this work is the overparameterized regime, i.e., the data dimension $d$ is far larger than the meta-training iterations $T$ ($d\gg T$). 
\subsection{Our Contributions} 
Our goal is to characterize the generalization  behaviours of the MAML output in the overparameterized regime, and to explore how different problem parameters, such as data and task distributions, the adaptation learning rate $\btr$, affect the test error. The main contributions are highlighted below.
\begin{itemize}
    \item Our first contribution is a sharp characterization (both upper and lower bounds) of the excess risk of MAML trained by SGD. The results are presented in a general manner, which depend on a new notion of effective meta weight, data spectrum, task covariance matrix, and other hyperparameters such as training and test learning rates. 
    In particular, the {\bf effective meta weight} captures an essential property of MAML, where the inner-loop gradient updates have distinctive effects on different dimensions of data eigenspace, i.e., the importance of "leading" space will be magnified whereas the "tail" space will be suppressed.  
    \item We investigate the influence of data and task distributions on the excess risk of MAML. For $\log$-decay data spectrum, our upper and lower bounds establish a sharp phase transition of the generalization. Namely, the excess risk vanishes for large $T$ (where benign fitting occurs) if the data spectrum decay rate is faster than the task diversity rate, and non-vanishing risk occurs otherwise. In contrast, for polynomial or exponential data spectrum decays, excess risk always vanishes for large $T$ irrespective of the task diversity spectrum. 

\item We showcase the important role the  adaptation learning rate $\btr$ plays in the excess risk and the early stopping time of MAML. We provably identify a novel tradeoff between the different impacts of $\btr$ on the "leading" and "tail" data spectrum spaces as the main reason behind the phenomena that the excess risk will first increase then decrease as $\btr$ changes from negative to positive values under general data settings. This complements the explanation based only on the "leading" data spectrum space given in~\cite{bernacchia2021meta} for the isotropic case.  
We further theoretically illustrate that $\btr$ plays a similar role in determining the early stopping time, i.e., the iteration at which MAML achieves steady generalization error.

\end{itemize}
\textbf{Notations.} 
We will use bold lowercase and  capital letters for vectors and matrices respectively. $\mathcal{N}\left(0, \sigma^{2}\right)$ denotes the Gaussian distribution with mean $0$ and variance $\sigma^2$. We use $f(x) \lesssim g(x)$ to denote the case $f(x) \leq c g(x)$ for some constant $c>0$. We use the standard big-O notation and its variants: $\mathcal{O}(\cdot),  \Omega(\cdot)$, where $T$ is the problem parameter that becomes large. Occasionally, we use the symbol $\widetilde{\mathcal{O}}(\cdot)$ 
to hide $\polylog(T)$ factors. $\mathbf{1}_{(\cdot)}$ denotes the indicator function. Let $x^{+}=\max\{x,0\}$.
\section{Related Work}%
\label{sec-related}
\paragraph{Optimization theory for MAML-type approaches} Theoretical guarantee of MAML was initially provided in~\cite{finn2017meta} by proving a universal approximation property under certain conditions. One line of theoretical works have focused on the optimization perspective. \cite{fallah2020convergence} established the convergence guarantee of one-step MAML for general nonconvex functions, and \cite{ji2022theoretical} extended such results to the multi-step setting. \cite{finn2019online}  analyzed the regret bound for online MAML. \cite{wang2020globala,wang2020globalb} studied the global optimality of MAML with sufficiently wide deep neural nets (DNN). Recently, \cite{collins2022maml} studied MAML from a representation point of view, and showed that MAML can provably recover the ground-truth subspace. h

\paragraph{Statistical theory for MAML-type approaches.} 
One line of theoretical analyses lie in the statistical aspect. \cite{fallah2021generalization} studied the generalization of MAML 
on recurring and unseen tasks. Information theory-type generalization bounds for MAML were developed in~\cite{jose2021information,chen2021generalization}.
\cite{chen2022bayesian} characterized the gap of generalization error between MAML and Bayes MAML. \cite{wang2020globalb} provided the statistical error bound for MAML with overparameterized DNN. Our work falls into this category, 
where the overparameterization has been rarely considered in previous works. Note that \cite{wang2020globalb} only derived the generalization bound from the complexity-based perspective to study the difference between the empirical and population losses for the obtained optimization solutions. Such complexity bound is typically related to the data dimension~\cite{neyshabur2018towards} and may yield vacuous bound in the high dimensional regime. However, our work show that the generalization error of MAML can be small even the data dimension is sufficiently large.

\paragraph{Overparamterized meta-learning.}
\cite{du2020few,sun2021towards} studied overparameterized meta-learning from a representation learning perspective. 
 The most relevant papers to our work are~\cite{zou2021unraveling,bernacchia2021meta}, where they derived the population risk 
in overparameterized settings to show the effect of the adaptation learning rate for MAML. Our analysis differs from these works from two essential perspectives: \romannum{1}).\ we analyze the excess risk of MAML based on the optimization trajectory of SGD in non-asymptotic regime, highlighting the dependence of iterations $T$, while they directly solved the MAML objective asymptotically; \romannum{2}). \cite{zou2021unraveling,bernacchia2021meta} mainly focused on the simple isotropic case for data and task covariance, while
 we explicitly explore the role of data and task distributions under general settings.

\paragraph{Overparameterized linear model.} There has been several recent progress in theoretical understanding of overparameterized linear model under different scenarios, where the main goal is to provide non-asymptotic generalization guarantees, such as studies of linear regression ~\cite{bartlett2020benign}, ridge regression~\cite{tsigler2020benign}, constant-stepsize SGD~\cite{zou2021benign}, decaying-stepsize SGD~\cite{wu2021last}, GD~\cite{xu2022relaxing}, Gaussian Mixture models~\cite{wang2021benign}. This paper aims to derive the non-asymptotic excess risk bound for MAML under mixed linear model, which can be independent of data dimension $d$ and still converge as the iteration $T$ enlarges.
\section{Preliminary}\label{sec-form}
\subsection{Meta Learning Formulation}
In this work, we consider a standard  meta-learning setting~\cite{fallah2021generalization}, where a number of tasks share some similarities, and the learner aims to find a good model prior by leveraging task similarities, so that the learner can quickly find a desirable model for a new task by adapting from such an initial prior.

{\bf Learning a proper initialization.}
Suppose we are given a collection of tasks $\textstyle\{\tau_t\}^{T}_{t=1}$ sampled from some distribution $\mathcal{T}$. For each task $\tau_t$, we observe $N$ samples $\textstyle\mathcal{D}_{t}\triangleq (\mathbf{X}_t,\mathbf{y}_{t})=\left\{\left(\mathbf{x}_{t, j}, y_{t, j}\right) \in \mathbb{R}^{d} \times \mathbb{R}\right\}_{j \in\left[N\right]}\stackrel{i.i.d.}{\sim} \mathbb{P}_{\phi_{t}}(y|\mathbf{x}) \mathbb{P}(\mathbf{x})$, where $\phi_t$ is the model parameter for the $t$-th task. The collection of $\{\mathcal{D}_{t}\}^{T}_{t=1}$ is denoted as $\mathcal{D}$. Suppose that $\mathcal{D}_{t}$ is randomly split into training and validation sets, denoted respectively as $\mathcal{D}^{\text{in}}_{t}\triangleq (\Xb^{\text{in}}_t,\yb_t^{\text{in}})$ and $\mathcal{D}^{\text{out}}_{t}\triangleq (\Xb^{\text{out}}_t,\yb_t^{\text{out}})$, correspondingly containing $n_{1}$ and $n_2$ samples (i.e., $N=n_1+n_2$).  
We let $\boldsymbol{\omega}\in\mathbb{R}^{d}$ denote the initialization variable. Each task $\tau_t$ applies an inner algorithm $\mathcal{A}$ with such an initial and obtains an output $\mathcal{A}(\boldsymbol{\omega};\mathcal{D}^{\text{in}}_{t})$. Thus, the adaptation performance of $\boldsymbol{\omega}$ for task $\tau_t$ can be measured by the mean squared loss over the validation set given by $\textstyle\ell(\mathcal{A}(\boldsymbol{\omega};\mathcal{D}^{\text{in}}_{t});\mathcal{D}^{\text{out}}_{t}):= \frac{1}{2n_2}\sum^{n_2}_{j=1} \left(\left\langle \mathbf{x}^{\text{out}}_{t,j}, \mathcal{A}(\boldsymbol{\omega};\mathcal{D}^{\text{in}}_{t})\right\rangle-y^{\text{out}}_{t,j}\right)^{2}$. The goal of meta-learning is to find an optimal
initialization $\hat{\boldsymbol{\omega}}^{*}\in\mathbb{R}^{d}$ by minimizing the following empirical meta-training loss:
 \begin{align}
\min_{\boldsymbol{\omega}\in\mathbb{R}^{d}} \widehat{\mathcal{L}}(\mathcal{A},\boldsymbol{\omega};\mathcal{D}) \quad \text{ where }\;
   \widehat{\mathcal{L}}(\mathcal{A},\boldsymbol{\omega};\mathcal{D})&=\frac{1}{T}\sum^{T}_{t=1}\ell(\mathcal{A}(\boldsymbol{\omega};\mathcal{D}^{\text{in}}_{t});\mathcal{D}^{\text{out}}_{t})\label{emp_loss}.
 \end{align}
In the testing process, suppose a new task $\tau$ sampled from $\mathcal{T}$ is given, which is associated with the dataset $\mathcal{Z}$ consisting of $m$ points  with the task. We apply the learned initial $\hat{\boldsymbol{\omega}}^{*}$ 
as well as the inner algorithm 
$\mathcal{A}$ on $\mathcal{Z}$ to produce a task predictor. Then the test performance can be evaluated via the following population loss: 
\begin{align}
    \mathcal{L}(\mathcal{A},\boldsymbol{\omega})=\mathbb{E}_{\tau \sim \mathcal{T}} \mathbb{E}_{\mathcal{Z},(\mathbf{x}, y)\sim \mathbb{P}_{\phi}(y \mid \mathbf{x}) \mathbb{P}(\mathbf{x})}  \left[\ell\left(\mathcal{A}\left(\boldsymbol{\omega}; \mathcal{Z}\right);(\mathbf{x},y)\right)\right]. 
    \label{obj}  
\end{align}

\paragraph{Inner Loop with one-step GD.} 
Our focus of this paper is the popular meta-learning algorithm MAML~\cite{finn2017model}, where inner stage takes a few steps of GD update initialized from $\boldsymbol{\omega}$. We consider one step for simplicity, which is commonly adopted in the previous studies~\cite{bernacchia2021meta,collins2022maml,gao2020modeling}. Formally, for any $\boldsymbol{\omega}\in\mathbb{R}^d$, and any dataset $(\mathbf{X},\mathbf{y})$ with $n$ samples, the inner loop algorithm for MAML with a learning rate $\beta$ is given by
\begin{align}
  \mathcal{A}(\boldsymbol{\omega};(\mathbf{X},\mathbf{y})):= \boldsymbol{\omega}-\beta \nabla_{\boldsymbol{\omega}} \ell\left(\boldsymbol{\omega};(\mathbf{X},\mathbf{y})\right)=
  (\mathbf{I}-\frac{\beta}{n}\mathbf{X}^{\top}\mathbf{X})\boldsymbol{\omega}+\frac{\beta}{n}\mathbf{X}^{\top}\mathbf{y}.
\end{align}
We allow the learning rate to differ at the meta-training and testing stages, denoted as $\beta^{\text{tr}}$ and $\beta^{\text{te}}$ respectively. Moreover, in subsequent analysis, we will include the dependence on the learning rate to the inner loop algorithm and loss functions as
$\mathcal{A}(\boldsymbol{\omega},\beta;(\mathbf{X},\mathbf{y}))$, $\widehat{\mathcal{L}}(\mathcal{A},\boldsymbol{\omega},\beta;\mathcal{D})$ and $\mathcal{L}(\mathcal{A},\boldsymbol{\omega},\beta)$.
\paragraph{Outer Loop with SGD.}
%
We adopt SGD to iteratively update the meta initialization variable $\boldsymbol{\omega}$ based on the empirical meta-training loss \cref{emp_loss}, which is how MAML is implemented in practice~\cite{finn2017meta}. 
Specifically, we use the constant stepsize SGD with iterative averaging~\cite{fallah2021generalization, denevi2018learning,denevi2019learning}, 
and the algorithm is summarized in \Cref{alg-meta-sgd}. Note that at each iteration, we use one task for updating the meta parameter, which can be easily generalized to the case with a mini-batch tasks for each iteration. 

\begin{algorithm}[ht]
  \caption{MAML with SGD}\label{alg-meta-sgd}
  \begin{algorithmic} 
  \REQUIRE Stepsize $\alpha>0$, meta learning rate $\beta^{\text{tr}}>0$
  \ENSURE $\boldsymbol{\omega}_{0}$
  \FOR{$t=1$ to $T$} 
  \STATE Receive task $\tau_t$ with data $\mathcal{D}_t$ 
  \STATE Randomly divided into training and validation set: $\mathcal{D}^{in}_{t}=(\mathbf{X}^{in}_t, \mathbf{y}^{in}_t)$, $\mathcal{D}^{out}_{t}=(\mathbf{X}^{out}_t, \mathbf{y}^{out}_t)$
  \STATE Update $\boldsymbol{\omega}_{t+1} =\boldsymbol{\omega}_{t}-\alpha \nabla \ell(\mathcal{A}(\boldsymbol{\omega},\beta^{\text{tr}};\mathcal{D}^{\text{in}}_{t});\mathcal{D}^{\text{out}}_{t})$
  \ENDFOR
  \RETURN $\overline{\boldsymbol{\omega}}_T=\frac{1}{T}\sum^{T-1}_{t=0} \boldsymbol{\omega}_{t}$
  \end{algorithmic}
  \end{algorithm}
\paragraph{Meta Excess Risk of SGD.} Let $\boldsymbol{\omega}^{*}$ denote the optimal solution to the population meta-test error 
\cref{obj}. 
We define the following excess risk 
for the output $\overline{\boldsymbol{\omega}}_T$ of SGD:
\begin{align}
R(\overline{\boldsymbol{\omega}}_T,\beta^{\text{te}})\triangleq   \mathbb{E}\left[\mathcal{L}(\mathcal{A},\overline{\boldsymbol{\omega}}_T,\beta^{\text{te}})\right]-\mathcal{L}(\mathcal{A},\boldsymbol{\omega}^{*},\beta^{\text{te}})\label{excess}
\end{align}
which identifies the difference between adapting from the SGD output $\overline{\boldsymbol{\omega}}_T$ and from the optimal initialization $\boldsymbol{\omega}^{*}$. Assuming that each task contains a fixed constant number of samples, the total number of samples over all tasks is $\mathcal{O}(T)$. Hence, the overparameterized regime can be identified as $d\gg T$, which is the focus of this paper, and is in contrast to the well studied underparameterized setting with finite dimension $d$ $(d\ll T)$.
The goal of this work is to characterize the impact of SGD dynamics, demonstrating how the iteration $T$ affects the excess risk, which has not been considered in the previous overparameterized MAML analysis~\cite{bernacchia2021meta,zou2021unraveling}. 


\subsection{Task and Data Distributions}
To gain more explicit knowledge of MAML, we specify the task and data distributions in this section.

{\bf Mixed Linear Regression.} 
We consider a canonical case in which the tasks are linear regressions. This setting has been commonly adopted recently in~\cite{bernacchia2021meta,bai2021important,kong2020meta}. 
Given a task $\tau$, its model parameter $\phi$ is determined by 
$\boldsymbol{\theta}\in\mathbb{R}^{d}$, 
and the output response is generated as follows:
\begin{align}
   y=\boldsymbol{\theta}^{\top} \mathbf{x}+z, \quad \xb\sim\mathcal{P}_{\xb},\quad z\sim \mathcal{P}_{z}
\end{align}
where $\xb$ is the input feature, which follows the same distribution $\mathcal{P}_{\xb}$ across different tasks, and $z$ is the i.i.d.\ Gaussian noise sampled from $\mathcal{N}(0,\sigma^2)$. The task signal  $\boldsymbol{\theta}$ has the mean $\boldsymbol{\theta}^{*}$ and the covariance $\Sigma_{\boldsymbol{\theta}}\triangleq \mathbb{E}[\boldsymbol{\theta}\boldsymbol{\theta}^{\top}]$. Denote the distribution of $\boldsymbol{\theta}$ as $\mathcal{P}_{\boldsymbol{\theta}}$. We do not make any additional assumptions on $\mathcal{P}_{\boldsymbol{\theta}}$, whereas recent studies on  MAML~\cite{bernacchia2021meta,zou2021unraveling} assume it to be Gaussian and isotropic. 

{\bf Data distribution.} For the data distribution $\mathcal{P}_{\xb}$, we 
 first introduce some mild regularity conditions:
\begin{enumerate}
  \item $\xb\in\mathbb{R}^d$ is mean zero with covariance operator $\bSigma=\mathbb{E}[\xb \xb^{\top}]$;
  \item The spectral decomposition of $\bSigma$ is $\boldsymbol{V} \boldsymbol{\Lambda} \boldsymbol{V}^{\top}=\sum_{i>0} \lambda_{i} \boldsymbol{v}_{i} \boldsymbol{v}_{i}^{\top}$, with decreasing eigenvalues $\lambda_1\geq \cdots\geq\lambda_d>0$, and suppose $\sum_{i>0}\lambda_{i} <\infty $.
  \item $\bSigma^{-\frac{1}{2}} \mathbf{x}$ is $\sigma_{\xb}$-subGaussian.
\end{enumerate}
To analyze the stochastic approximation method SGD
, we take the following standard fourth moment condition~\cite{zou2021benign, jain2017markov, berthier2020tight}.
\begin{assumption}[Fourth moment condition] There exist positive constants $c_1,b_1>0$, such that for any positive semidefinite (PSD) 
matrix $\mathbf{A}$, it holds that
\begin{align*}
b_1 \operatorname{tr}(\bSigma \mathbf{A}) \Sigma+\Sigma \mathbf{A }\bSigma \preceq\mathbb{E}_{\mathbf{x} \sim \mathcal{P}_{\mathbf{x}}}\left[\mathbf{x x}^{\top} \mathbf{A} \mathbf{x} \mathbf{x}^{\top}\right] \preceq c_1 \operatorname{tr}(\bSigma \mathbf{A}) \Sigma
\end{align*}
For the Gaussian distribution, it suffices to take $c_1=3,b_1=2.$
\end{assumption}
 \subsection{Connection to a Meta Least Square Problem.} After instantiating our study on the task and data distributions in the last section, note that $\textstyle\nabla\ell(\mathcal{A}(\boldsymbol{\omega},\beta^{\text{tr}};\mathcal{D}^{\text{in}}_{t});\mathcal{D}^{\text{out}}_{t})$ is linear 
 with respect to $\boldsymbol{\omega}$. Hence, we can reformulate the problem \cref{emp_loss} as a least square (LS) problem with transformed meta inputs and output responses. 
 \begin{proposition}[Meta LS Problem]\label{prop1} Under the mixed linear regression model, 
 the expectation of the meta-training loss \cref{emp_loss} taken over task and data distributions can be rewritten as:
  \begin{align}\label{linear-loss}
    \mathbb{E}\left[\widehat{\mathcal{L}}(\mathcal{A},\boldsymbol{\omega},\beta^{\text{tr}};\mathcal{D})\right]= \mathcal{L}(\mathcal{A},\boldsymbol{\omega},\beta^{\text{tr}})=
   \mathbb{E}_{\Bb , \boldsymbol{\gamma}}  \frac{1}{2}\left[\left\|\Bb\boldsymbol{\omega}-\boldsymbol{\gamma}\right\|^{2}\right].
\end{align}
The meta data are given by
\begin{align}\Bb =&  \frac{1}{\sqrt{n_2}}\Xb^{out}\Big(\mathbf{I}-\frac{\beta^{\text{tr}}}{n_1} {\Xb^{\text{in}}}^{T} {\Xb^{\text{in}}}\Big)\nonumber\\ \boldsymbol{\gamma} =&  \frac{1}{\sqrt{n_2}}\Big(  \Xb^{\text{out}}\Big( \mathbf{I}-\frac{\beta^{\text{tr}}}{n_1} {\Xb^{\text{in}}}^{T} {\Xb^{\text{in}}}\Big)\boldsymbol{\theta}+\zb^{out}-\frac{\beta^{\text{tr}}}{n_1} \Xb^{\text{out}}{\Xb^{\text{in}}}^{\top}\zb^{\text{in}}\Big)\label{meta-data} \end{align}
where $\Xb^{\text{in}}\in \mathbb{R}^{n_1\times d}$,$\zb^{\text{in}}\in \mathbb{R}^{n_1}$,$\Xb^{\text{out}}\in \mathbb{R}^{n_2\times d}$ and $\zb^{\text{out}}\in \mathbb{R}^{n_2}$ denote the inputs and noise for training and validation.
Furthermore,
we have
\begin{align}
  \boldsymbol{\gamma} = \Bb\boldsymbol{\theta}^{*}+\boldsymbol{\xi}\quad \text{ with meta noise } \mathbb{E}[\boldsymbol{\xi}\mid\Bb]=0. \label{linear}
 \end{align} 
 \end{proposition}
Therefore, the meta-training objective 
is equivalent to searching for a $\boldsymbol{\omega}$, which is close to the task mean  $\boldsymbol{\theta}^{*}$. 
Moreover, 
with the specified data and task model, the optimal solution for meta-test loss \cref{obj} can be directly calculated~\cite{gao2020modeling}, and we obtain $\boldsymbol{\omega}^{*}=\mathbb{E}[\btheta] =\btheta^{*}$. Hence, the meta excess risk~\cref{excess} is identical to the standard excess risk~\cite{bartlett2020benign} for the linear model \cref{linear}, i.e., $ R(\overline{\boldsymbol{\omega}}_T,\beta^{\text{te}})=   \mathbb{E}_{\Bb , \boldsymbol{\gamma}}  \frac{1}{2}\left[\left\|\Bb\wl_{T}-\boldsymbol{\gamma}\right\|^{2}-\left\|\Bb\btheta^{*}-\boldsymbol{\gamma}\right\|^{2}\right]$, but with more complicated input and output data expressions. The following analysis will focus on this transformed linear model. 

Furthermore, we can calculate the statistical properties of the reformed input $\Bb$, and obtain the meta-covariance: $$ E[\Bb^{\top}\Bb]=(\Ib-\beta^{\text{tr}}\bSigma)^2\bSigma+\frac{{\beta^{\text{tr}}}^2}{n_1}(F-\bSigma^3)$$ where $F=E[\xb\xb^{\top}\Sigma \xb\xb^{\top}]$. Let $\Xb\in\mathbb{R}^{n\times d}$ denote the collection of $n$ i.i.d.\ samples from $\mathcal{P}_{\xb}$, and denote
 $$
 \Hb_{n,\beta}=\mathbb{E}[(\mathbf{I}-\frac{\beta}{n}\Xb^{\top}\Xb)\Sigma(\mathbf{I}-\frac{\beta}{n}\Xb^{\top}\Xb)
  ]=(\Ib-\beta\bSigma)^2\bSigma+\frac{\beta^2}{n}(F-\bSigma^3).
 $$
We can then write $E[\Bb^{\top}\Bb]=\Hb_{n_1,\beta^{\text{tr}}}$.
Regarding the form of $\Bb$ and $\Hb_{n_1,\beta^{\text{tr}}}$, we need some
further conditions on the higher order moments of the data distribution. 
 \begin{assumption}[Commutity
 ]\label{ass-comm}
  $F=E[\xb\xb^{\top}\bSigma \xb\xb^{\top}]$ commutes with the data covariance $\bSigma$.
 \end{assumption}
\Cref{ass-comm} holds for Gaussian data. Such commutity of $\bSigma$ has also been considered in ~\cite{zou2021benign}. 
\begin{assumption}[Higher order moment condition]\label{ass:higherorder} 
  Given $|\beta|<\frac{1}{\lambda_1}$ and $\bSigma$, there exists a constant $C(\beta,\bSigma)>0$, for large $n>0$, s.t. for any unit vector $\vb\in\mathbb{R}^d$, we have:
  \begin{align}\label{hoc}
      \mathbb{E}[\|\vb^{\top}\Hb^{-\frac{1}{2}}_{n,\beta}(\mathbf{I}-\frac{\beta}{n}\Xb^{\top}\Xb)\bSigma (\mathbf{I}-\frac{\beta}{n}\Xb^{\top}\Xb)\Hb^{-\frac{1}{2}}_{n,\beta}\vb\|^2]< C(\beta,\bSigma).
  \end{align}
\end{assumption}
In \Cref{ass:higherorder}, the analytical form of $C(\beta,\bSigma)$ can be derived if $\bSigma^{-\frac{1}{2}}\mathbf{x}$ is Gaussian.
Moreover, if $\beta=0$, then we obtain $C(\beta,\bSigma)=1$. Further technical  discussions  
are presented in Appendix. 
\section{Main Results}\label{sec-main}
In this section, we present our analyses on generalization properties of MAML optimized by average SGD and derive insights on the effect of various parameters. Specifically, our results consist of three parts. First, we characterize the meta excess risk of MAML trained with SGD. Then, we establish the generalization error bound for various types of data and task distributions, to reveal which kind of overparameterization regarding data and task is essential for diminishing meta excess risk. Finally, we explore how the adaptation learning rate $\btr$ affects the excess risk and the training dynamics. 

\subsection{Performance Bounds}

Before starting our results, we first introduce relevant notations and concepts. We define the following rates of interest (See \Cref{remark-f} for further discussions)
\begin{align*} c(\beta,\bSigma) &:= c_1(1+8|\beta|\lambda_1\sqrt{C(\beta,\bSigma)}\sigma_x^2+ 64\sqrt{C(\beta,\bSigma)}\sigma_x^4\beta^2\operatorname{tr}(\bSigma^2))\\
  f(\beta,n,\sigma,\bSigma,\bSigma_{\boldsymbol{\theta}})&:=c(\beta,\bSigma)\operatorname{tr}({\bSigma_{\boldsymbol{\theta}}\bSigma})+4c_1\sigma^2\sigma_x^2\beta^2\sqrt{C(\beta,\bSigma)}\operatorname{tr}(\bSigma^2)+\sigma^2/n\\
  g(\beta,n, \sigma,\bSigma, \bSigma_{\btheta}) & :={\sigma^2+b_1\operatorname{tr}(\bSigma_{\btheta}\Hb_{n,\beta})+\beta^2\mathbf{1}_{\beta\leq 0} b_1 \operatorname{tr}(\bSigma^2)/{n}}.
\end{align*}
Moreover, for a positive semi-definite matrix $\Hb$, s.t. $\Hb$ and $\bSigma$ can be diagonalized simultaneously,  let $\mu_i(\Hb)$ denote its corresponding eigenvalues for $\vb_i$, i.e. $\Hb = \sum_{i}\mu_i(\Hb)\vb_i\vb_i^{\top}$ (Recall $\vb_i$ is the $i$-th eigenvector of $\bSigma$).

We next introduce the following new notion of the \emph{effective meta weight}, which will serve as an important quantity for capturing the generalization of MAML. 
\begin{definition}[Effective Meta Weights]\label{meta-weight}
  For $|\btr|,|\bte|<1/\lambda_1$, given step size $\alpha $ and iteration $T$, define
  \begin{equation}
      \Xi_i (\bSigma
      ,\alpha,T)=\begin{cases}
      \mu_i(\Hb_{m,\beta^{\text{te}}})/\left(T \mu_i(\Hb_{n_1,\beta^{\text{tr}}})\right) &  \mu_{i}(\Hb_{n_1,\beta^{\text{tr}}})\geq \frac{1}{\alpha T}; \\
      T\alpha^2 \mu_i(\Hb_{n_1,\btr})\mu_i(\Hb_{m,\beta^{\text{te}}})&  \mu_{i}(\Hb_{n_1,\beta^{\text{tr}}})< \frac{1}{\alpha T}.
      \end{cases}
  \end{equation}
  We call $ \mu_i(\Hb_{m,\beta^{\text{te}}})/ \mu_i(\Hb_{n_1,\beta^{\text{tr}}})$ and $ \mu_i(\Hb_{m,\beta^{\text{te}}})\mu_i(\Hb_{n_1,\beta^{\text{tr}}})$ the \textbf{meta ratio} (See \Cref{remark-weight}). 
\end{definition}
 We omit the arguments of the effective meta weight $\Xi_i$ for simplicity in the following analysis. 

Our first results characterize matching upper and lower bounds on the meta excess risk of MAML in terms of the effective meta weight.
\begin{theorem}[Upper Bound]\label{thm-upper} Let $\omega_i=\left\langle\boldsymbol{\omega}_{0}-\boldsymbol{\theta}^{*}, \mathbf{v}_{i}\right\rangle$. If $|\btr|,|\bte|<1/\lambda_1$, $n_1$ is large ensuring that $\mu_i(\Hb_{n_1,\beta^{\text{tr}}})>0$, $\forall i$ and
$\alpha<1/\left(c(\btr,\bSigma) \operatorname{tr}(\bSigma)\right)$, then the meta excess risk $R(\overline{\boldsymbol{\omega}}_T,\bte)$ is bounded above as follows
\[R(\wl_{T},\bte)\leq \text{Bias}+ \text{Var} \]
where
  \begin{align*}
     \text{Bias}  & =  \frac{2}{\alpha^2 T} \sum_{i}\Xi_i \frac{\omega_i^2}{\mu_i(\Hb_{n_1,\beta^{\text{tr}}})} \\
     \text{Var}  &= \frac{2}{(1-\alpha c(\btr,\bSigma) \operatorname{tr}(\bSigma))}\left(\sum_{i}\Xi_{i} \right)
   \\
 \quad  \times & [\underbrace{f(\btr,n_2,\sigma,\bSigma_{\boldsymbol{\theta}},\bSigma)}_{V_1}+\underbrace{ 2c(\btr,\bSigma)
 \sum_{i}\left( \frac{\mathbf{1}_{\mu_{i}(\Hb_{n_1,\beta^{\text{tr}}})\geq \frac{1}{\alpha T}}}{T\alpha \mu_i(\Hb_{n_1,\beta^{\text{tr}}})}+\mathbf{1}_{\mu_{i}(\Hb_{n_1,\beta^{\text{tr}}})< \frac{1}{\alpha T} } \right) \lambda_{i}\omega^2_i}_{V_2}] 
\end{align*}

\end{theorem}
\begin{remark}\label{remark-1}
The primary error source of the upper bound are two folds. The bias term corresponds to the error if we directly implement GD updates towards the meta objective~\cref{linear-loss}. The variance error is composed of the disturbance of meta 
 noise  $\boldsymbol{\xi}$ (the $V_1$ term), and the randomness of SGD itself (the $V_2$ term). Regardless of data or task distributions, for proper stepsize $\alpha$, we can easily derive that the bias term is $\mathcal{O}(\frac{1}{T})$, and the $V_2$ term is also  $\mathcal{O}(\frac{1}{T})$, which is dominated by $V_1$ term ($\Omega(1)$). Hence, to achieve the vanishing risk, we need to understand the roles of $\Xi_i$ and  $f(\cdot)$ 
\end{remark}
\begin{remark}[Effective Meta Weights]\label{remark-weight}
By \Cref{meta-weight}, we separate the data eigenspace into “\textbf{leading}” $(\geq \frac{1}{\alpha T})$ and “{\bf tail}” $(< \frac{1}{\alpha T})$ spectrum spaces with different meta weights. 
The meta ratios 
indicate the impact of one-step gradient update. For large $n$, 
$\mu_i(\Hb_{n,\beta})\approx (1-\beta\lambda_i)^2\lambda_i$, and hence a larger $\btr$ in training will increase the weight for “leading” space and decrease the weight for “tail” space, while a  larger $\bte$ always decreases the weight.
\end{remark}
\begin{remark}[Role of $f(\cdot)$]\label{remark-f} $f(\cdot)$ in variance term consists of various sources of meta noise $\boldsymbol{\xi}$, including inner gradient updates ($\beta$), task diversity ($\bSigma_{\btheta}$) and noise from regression tasks ($\sigma$). As mentioned in \Cref{remark-1}, understanding $f(\cdot)$ is critical in our analysis. Yet, 
due to the multiple randomness origins, techniques for classic linear  regression~\cite{zou2021benign,jain2017markov} cannot be directly applied here. 
Our analysis overcomes such non-trivial challenges. $g(\cdot)$ in \Cref{thm-lower} plays a similar role to $f(\cdot)$. 
\end{remark}
Therefore, \Cref{thm-upper} implies that   overparameterization is crucial for diminishing risk under the following conditions:
 \begin{itemize}[itemsep=2pt,topsep=0pt,parsep=0pt]
     \item For $f(\cdot)$: $\operatorname{tr}(\bSigma\bSigma_{\btheta})$ and $\operatorname{tr}(\bSigma^2)$ is small compared to $T$;
     \item For $\Xi_i$: the dimension of "leading" space is $o(T)$, and the summation of meta ratio over "tail" space is $o(\frac{1}{T})$.
 \end{itemize}
 
We next provide a lower bound on the meta excess risk, which matches the upper bound in order. 
\begin{theorem}[Lower Bound]~\label{thm-lower}
  Following the similar notations in ~\Cref{thm-upper}, Then 
  \begin{align*}
 R(\overline{\boldsymbol{\omega}}_T,\bte)  \ge & \frac{1}{100\alpha^2 T} \sum_{i}\Xi_i \frac{\omega_i^2}{\mu_i(\Hb_{n_1,\beta^{\text{tr}}})} +\frac{1}{n_2}\cdot \frac{1}{(1-\alpha c(\btr,\bSigma) \operatorname{tr}(\bSigma))}\sum_{i}\Xi_{i}   \\
   \times & [\frac{1}{100} g(\btr,n_1, \bSigma, \bSigma_{\btheta})+\frac{b_1}{1000}
  \sum_{i}\Big( \frac{\mathbf{1}_{\mu_{i}(\Hb_{n_1,\beta^{\text{tr}}})\geq \frac{1}{\alpha T}}}{T\alpha \mu_i(\Hb_{n_1,\beta^{\text{tr}}})}+\mathbf{1}_{\mu_{i}(\Hb_{n_1,\beta^{\text{tr}}})< \frac{1}{\alpha T} } \Big) \lambda_{i}\omega^2_i].
 \end{align*}
\end{theorem}
Our lower bound can also be decomposed into bias and variance terms as the upper bound. The bias term well matches the upper bound up to absolute constants. The variance term differs from the upper bound only by $\frac{1}{n_2}$, where $n_2$ is the batch size of each task, and is treated as a constant (i.e., does not scale with $T$)~\cite{jain2018parallelizing,shalev2013accelerated} in practice. Hence, in the overparameterized regime where $d\gg T$ and $T$ tends to be sufficiently large, the variance term also matches that in the upper bound w.r.t.\ $T$.



\subsection{The Effects of Task Diversity}\label{sec-main-task}
From \Cref{thm-upper} and \Cref{thm-lower}, we observe that the task diversity $\bSigma_{\btheta}$ in $f(\cdot)$ and $g(\cdot)$ plays a crucial role in the performance guarantees for MAML. In this section, 
we explore several types of data distributions to further characterize the effects of the task diversity.

We take the single task setting as a comparison with meta-learning, where the task diversity diminishes (tentatively say $\bSigma_{\btheta}\rightarrow\mathbf{0}$), i.e., each task parameter $\boldsymbol{\theta}=\boldsymbol{\theta}^{*}$. In such a case,  it is unnecessary to do one-step gradient in the inner loop and we set $\btr=0$, which is equivalent to  directly running SGD.  Formally, the {\bf single task setting} can be described as outputting $\overline{\boldsymbol{\omega}}^{\text{sin}}_{T}$ with iterative SGD that minimizes $\widehat{\mathcal{L}}(\mathcal{A},\boldsymbol{\omega},0;\mathcal{D})$ with meta linear model as $\boldsymbol{\gamma} = \frac{1}{\sqrt{n_2}}(\Xb^{out}\boldsymbol{\theta}^{*}+\zb^{\text{out}})$.


\Cref{thm-upper} implies that the data spectrum should decay fast, which leads to a small dimension of "leading" space and small meta ratio summation over "tail" space. Let us first consider a relatively slow decaying case:  $\lambda_k=k^{-1}\log^{-p}(k+1)$ for some $p>1$. Applying \Cref{thm-upper}, we immediately derive the theoretical guarantees for single task:
\begin{lemma}[Single Task]\label{lem-single}
    If $|\beta^{\text{te}}|<\frac{1}{\lambda_1}$ and if the spectrum of $\bSigma$ satisfies $\lambda_k=k^{-1}\log^{-p}(k+1)$, then $R(\overline{\boldsymbol{\omega}}^{\text{sin}}_T,\beta^{\text{te}})=\mathcal{O}(\frac{1}{\log^{p}(T)})$
\end{lemma}
At the test stage, if we set $\beta^{\text{te}}=0$, then the meta excess risk for the single task setting, i.e.,  $R(\overline{\boldsymbol{\omega}}^{\text{sin}}_T,0)$, is exactly the excess risk in classical linear regression~\cite{zou2021benign}. \Cref{lem-single} can be regarded as a generalized version of Corollary 2.3 in \cite{zou2021benign}, where they provide the upper bound for $R(\overline{\boldsymbol{\omega}}^{\text{sin}}_T,0)$, while we allow a one-step fine-tuning for testing. 

Lemma~\ref{lem-single} suggests that the $\log$-decay  is 
sufficient to assure that $R(\overline{\boldsymbol{\omega}}^{\text{sin}}_T,0)$ is diminishing when $d\gg T$.
 However, in meta-learning with multi-tasks, the task diversity captured by the task spectral distribution can highly affect the meta excess risk. In the following, our \Cref{thm-upper} and \Cref{thm-lower} (i.e., upper and lower bounds) establish a sharp phase transition of the generalization for MAML for the same data spectrum considered in Lemma~\ref{lem-single}, which is in contrast to the single task setting (see \Cref{lem-single}), where $\log$-decay data spectrum always yields vanishing excess risk.

\begin{proposition}[MAML, $\log$-Decay Data  Spectrum]\label{prop-hard} 
 Given $|\btr|, |\beta^{\text{te}}|<\frac{1}{\lambda_1}$, under the same data distribution as in \Cref{lem-single} with $d=\mathcal{O}(\operatorname{poly}(T))$  and the spectrum of $\bSigma_{\btheta}$, denoted as $\nu_i$,  satisfies $\nu_k=\log^{r}(k+1)$ for some $r>0$,  then 
    $$R(\overline{\boldsymbol{\omega}}_T,\beta^{\text{te}})=\begin{cases}
  \Omega(\log^{r-2p+1}(T))&{r\geq 2p-1}\\ \mathcal{O}(\frac{1}{\log^{p-(r-p+1)^{+}}(T)}) &{r<2p-1}
  \end{cases}$$
\end{proposition}
\Cref{prop-hard} implies that under $\log$-decay data spectrum parameterized by $p$, the meta excess risk of MAML experiences a phase transition determined by the spectrum parameter $r$ of task diversity. While slower task diversity rate $r < 2p-1$ guarantees vanishing excess risk, faster task diversity rate $r \ge 2p-1$ necessarily results in non-vanishing excess risk.  
\Cref{prop-hard} and \Cref{lem-single} together indicate that while $\log$-decay data spectrum always yields benign fitting (vanishing risk) in the single task setting, it can yield non-vanishing risk in meta learning due to fast task diversity rate.


We further validate our theoretical results in  \Cref{prop-hard} by experiments. 
We consider the case $p=2$. As shown in \Cref{fig:subfig:1}, when $r<2p-1$, the test error quickly converges to the Bayes error. When $r>2p-1$,  \Cref{fig:subfig:2} illustrates that MAML already 
converges on the training samples, but the test error (which is further zoomed in \Cref{fig:subfig:3}) levels off and does not vanish, showing MAML generalizes poorly when $r>2p-1$.

\begin{figure}[ht]
\centering
\begin{subfigure}{.31\textwidth}
		\centering
    \includegraphics[width=\linewidth]{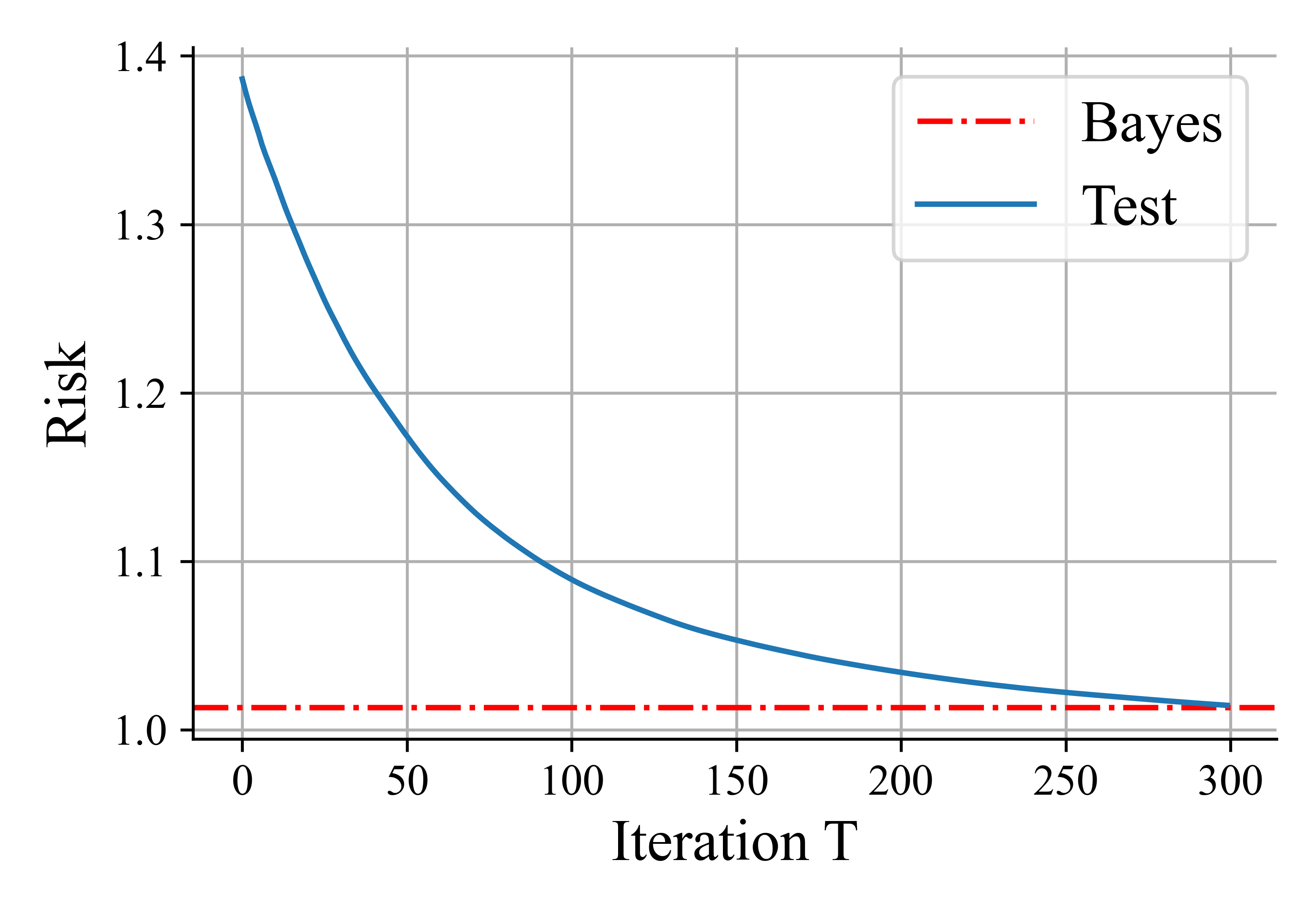}
  \caption{\small$\nu_i=0.25\log^{1.5}(i+1)$}
  \label{fig:subfig:1} 
	\end{subfigure}
\begin{subfigure}{.31\textwidth}
    \centering
	\includegraphics[width=\linewidth]{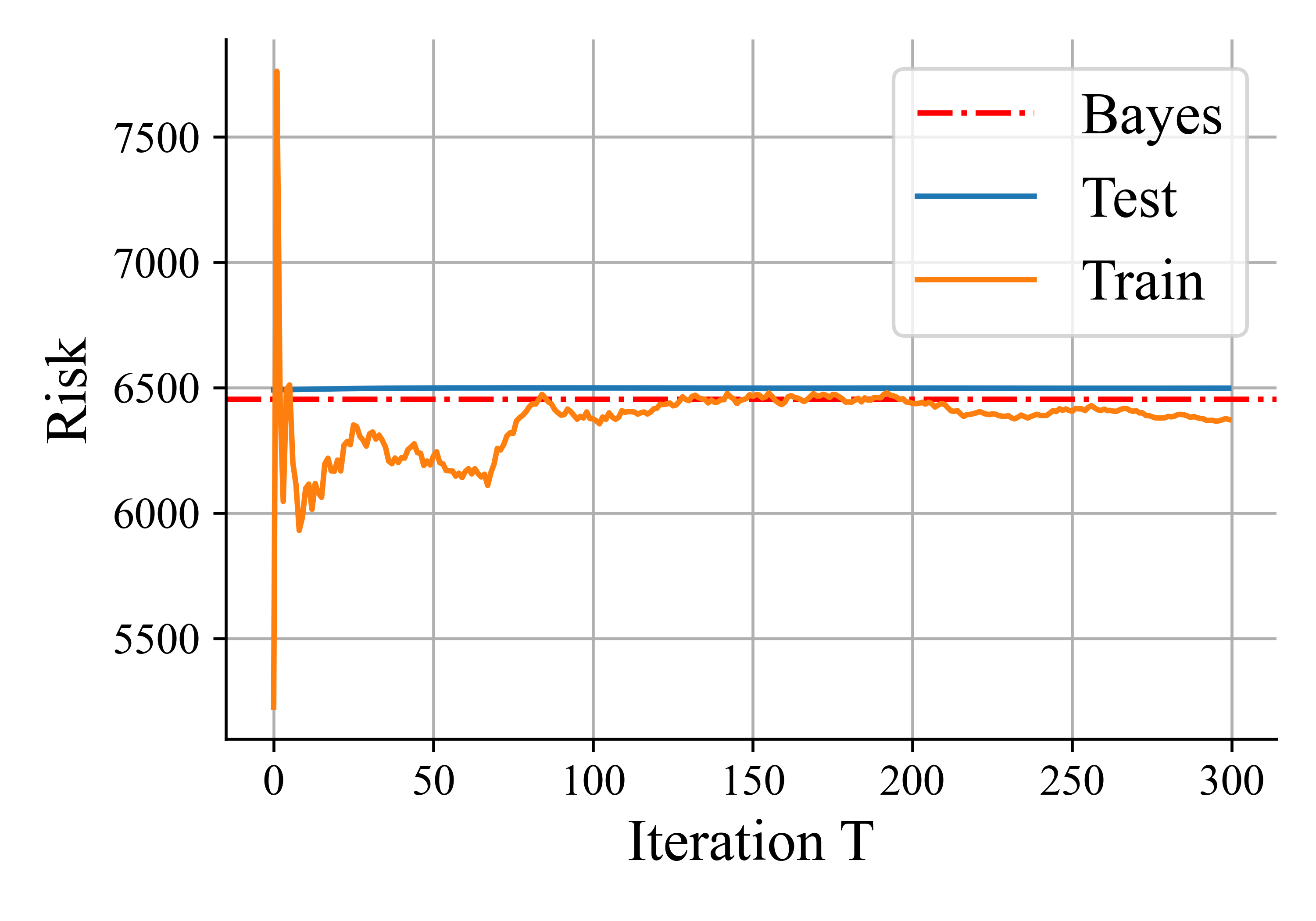}
  \caption{\small$\nu_i=0.25\log^{8}(i+1)$}
  \label{fig:subfig:2} 
  \end{subfigure}
  \begin{subfigure}{.31\textwidth}
    \centering
	\includegraphics[width=\linewidth]{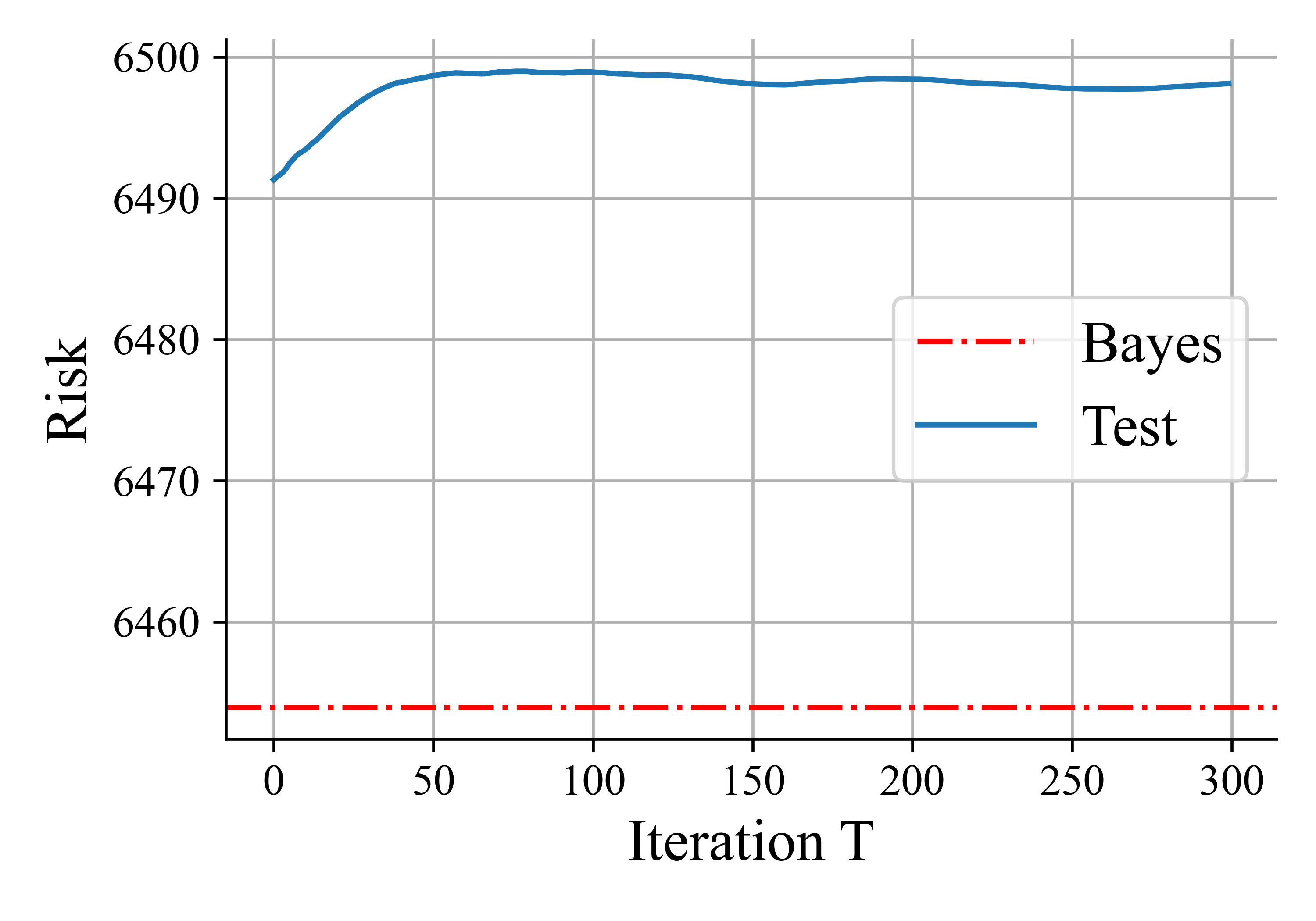}
  \caption{\small $\nu_i=0.25\log^{8}(i+1)$}
  \label{fig:subfig:3} 
  \end{subfigure}
  \caption{The effects of task diversity. $d=500$, $T=300$, $\lambda_i = \frac{1}{i\log(i+1)^2}$, $\btr=0.02$, $\bte=0.2$
  }
  \label{fig:twopicture} 
   \vspace{-0.3cm}
\end{figure}

Furthermore, we show that the above phase transition that occurs for $\log$-decay data distributions no longer exists for data distributions with faster decaying spectrum.

\begin{proposition}[MAML, Fast-Decay Data Spectrum]\label{prop-fast} Under the same task distribution as in \Cref{prop-hard}, i.e., the spectrum of $\bSigma_{\btheta}$, denoted as $\nu_i$, satisfies $\nu_k=\log^{r}(k+1)=\widetilde{O}(1)$ for some $r>0$, and the data distribution with $d=\mathcal{O}(\operatorname{poly}(T))$ satisfies:
\begin{enumerate}[itemsep=2pt,topsep=0pt,parsep=0pt]
  \item $\lambda_k=k^{-q}$ for some $q>1$, $R(\overline{\boldsymbol{\omega}}^{\text{sin}}_T,\beta^{\text{te}})=\mathcal{O}\left(\frac{1}{T^{\frac{q-1}{q}}}\right)$ and $R(\overline{\boldsymbol{\omega}}_T,\beta^{\text{te}})=\widetilde{\mathcal{O}}\left(\frac{1}{T^{\frac{q-1}{q}}}\right)$;
  \item  $\lambda_k=e^{-k}$,  $R(\overline{\boldsymbol{\omega}}^{\text{sin}}_T,\beta^{\text{te}})=\widetilde{\mathcal{O}}(\frac{1}{T})$ and  $R(\overline{\boldsymbol{\omega}}_T,\beta^{\text{te}})=\widetilde{\mathcal{O}}(\frac{1}{T})$.
  \end{enumerate}
\end{proposition}



\subsection{On the Role of Adaptation Learning Rate}\label{sec-main-stopping}

The analysis in \cite{bernacchia2021meta} suggests a surprising observation that a negative learning rate (i.e., when $\beta^{\text{tr}}$ takes a negative value) optimizes the generalization for MAML under mixed linear regression models. Their results indicate that the testing risk initially increases and then decreases as $\beta^{\text{tr}}$ varies from negative to positive values around zero for Gaussian isotropic input data and tasks. Our following proposition supports such a trend, but with a novel tradeoff in SGD dynamics as a new reason for the trend, under more general data distributions. 
Denote $\overline{\boldsymbol{\omega}}^{\beta}_T$ as the average SGD solution of MAML after $T$ iterations that uses $\beta$ as the inner loop learning rate.
\begin{proposition}\label{prop-tradeoff}
  Let $s=T\log^{-p}(T)$ and $d=T\log^{q}(T)$, where $p,q>0$. 
  If the spectrum of $\bSigma$ satisfies
$$\lambda_{k}= \begin{cases}1 / s, & k \leq s \\ 1 /(d-s), & s+1 \leq k \leq d. \end{cases}
  $$
Suppose the spectral parameter $\nu_i$ of $\bSigma_{\btheta}$ is $O(1)$, and let the step size $\alpha=\frac{1}{2 c(\btr, \bSigma) \operatorname{tr}(\bSigma)}$. Then for large $n_1$, $|\beta^{\text{tr}}|, |\beta^{\text{te}}|<\frac{1}{\lambda_1}$, we have 
 \begin{align}\label{eq-tradeoff}
 R(\overline{\boldsymbol{\omega}}^{\beta^{\text{tr}}}_T,\bte)\lesssim  
    & \mathcal{O}\Big(\frac{1}{\log^{p}(T)}\Big) \frac{1}{(1-\btr \lambda_{1})^{2}}+\mathcal{O}\Big(\frac{1}{\log^{q} (T)}\Big)\Big(1-\btr \lambda_{d}\Big)^{2}
    +\widetilde{\mathcal{O}}(\frac{1}{T}).
\end{align}
\end{proposition}
The first two terms in the bound of \cref{eq-tradeoff} correspond to the impact of effective meta weights $\Xi_i$ on the "leading" and "tail" spaces, respectively, as we discuss in \Cref{remark-weight}. Clearly, the learning rate $\btr$ plays a tradeoff role in these two terms, particularly when $p$ is close to $q$. This explains the fact that the test error first increases and then decreases as $\btr$ varies from negative to positive values around zero. Such a tradeoff also serves as the reason for the first-increase-then-decrease trend of the test error under more general data distributions as we demonstrate in \Cref{fig:tradeoff}. This complements the reason suggested in \cite{bernacchia2021meta}, which captures only the quadratic form $\frac{1}{\left(1-\btr \lambda_{1}\right)^{2}}$ of $\btr$ for isotropic $\bSigma$, where there exists only the "leading" space without "tail" space.

\begin{figure}[ht]
  \centering
  \begin{subfigure}{.31\textwidth}
      \centering
      \includegraphics[width=\linewidth]{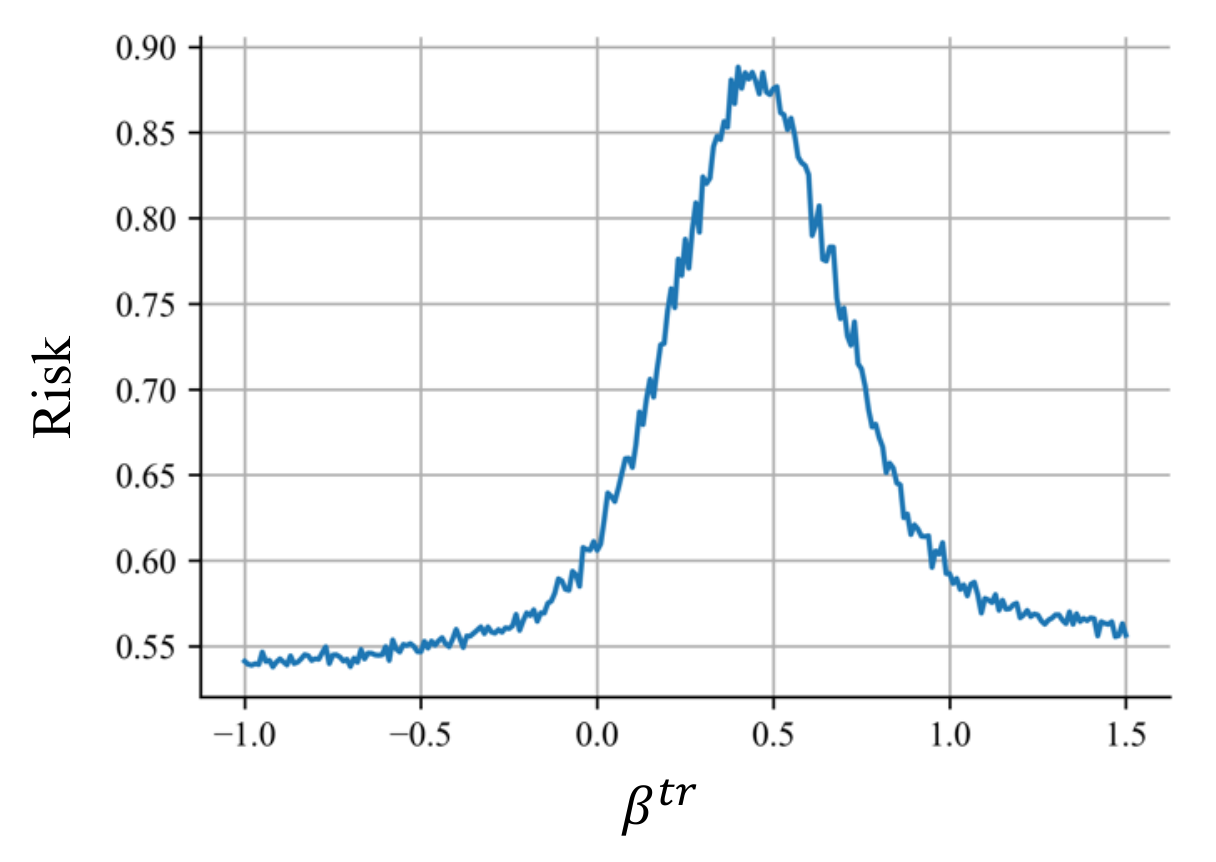}
    \caption{$\lambda_i=\frac{1}{i\log(i+1)^2}$}
    \label{fig:tradeoff:1} 
    \end{subfigure}
  \begin{subfigure}{.31\textwidth}
      \centering
    \includegraphics[width=\linewidth]{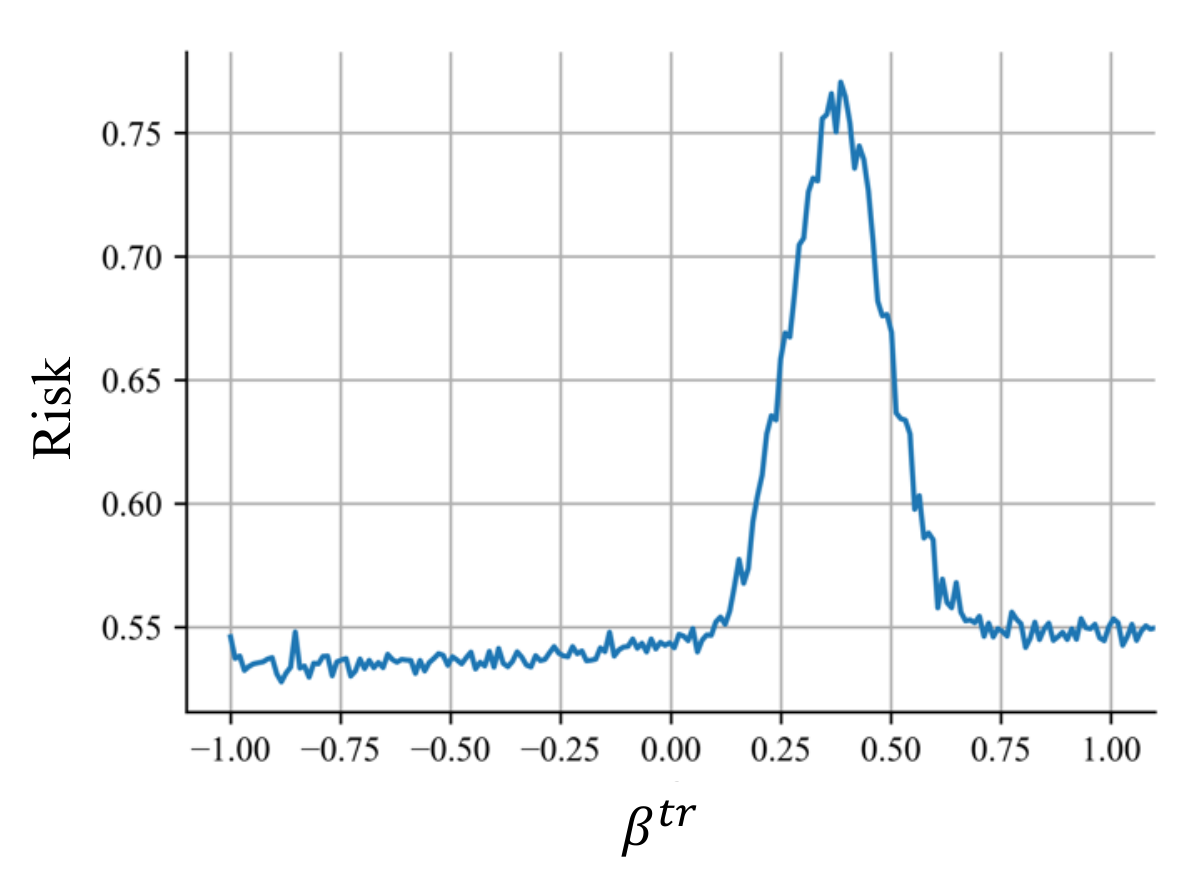}
    \caption{$\lambda_i=\frac{1}{i\log(i+1)^3}$}
    \label{fig:tradeoff:2} 
    \end{subfigure}
    \begin{subfigure}{.31\textwidth}
      \centering
    \includegraphics[width=\linewidth]{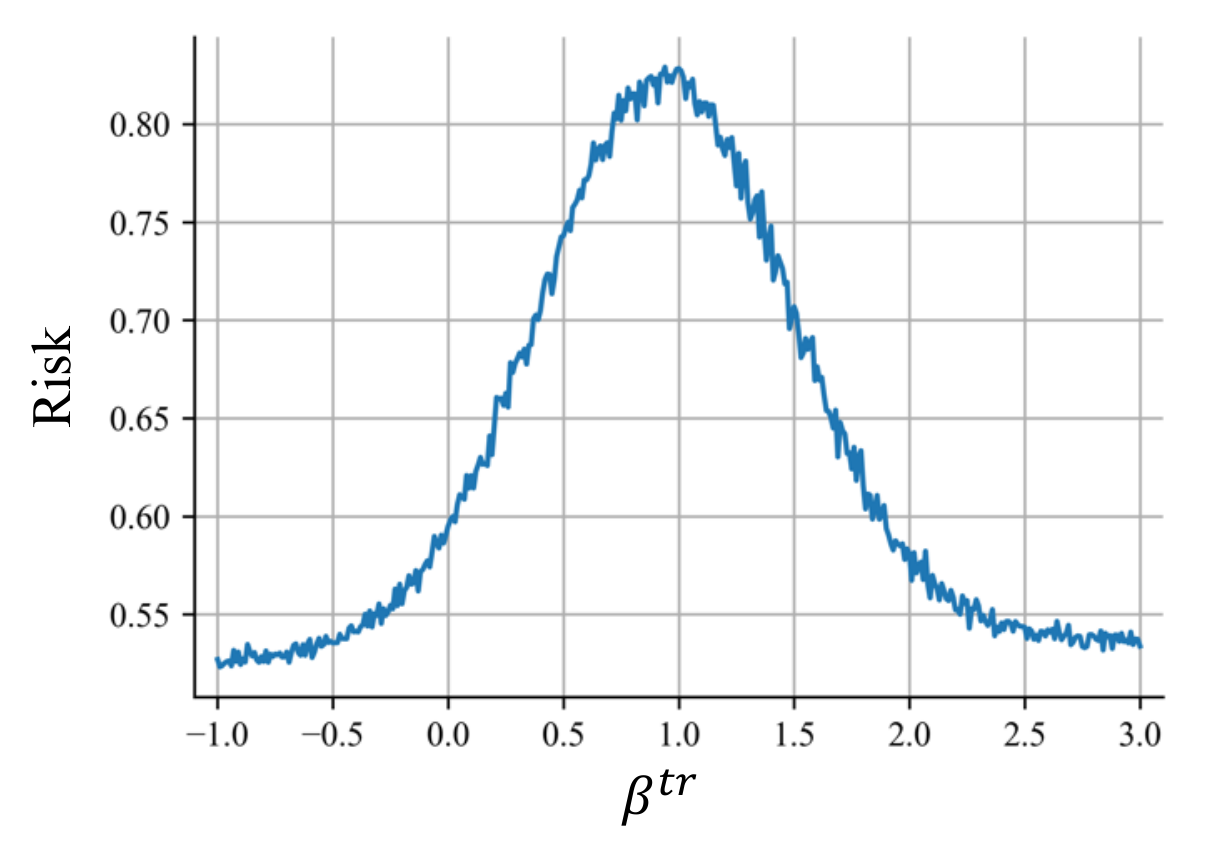}
    \caption{$\lambda_i=\frac{1}{i^2}$}
    \label{fig:tradeoff:3} 
    \end{subfigure}
    \caption{$R(\overline{\boldsymbol{\omega}}^{\beta^{\text{tr}}}_T,\bte)$ as a function of $\beta^{\text{tr}}$. $d=200$, $T=100$, $\bSigma_{\btheta}=\frac{0.8^2}{d}\mathbf{I}$, $\bte=0.2$
    }
    \label{fig:tradeoff} 
  \end{figure}

Based on the above results, incorporating with our dynamics analysis, we surprisingly find that $\btr$ not only affects the final risk, but also plays a pivot role towards the early iteration that the testing error tends to be steady. To formally study such a property, we define the stopping time as follows.
\begin{definition}[Stopping time]
  Given $\btr,\bte$, for any $\epsilon>0$, the corresponding stopping time $t_{\epsilon}(\btr,\bte)$ is defined as:
\[
  t_{\epsilon}(\btr,\bte) = \min t\quad 
\text{s.t. }\;  R(\wl^{\btr}_t;\bte)<\epsilon.
\]
\end{definition}
In the sequel, we may omit the arguments in $t_{\epsilon}$ for simplicity. We consider the similar data distribution in \Cref{prop-tradeoff} but parameterized by $K$, i.e., $s=K\log^{-p}(K)$ and $d=K\log^{q}(K)$, where $p,q> 0$. Then we can derive the following characterization for $t_{\epsilon}$.
\begin{corollary}~\label{col-stop}
If the assumptions in \Cref{prop-tradeoff} hold and $p=q$. Further, let $\bSigma_{\btheta}=\eta^2\mathbf{I}$, and $|\beta^{\text{tr}}|<\frac{1}{\lambda_1}$. Then for $t_{\epsilon}(\btr,\bte)\in (s, K]$, we have:
  \begin{align}
\exp\Big(\epsilon^{-\frac{1}{p}} \Big[\frac{L_{l}}{(1-\btr\lambda_1)^2}+ L_{t} (1-\btr\lambda_d)^2\Big]^{\frac{1}{p}}\Big) 
 \leq  t_{\epsilon}\leq    \exp\Big(\epsilon^{-\frac{1}{p}}\Big[\frac{U_{l}}{(1-\btr\lambda_1)^2}+ U_{t} (1-\btr\lambda_d)^2\Big]^{\frac{1}{p}}\Big) \label{eq-stopping}
  \end{align}
  where $L_l$, $L_t$, $U_l$, $U_t>0$ are factors for "leading" and "tail" spaces that are independent of $K$\footnote{Such terms have been suppressed for clarity. Details are presented in the appendix.}.
\end{corollary}
\Cref{eq-stopping} suggests that the early stopping time $t_{\epsilon}$ is also controlled by the tradeoff role that $\btr$ plays in the "leading" ($U_l,L_l$) and "tail" spaces ($U_t,L_t$), which takes a similar form as the bound in \Cref{prop-tradeoff}. Therefore, the trend for $t_{\epsilon}$ in terms of $\btr$ will exhibit similar behaviours as the final excess risk, and hence the optimal $\btr$ for the final excess risk will lead to an earliest stopping time. We plot the training and test errors for different $\btr$ in Figure~\ref{fig:stopping}, under the same data distributions as \Cref{fig:tradeoff:1} to validate our theoretical findings. As shown in \Cref{fig:stopping:1}, $\btr$ does not make much difference in the training stage (the process converges for all $\btr$ when $T$ is larger than $100$). However, in \Cref{fig:stopping:2} at test stage, $\btr$ significantly affects the iteration when the test error starts to become relatively flat.
Such an early stopping time first increases then decreases as $\btr$ varies from $-0.5$ to $0.7$, which resembles the change of final excess risk in \Cref{fig:tradeoff:1}.
\begin{figure}[H]
  \centering
  \begin{subfigure}{.4\textwidth}
      \centering
      \includegraphics[width=\linewidth]{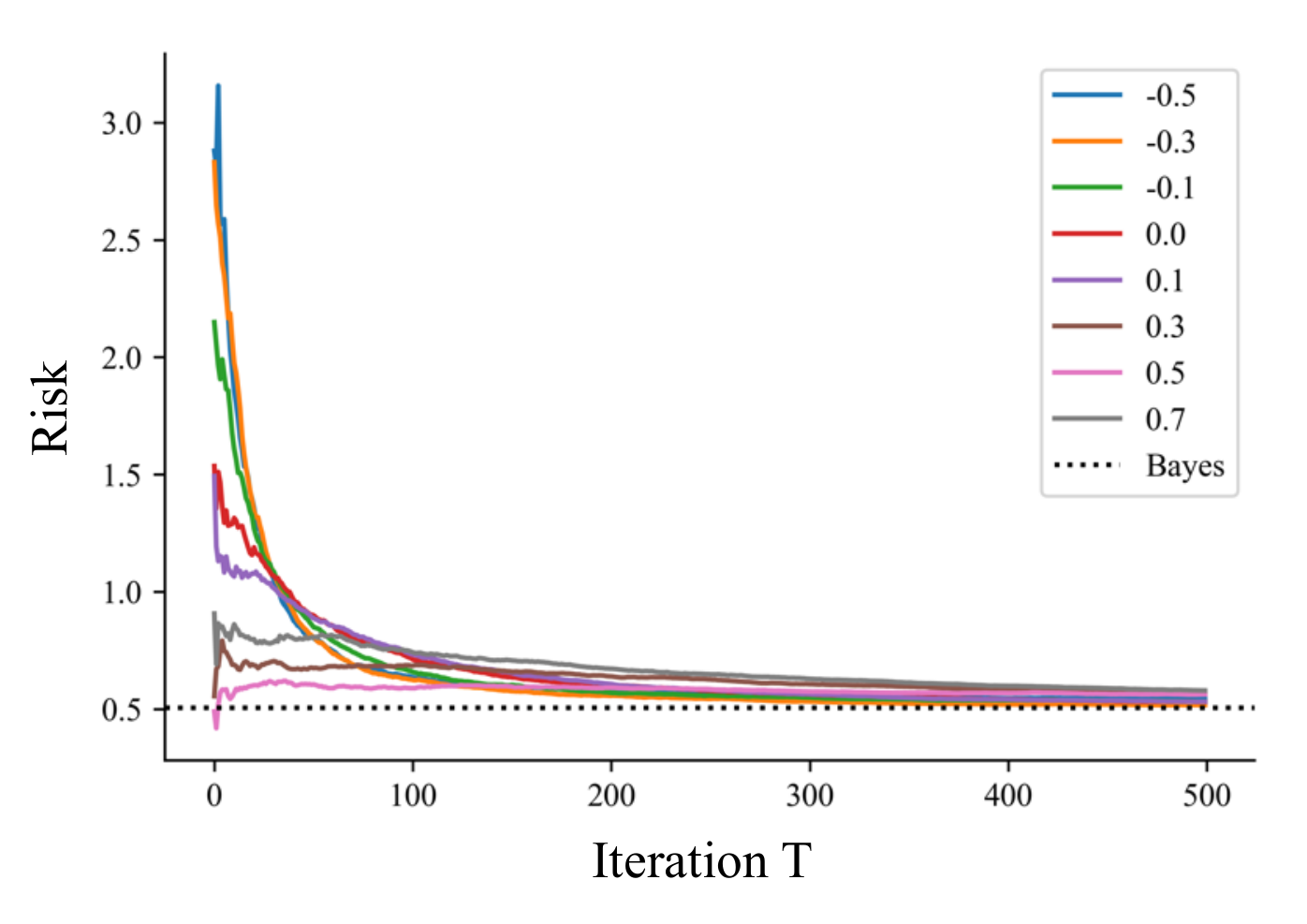}
    \caption{Training Risk}
    \label{fig:stopping:1} 
    \end{subfigure}
  \begin{subfigure}{.4\textwidth}
      \centering
    \includegraphics[width=\linewidth]{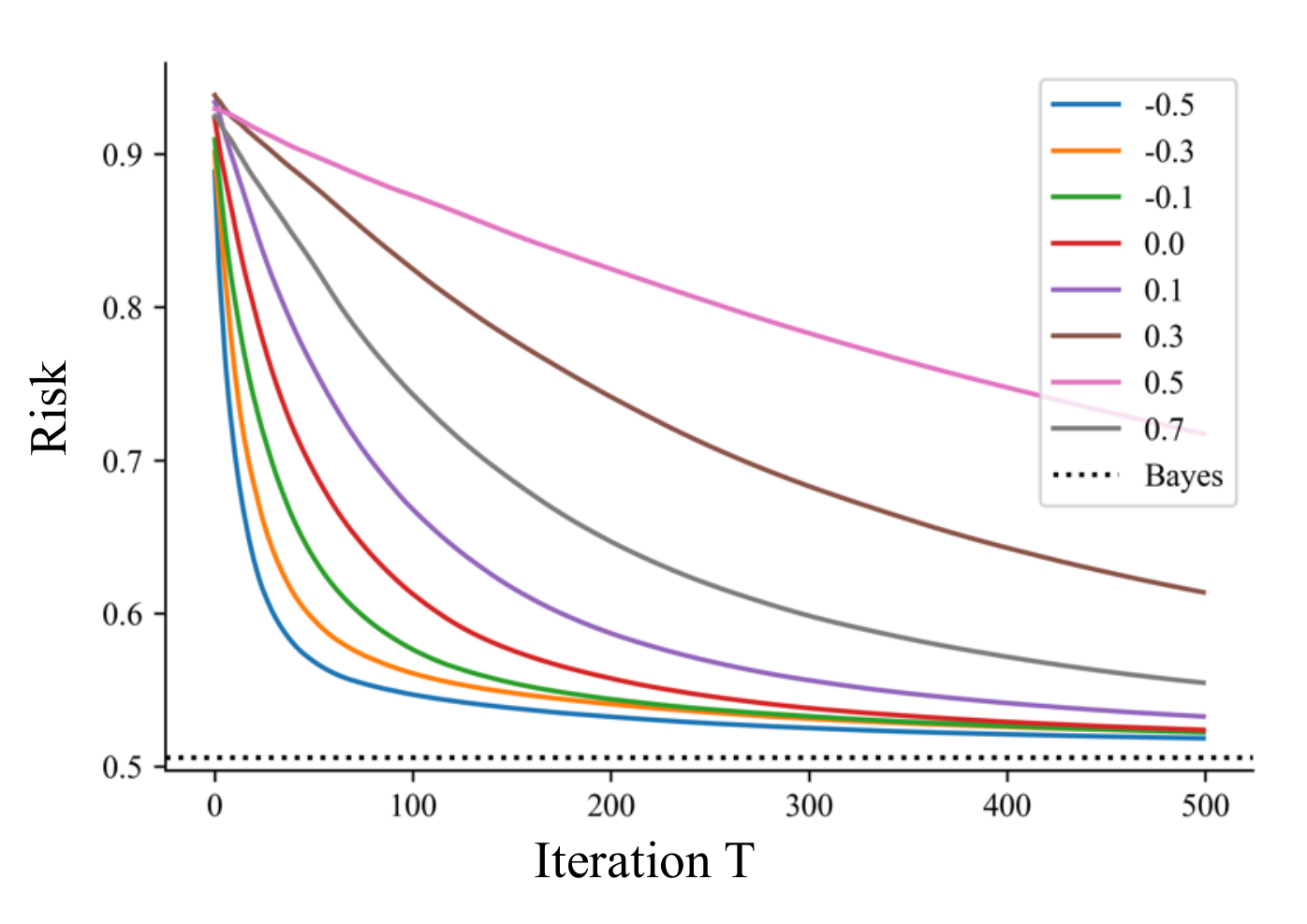}
    \caption{Test Error}
    \label{fig:stopping:2} 
    \end{subfigure}
    \caption{Training and test curves for different $\beta^{\text{tr}}$. $d=500$, $\lambda_i=\frac{1}{i\log^2(i+1)}$,$\bSigma_{\btheta}=\frac{0.8^2}{d}\mathbf{I}$, $\bte=0.2$
    }
    \label{fig:stopping} 
  \end{figure}
\section{Conclusions }\label{sec-conclusion}
In this work, we give the theoretical treatment towards the generalization property of MAML based on their optimization trajectory in non-asymptotic and overparameterized regime.
We provide both upper and lower bounds on the excess risk of MAML trained with average SGD. Furthermore, we explore which type of data and task distributions are
crucial for diminishing error with overparameterization, and discover the influence of adaption learning rate both on the generalization error and the dynamics, which brings novel insights towards the distinct effects of MAML's one-step gradient updates on "leading" and "tail" parts of data eigenspace.

\bibliographystyle{plain}
\bibliography{main}

\appendix
\newpage
\renewcommand{\appendixpagename}{\centering \sffamily \LARGE Appendix}

\appendixpage


\vspace{5mm}
\section{Proof of  Proposition~\ref{prop1}}
We first show how to connect the loss function associated with MAML to a Meta Least Square Problem.
 \begin{proposition}[\Cref{prop1} Restated] Under the mixed linear regression model, 
 the expectation of the meta-training loss 
 taken over task and data distributions can be rewritten as:
  \begin{align}
    \mathbb{E}\left[\widehat{\mathcal{L}}(\mathcal{A},\boldsymbol{\omega},\beta^{\text{tr}};\mathcal{D})\right]= \mathcal{L}(\mathcal{A},\boldsymbol{\omega},\beta^{\text{tr}})=
  \mathbb{E}_{\Bb , \boldsymbol{\gamma}}  \frac{1}{2}\left[\left\|\Bb\boldsymbol{\omega}-\boldsymbol{\gamma}\right\|^{2}\right].
\end{align}
The meta data are given by
\begin{align}\Bb =& \frac{1}{\sqrt{n_2}}\Xb^{out}\Big(\mathbf{I}-\frac{\beta^{\text{tr}}}{n_1} {\Xb^{\text{in}}}^{T} {\Xb^{\text{in}}}\Big)\label{ap-eq-data}\\ \boldsymbol{\gamma} =&  \frac{1}{\sqrt{n_2}}\Big(  \Xb^{\text{out}}\Big( \mathbf{I}-\frac{\beta^{\text{tr}}}{n_1} {\Xb^{\text{in}}}^{T} {\Xb^{\text{in}}}\Big)\boldsymbol{\theta}+\zb^{out}-\frac{\beta^{\text{tr}}}{n_1} \Xb^{\text{out}}{\Xb^{\text{in}}}^{\top}\zb^{\text{in}}\Big) \end{align}
where $\Xb^{\text{in}}\in \mathbb{R}^{n_1\times d}$,$\zb^{\text{in}}\in \mathbb{R}^{n_1}$,$\Xb^{\text{out}}\in \mathbb{R}^{n_2\times d}$ and $\zb^{\text{out}}\in \mathbb{R}^{n_2}$ denote the inputs and noise for training and validation.
Furthermore,
we have
\begin{align}
  \boldsymbol{\gamma} = \Bb\boldsymbol{\theta}^{*}+\boldsymbol{\xi}\quad \text{ with meta noise } \mathbb{E}[\boldsymbol{\xi}\mid\Bb]=0. \label{ap-linear}
 \end{align} 
 \end{proposition}
 \begin{proof}
We first rewrite $\mathcal{L}(\mathcal{A},\boldsymbol{\omega},\beta^{\text{tr}})$ as follows:
\begin{align*}
    \mathcal{L}(\mathcal{A},\boldsymbol{\omega},\beta^{\text{tr}})
    &=\mathbb{E}\left[ \ell(\mathcal{A}(\boldsymbol{\omega},\btr;\mathcal{D}^{\text{in}});\mathcal{D}^{\text{out}})\right]\\
    &=\mathbb{E}\left[ \frac{1}{2n_2}\sum^{n_2}_{j=1} \left(\left\langle \mathbf{x}^{\text{out}}_{j}, \mathcal{A}(\boldsymbol{\omega},\btr;\mathcal{D}^{\text{in}})\right\rangle-y^{\text{out}}_{j}\right)^{2}\right]\\
    &=\mathbb{E}\left[ \frac{1}{2n_2}\|\Xb^{\text{out}}\Big(\mathbf{I}-\frac{\beta^{\text{tr}}}{n_1} {\Xb^{\text{in}}}^{T} {\Xb^{\text{in}}}\Big)\bomega+\frac{\btr}{n_1}{\Xb^{\text{in}}}^{T}\yb^{\text{in}}-\yb^{\text{out}} \|^2\right].
\end{align*} 
Using the mixed linear model:
\begin{align}
     \yb^{\text{in}}= \mathbf{X}^{\text{in}}\boldsymbol{\theta}+\zb^{\text{in}},\quad  \yb^{\text{out}}= \mathbf{X}^{\text{out}}\boldsymbol{\theta}+\zb^{\text{out}}, 
\end{align}
we have 
\begin{align*}
    \mathcal{L}&(\mathcal{A},\boldsymbol{\omega},\beta^{\text{tr}})\\
    &=\mathbb{E}\left[ \frac{1}{2n_2}\|\Xb^{\text{out}}\Big(\mathbf{I}-\frac{\beta^{\text{tr}}}{n_1} {\Xb^{\text{in}}}^{T} {\Xb^{\text{in}}}\Big)\bomega\right.\\
    &-\left. \Big(  \Xb^{\text{out}}\Big( \mathbf{I}-\frac{\beta^{\text{tr}}}{n_1} {\Xb^{\text{in}}}^{T} {\Xb^{\text{in}}}\Big)\boldsymbol{\theta}+\zb^{out}-\frac{\beta^{\text{tr}}}{n_1} \Xb^{\text{out}}{\Xb^{\text{in}}}^{\top}\zb^{\text{in}}\Big)\|^2\right]\\
    &=   \mathbb{E}_{\Bb , \boldsymbol{\gamma}}  \frac{1}{2}\left[\left\|\Bb\boldsymbol{\omega}-\boldsymbol{\gamma}\right\|^{2}\right].
\end{align*} 
Moreover, note that $\btheta-\btheta^{*}$ has mean zero and is independent of data and noise, and define 
\begin{align}
    \boldsymbol{\xi}=\frac{1}{\sqrt{n_2}}\left(  \Xb^{\text{out}}\left( \mathbf{I}-\frac{\btr}{n_1} {\Xb^{\text{in}}}^{T} {\Xb^{\text{in}}}\right)(\btheta-\btheta^{*})+\zb^{\text{out}}-\frac{\btr}{n_1} \Xb^{\text{out}}{\Xb^{\text{in}}}^{\top}\zb^{in}\right).\label{ap-eq-xi}
\end{align}
We call $\boldsymbol{\xi}$ as meta noise, and then we have
\begin{align*}
     &\boldsymbol{\gamma} = \Bb\boldsymbol{\theta}^{*}+\boldsymbol{\xi}\quad \text{ and }\quad  \mathbb{E}[\boldsymbol{\xi}\mid\Bb]=0.
\end{align*}
\end{proof}
 \begin{lemma}[Meta Excess Risk]\label{ap-lemma-excess}
  Under the mixed linear regression model, the meta excess risk can be rewritten as follows:
  \begin{align*}
     R(\wl_T, \bte)=\frac{1}{2}\mathbb{E}\|\wl_T-\btheta^{*}\|^2_{\Hte}
  \end{align*}
  where $\|\ab\|_{\Ab}^{2}=\ab^{T} \Ab \ab$. Moreover, the Bayes error is given by
   \begin{align*}
  \mathcal{L}(\mathcal{A},\boldsymbol{\omega}^{*},\beta^{\text{te}})=\frac{1}{2}\operatorname{tr}(\bSigma_{\btheta}\Hte)+\frac{\sigma^2{\bte} ^2}{2m}+\frac{\sigma^2}{2}.
  \end{align*}
 \end{lemma}
 \begin{proof}
 Recall that
 \begin{align*}
R(\overline{\boldsymbol{\omega}}_T,\beta^{\text{te}})\triangleq   \mathbb{E}\left[\mathcal{L}(\mathcal{A},\overline{\boldsymbol{\omega}}_T,\beta^{\text{te}})\right]-\mathcal{L}(\mathcal{A},\boldsymbol{\omega}^{*},\beta^{\text{te}})
\end{align*}
where $\boldsymbol{\omega}^{*}$ denotes the optimal solution to the population meta-test error. Under the mixed linear model, such a solution can be directly calculated~\cite{gao2020modeling}, and we obtain $\boldsymbol{\omega}^{*}=\mathbb{E}[\btheta] =\btheta^{*}$. Hence, 
\[R(\wl_T, \bte)=\mathbb{E}_{\Bb , \boldsymbol{\gamma}}  \frac{1}{2}\left[\left\|\Bb\wl_{T}-\boldsymbol{\gamma}\right\|^{2}-\left\|\Bb\btheta^{*}-\boldsymbol{\gamma}\right\|^{2}\right],\] 
where \begin{align}\Bb =&  {\xb^{\text{out}}}^{\top}\Big(\mathbf{I}-\frac{\beta^{\text{te}}}{m} {\Xb^{\text{in}}}^{T} {\Xb^{\text{in}}}\Big)\nonumber\\ \boldsymbol{\gamma} =&   {\xb^{\text{out}}}^{\top}\Big( \mathbf{I}-\frac{\beta^{\text{te}}}{m} {\Xb^{\text{in}}}^{T} {\Xb^{\text{in}}}\Big)\boldsymbol{\theta}+\zb^{\text{out}}-\frac{\beta^{\text{te}}}{m} {\xb^{\text{out}}}^{\top}{\Xb^{\text{in}}}^{\top}\zb^{\text{in}}, \end{align}
and
$\xb^{\text{out}}\in\mathbb{R}^{d}$, $\zb^{\text{out}}\in\mathbb{R}^{d}$, $\Xb^{\text{in}}\in\mathbb{R}^{m\times d}$ and $\zb^{\text{in}}\in\mathbb{R}^{m}$. The forms of $\Bb$ and $\bgamma$ are slightly different since we allow a new adaptation rate $\bte$ and the inner loop has $m$ samples at test stage. Similarly
\begin{align}
   \xi=\left(\underbrace{  {\xb^{\text{out}}}^{\top}\left( \mathbf{I}-\frac{\bte}{m} {\Xb^{\text{in}}}^{T} {\Xb^{\text{in}}}\right)(\btheta-\btheta^{*})}_{\xi_1}+\underbrace{\zb^{\text{out}}}_{\xi_2}\underbrace{-\frac{\btr}{m} {\xb^{\text{out}}}^{\top}{\Xb^{\text{in}}}^{\top}\zb^{in}}_{\xi_3}\right).
\end{align}
Then we have
\begin{align*}
    R(\wl_T, \bte)&=\mathbb{E}_{\Bb , \boldsymbol{\gamma}} \frac{1}{2}\left[\left\|\Bb\wl_{T}-\boldsymbol{\gamma}\right\|^{2}-\left\|\Bb\btheta^{*}-\boldsymbol{\gamma}\right\|^{2}\right]\\
    &=\mathbb{E}_{\Bb , \boldsymbol{\gamma}} \frac{1}{2}\left[\| \Bb(\wl_{T}-\btheta^{*})\|^2\right]\\
    &=\frac{1}{2}\mathbb{E}\|\wl_T-\btheta^{*}\|^2_{\Hte}
\end{align*}
where the last equality follows because $\mathbb{E}\left[\Bb^{\top}\Bb\right]=\Hte$ at the test stage.

The Bayes error can be calculated as follows:
\begin{align*}
    \mathcal{L}(\mathcal{A},\boldsymbol{\omega}^{*},\beta^{\text{te}})&= \mathbb{E}_{\Bb , \boldsymbol{\gamma}}  \frac{1}{2}\left[\left\|\Bb\btheta^{*}-\boldsymbol{\gamma}\right\|^{2}\right]=\mathbb{E}_{\Bb , \boldsymbol{\gamma}}  \frac{1}{2}\left[\xi^2\right]\\
    &\overset{(a)}{=}\frac{1}{2}\left(\mathbb{E}\left[\xi_1^{2}\right]+\mathbb{E}\left[\xi_2^{2}\right]+\mathbb{E}\left[\xi_3^{2}\right]\right)\\
    &=\frac{1}{2}(\operatorname{tr}(\bSigma_{\btheta}\Hte)+\frac{{\bte}^2\sigma^2}{m}+\sigma^2)
\end{align*}
where $(a)$ follows because $\xi_1,\xi_2,\xi_3$ are independent and have zero mean conditioned on $\Xb^{\text{in}}$ and $\xb^{\text{out}}$.
 \end{proof}

\section{Analysis for Upper Bound (Theorem~\ref{thm-upper}) }
\subsection{Preliminaries}
We first introduce some additional notations. 
\begin{definition}[Inner product of matrices]
For any two matrices $\Cb,\Db$, the inner product of them is defined as 
$$
\langle \Cb,\Db \rangle = \operatorname{tr}(\Cb^{\top}\Db).
$$
\end{definition}
We will use the following property about the inner product of matrices throughout our proof.
\begin{property}
If $\mathbf{C} \succeq 0$ and $\mathbf{D} \succeq \mathbf{D}^{\prime}$, then we have $\langle\mathbf{C}, \mathbf{D}\rangle \geq\left\langle\mathbf{C}, \mathbf{D}^{\prime}\right\rangle$.
\end{property}
\begin{definition}[Linear operator]
Let $\otimes$ denote the tensor product. Define the following linear operators on symmetric matrices:
$$
\begin{gathered}
    \mathcal{M}=\mathbb{E}\left[ \Bb^{\top}\otimes\Bb^{\top}\otimes \Bb\otimes\Bb\right]\quad
    \widetilde{\mathcal{M}}:= \Htr\otimes \Htr
    \quad 
   \mathcal{I}:= \mathbf{I} \otimes \mathbf{I} 
   \\ \mathcal{T}:=\Hb_{n_1,\btr} \otimes \mathbf{I}+\mathbf{I} \otimes \Hb_{n_1,\btr}-\alpha \mathcal{M}, \quad \widetilde{\mathcal{T}}=\Hb_{n_1,\btr} \otimes \mathbf{I}+\mathbf{I} \otimes \Hb_{n_1,\btr}-\alpha \Hb_{n_1,\btr} \otimes \Hb_{n_1,\btr}.
\end{gathered}
$$
\end{definition}
We next define the operation of the above linear operators on a symmetric matrix $\Ab$ as follows.
$$
\begin{gathered}
 \mathcal{M} \circ \mathbf{A}=\mathbb{E}\left[\mathbf{B}^{\top}\Bb \mathbf{A} \Bb^{\top}\Bb\right], \quad \widetilde{\mathcal{M}} \circ \mathbf{A}=\Htr \mathbf{A} \Htr, \quad \mathcal{I} \circ \mathbf{A}=\mathbf{A},  
 \\
 \mathcal{T}\circ \Ab = \Htr\Ab+\Ab\Htr -\alpha \mathbb{E}\left[\mathbf{B}^{\top}\Bb \mathbf{A} \Bb^{\top}\Bb\right]\\
 \tilde{\mathcal{T}}\circ \Ab =  \Htr\Ab+\Ab\Htr -\alpha\Htr\Ab\Htr.
\end{gathered}
$$
Based on the above definitions, we have the following equations hold.
$$
\begin{gathered}
(\mathcal{I}-\alpha \mathcal{T}) \circ \mathbf{A}=\mathbb{E}\left[\left(\mathbf{I}-\alpha \Bb^{\top}\Bb\right) \mathbf{A}\left(\mathbf{I}-\alpha \Bb^{\top}\Bb\right)\right]\\(\mathcal{I}-\alpha \tilde{\mathcal{T}}) \circ \mathbf{A}=(\mathbf{I}-\alpha \Htr) \mathbf{A}(\mathbf{I}-\alpha \Htr).
\end{gathered}
$$
For the linear operators, we have the following technical lemma.
\begin{lemma}\label{lemma-linearop} We call the linear operator $\mathcal{O}$ a PSD mapping, if for every symmetric PSD matrix $\mathbf{A}$,  $\mathcal{O}\circ \Ab$ is also PSD matrix. Then we have:
\begin{enumerate}[label=\roman*]
    \item[(i)]    $ \mathcal{M}$, $\widetilde{\mathcal{M}}$ and 
    $(\mathcal{M}-\widetilde{\mathcal{M}}) $  are all PSD mappings.
    \item[(ii)]  $\tilde{\mathcal{T}}-\mathcal{T}$,  $\mathcal{I}-\alpha \mathcal{T}$ and $\mathcal{I}-\alpha \tilde{\mathcal{T}}$ are all PSD mappings.
    \item[(iii)] If $0<\alpha< \frac{1} { \max_{i}\{\mu_{i}(\Htr)\}}$, then $\tilde{\mathcal{T}}^{-1}$ exists, and is a PSD mapping.
    \item[(iv)]  If $0<\alpha<\frac{1} { \max_{i}\{\mu_{i}(\Htr)\}}$, $\tilde{\mathcal{T}}^{-1} \circ \Htr\preceq \mathbf{I}$.
    \item[(v)] If $0<\alpha<\frac{1}{c(\btr,\bSigma) \operatorname{tr}(\bSigma
    )}$, then $\mathcal{T}^{-1} \circ \mathbf{A}$ exists for PSD matrix $\mathbf{A}$, and $\mathcal{T}^{-1}$ is a PSD mapping.
\end{enumerate}
\end{lemma}
\begin{proof}
Items (i) and (iii) directly follow from the proofs in \cite{jain2017markov,zou2021benign}.
For $(\romannum{4})$, by the existence of $\tilde{\mathcal{T}}^{-1}$, we have
\begin{align*}
    \tilde{\mathcal{T}}^{-1} \circ \Htr&=\sum_{t=0}^{\infty} \alpha (\mathcal{I}- \alpha\tilde{\mathcal{T}})^{t}\circ \Htr\\
    &=\sum_{t=0}^{\infty} \alpha (\mathbf{I}-\alpha \Htr)^{t}\Htr(\mathbf{I}-\alpha \Htr)^{t}\\
    &\preceq \sum_{t=0}^{\infty} \alpha (\mathbf{I}-\alpha \Htr)^{t}\Htr=\mathbf{I}.
\end{align*}
For $(\romannum{5})$, for any PSD matrix $\mathbf{A}$, consider 
$$
\mathcal{T}^{-1} \circ \mathbf{A}=\alpha \sum_{k=0}^{\infty}(\mathcal{I}-\alpha \mathcal{T})^{k} \circ \mathbf{A}.
$$

We first show that $\sum_{k=0}^{\infty}(\mathcal{I}-\alpha \mathcal{T})^{k} \circ \mathbf{A}$ is finite, and then it suffices to show that the trace is finite, i.e.,
\begin{align}
   \sum_{k=0}^{\infty} \operatorname{tr}\left((\mathcal{I}-\alpha \mathcal{T})^{k} \circ \mathbf{A}\right)<\infty. \label{b1-fin}
\end{align}
 Let $\mathbf{A}_k=(\mathcal{I}-\gamma \mathcal{T})^{k} \circ \mathbf{A}$. Combining with the definition of $\mathcal{T}$, we obtain
$$
\begin{aligned}
\operatorname{tr}\left(\mathbf{A}_{k}\right) &=\operatorname{tr}\left(\mathbf{A}_{k-1}\right)-2\alpha \operatorname{tr}\left(\Htr \mathbf{A}_{k-1}\right)+\alpha^{2} \operatorname{tr}\left(\mathbf{A} \mathbb{E}\left[\Bb^{\top}\Bb \Bb^{\top}\Bb \right]\right). 
\end{aligned}
$$
Letting $\Ab=\Ib$ in \Cref{prop-4}
, we have $\mathbb{E}\left[\Bb^{\top}\Bb \Bb^{\top}\Bb \right]\preceq c(\btr,\bSigma
) \operatorname{tr}(\bSigma) \Htr$. Hence
$$
\begin{aligned}
\operatorname{tr}\left(\mathbf{A}_{k}\right) & \leq \operatorname{tr}\left(\mathbf{A}_{k-1}\right)-\left(2 \alpha-\alpha^{2} c(\btr,\bSigma
) \operatorname{tr}(\bSigma)\right) \operatorname{tr}\left(\Htr \mathbf{A}_{k-1}\right) \\
& \leq \operatorname{tr}\left((\mathbf{I}-\alpha \Htr) \mathbf{A}_{k-1}\right)\quad \text{ by } \alpha< \frac{1}{c(\btr,\bSigma) \operatorname{tr}(\bSigma
    )}\\
& \leq\left(1-\alpha \min_{i}\{\mu_{i}(\Htr)\}\right) \operatorname{tr}\left(\mathbf{A}_{k-1}\right).
\end{aligned}
$$
If $\alpha<\frac{1}{\min_{i}\{\mu_{i}(\Htr)\}}$, then we substitute it into \cref{b1-fin} and obtain
$$
\sum_{k=0}^{\infty} \operatorname{tr}\left((\mathcal{I}-\alpha \mathcal{T})^{k} \circ \mathbf{A}\right)=\sum_{k=0}^{\infty} \operatorname{tr}\left(\mathbf{A}_{k}\right) \leq \frac{\operatorname{tr}(\mathbf{A})}{\alpha\min_{i}\{\mu_{i}(\Htr)\}}<\infty
$$
which guarantees the existence of  $\mathcal{T}^{-1}$. Moreover, $\Ab_k$ is a PSD matrix for every $k$ since $\mathcal{I}-\alpha \mathcal{T}$ is a PSD mapping. The $\mathcal{T}^{-1} \circ \mathbf{A}=\alpha \sum_{k=0}^{\infty} \Ab_k$ must be a PSD matrix, which implies that $\mathcal{T}^{-1}$ is PSD mapping.
\end{proof}
\begin{property}[Commutity]
Suppose Assumption $2$ holds, then for all $n>0$, $|\beta|<1/\lambda_1$, $\mathbf{H}_{n,\beta}$ with different $n$ and $\beta$ commute with each other.
\end{property}
\subsection{Fourth Moment Upper Bound for Meta Data}
In this section, we provide a technical result for the fourth moment of meta data $\Bb$, which is essential throughout the proof of our upper bound.
\begin{proposition}\label{prop-4}
         Suppose Assumptions 1-3 hold. Given $|\beta|<\frac{1}{\lambda_1}$, for any PSD matrix $\Ab$,  we have 
         \begin{align*}
   \mathbb{E}\left[\Bb^{\top}\Bb \mathbf{A} \Bb^{\top}\Bb \right]\preceq c(\btr,\bSigma)\mathbb{E}\left[\operatorname{tr}(\mathbf{A} \bSigma) \right] \Htr
\end{align*}
where $c(\beta,\bSigma):= c_1\left(1+ 8|\beta|\lambda_1\sqrt{C(\beta,\bSigma)}\sigma_x^2+64\sqrt{C(\beta,\bSigma)}\sigma_x^4\beta^2\operatorname{tr}(\bSigma^2)\right)$.
\end{proposition}
\begin{proof}
Recall that $\Bb=\frac{1}{\sqrt{n_2}} \Xb^{\text{out}} (\mathbf{I}-\frac{\beta}{n_1}{\mathbf{X}^{\text{in}}}^{\top}\mathbf{X}^{\text{in}})$. With a slight abuse of notations, we write $\btr$ as $\beta$, $\mathbf{X}^{\text{in}}$ as $\mathbf{X}$ in this proof. First consider the case $\beta\geq 0$. By the definition of $\Bb$, we have 
       \begin{align*}
         \mathbb{E}&\left[\Bb^{\top}\Bb \mathbf{A} \Bb^{\top}\Bb \right]\\  &= \mathbb{E}\left[(\mathbf{I}-\frac{\beta}{n_1}\mathbf{X}^{\top}\mathbf{X})\frac{1}{n_2}{\Xb^{\text{out}}}^{\top} \Xb^{\text{out}}(\mathbf{I}-\frac{\beta}{n_1}\mathbf{X}^{\top}\mathbf{X})\mathbf{A} (\mathbf{I}-\frac{\beta}{n_1}\mathbf{X}^{\top}\mathbf{X})\frac{1}{n_2}{\Xb^{\text{out}}}^{\top} \Xb^{\text{out}}(\mathbf{I}-\frac{\beta}{n_1}\mathbf{X}\mathbf{X})\right] \\
        &\preceq c_1 \mathbb{E}\left[\operatorname{tr}\left((\mathbf{I}-\frac{\beta}{n_1}\mathbf{X}^{\top}\mathbf{X})\mathbf{A} (\mathbf{I}-\frac{\beta}{n_1}\mathbf{X}^{\top}\mathbf{X})\bSigma\right) (\mathbf{I}-\frac{\beta}{n_1}\mathbf{X}^{\top}\mathbf{X})\bSigma(\mathbf{I}-\frac{\beta}{n_1}\mathbf{X}^{\top}\mathbf{X})\right]
         \\&\preceq c_1 \mathbb{E}\left[\operatorname{tr}\left(\mathbf{A} (\bSigma+ \frac{\beta^2}{n_1^2} \mathbf{X}^{\top}\mathbf{X}\bSigma \mathbf{X}^{\top}\mathbf{X})\right) (\mathbf{I}-\frac{\beta}{n_1}\mathbf{X}^{\top}\mathbf{X})\bSigma(\mathbf{I}-\frac{\beta}{n_1}\mathbf{X}^{\top}\mathbf{X})\right]
    \end{align*}
    where the second inequality follows from Assumption 1. Let $\xb_i$ denote the $i$-th row of $\Xb$. Note that $\xb_i =\Sigma^{\frac{1}{2}}\zb_i$, where $\zb_i$ is independent $\sigma_x$-sub-gaussian vector.
    For any $\xb_{i_1},\xb_{i_2},\xb_{i_3},\xb_{i_4}$, where $1\leq i_1,i_2,i_3,i_4\leq n_1 $, we have:
\begin{align*}
   &\mathbb{E}\left[\operatorname{tr}(\Ab\xb_{i_1}\xb_{i_2}^{\top}\bSigma \xb_{i_3}\xb_{i_4}^{\top})(\mathbf{I}-\frac{\beta}{n_1}\mathbf{X}^{\top}\mathbf{X})\bSigma(\mathbf{I}-\frac{\beta}{n_1}\mathbf{X}^{\top}\mathbf{X})\right]\\
   &= \mathbb{E}\left[\operatorname{tr}(\bSigma^{\frac{1}{2}}\Ab\bSigma^{\frac{1}{2}}\zb_{i_1}\zb_{i_2}^{\top}\bSigma^{2}\zb_{i_3}\zb_{i_4}^{\top})(\mathbf{I}-\frac{\beta}{n_1}\mathbf{X}^{\top}\mathbf{X})\bSigma(\mathbf{I}-\frac{\beta}{n   _1}\mathbf{X}^{\top}\mathbf{X})\right]\\
   & = \sum_{k,j}\mu_k\lambda^2_j \mathbb{E}\left[(\zb_{i_4}^{\top}\ub_k)(\zb_{i_1}^{\top}\ub_k)(\zb_{i_4}^{\top}\vb_j)(\zb_{i_1}^{\top}\vb_j)(\mathbf{I}-\frac{\beta}{n_1}\mathbf{X}^{\top}\mathbf{X})\Sigma(\mathbf{I}-\frac{\beta}{n_1}\mathbf{X}^{\top}\mathbf{X})\right]
\end{align*}
where the SVD of $\bSigma^{\frac{1}{2}}\Ab\bSigma^{\frac{1}{2}}$ is $\sum_{j} \mu_{j} \ub_{j}\ub^{\top}_{j}$, the SVD of $\bSigma$ is $\sum_{j} \lambda_{j} \vb_{j}\vb^{\top}_{j}$. For any unit vector $\wb\in\mathbb{R}^{d}$, we have:
\begin{align*}
    \wb^{\top}&\mathbb{E}\left[\Hb^{-\frac{1}{2}}_{n_1,\beta}\operatorname{tr}(\Ab\xb_{i_1}\xb_{i_2}^{\top}\bSigma \xb_{i_3}\xb_{i_4}^{\top})(\mathbf{I}-\frac{\beta}{n}\mathbf{X}^{\top}\mathbf{X})\bSigma(\mathbf{I}-\frac{\beta}{n}\mathbf{X}^{\top}\mathbf{X})\Hb^{-\frac{1}{2}}_{n,\beta}\right]\wb\\
   & \leq \sum_{k,j}\mu_k\lambda^2_j \sqrt{\mathbb{E}\left[\left((\zb_{i_4}^{\top}\ub_k)(\zb_{i_1}^{\top}\ub_k)(\zb_{i_4}^{\top}\vb_j)(\zb_{i_1}^{\top}\vb_j)^2\right)\right]  }  \\
   &\quad\times \sqrt{\mathbb{E}\left[\|\wb^{\top}\Hb^{-\frac{1}{2}}_{n_1,\beta}(\mathbf{I}-\frac{\beta}{n_1}\Xb^{\top}\Xb)\Sigma (\mathbf{I}-\frac{\beta}{n_1}\Xb^{\top}\Xb)\Hb^{-\frac{1}{2}}_{n_1,\beta}\wb\|^2\right]}\\
   &\leq  64\sqrt{C(\beta,\bSigma)}\sigma_x^4 \operatorname{tr}(A\bSigma)\operatorname{tr}(\bSigma^2)
\end{align*}
where the first inequality follows from the Cauchy Schwarz inequality; the last inequality is due to Assumption 3 and the property of sub-Gaussian distributions~\cite{vershynin2018high}.
Therefore, 
\begin{align*}
    \mathbb{E}&\left[\Hb^{-\frac{1}{2}}_{n_1,\beta}\operatorname{tr}(\Ab\xb_{i_1}\xb_{i_2}^{\top}\bSigma \xb_{i_3}\xb_{i_4}^{\top})(\mathbf{I}-\frac{\beta}{n_1}\mathbf{X}^{\top}\mathbf{X})\bSigma(\mathbf{I}-\frac{\beta}{n_1}\mathbf{X}^{\top}\mathbf{X})\Hb^{-\frac{1}{2}}_{n_1,\beta}\right]\\
    &\preceq 64\sqrt{C(\beta,\bSigma)}\sigma_x^4 \operatorname{tr}(A\bSigma^2)\mathbf{I}
\end{align*}
which implies
$$\mathbb{E}\left[\operatorname{tr}(\Ab\xb_{i_1}\xb_{i_2}^{\top}\bSigma \xb_{i_3}\xb_{i_4}^{\top})(\mathbf{I}-\frac{\beta}{n_1}\mathbf{X}^{\top}\mathbf{X})\bSigma(\mathbf{I}-\frac{\beta}{n_1}\mathbf{X}^{\top}\mathbf{X})\right]\preceq 64\sqrt{C(\beta,\bSigma)}\sigma_x^4 \operatorname{tr}(A\bSigma^2) \Hb_{n_1,\beta}.$$
Hence, 
\begin{align*}
   & \mathbb{E}\left[\Bb^{\top}\Bb \mathbf{A} \Bb^{\top}\Bb \right]\\
   &\preceq c_1\mathbb{E}\left[\operatorname{tr}\left(\mathbf{A} (\bSigma+ 64\sqrt{C}\sigma_x^4\beta^2 \bSigma\operatorname{tr}(\bSigma^2))\right) (\mathbf{I}-\frac{\beta}{n}\mathbf{X}^{\top}\mathbf{X})\bSigma(\mathbf{I}-\frac{\beta}{n}\mathbf{X}^{\top}\mathbf{X})\right]\\
   &\preceq c_1(1+64\sqrt{C(\beta,\bSigma)}\sigma_x^4\beta^2\operatorname{tr}(\bSigma^2))\mathbb{E}\left[\operatorname{tr}(\mathbf{A} \bSigma) \right] \mathbf{H}_{n_1,\beta}.
\end{align*}
Now we turn to $\beta<0$, and derive
         \begin{align*}
        \mathbb{E}&\left[\Bb^{\top}\Bb \mathbf{A} \Bb^{\top}\Bb \right] \\ 
        &\preceq c_1 \mathbb{E}\left[\operatorname{tr}\left((\mathbf{I}-\frac{\beta}{n_1}\mathbf{X}^{\top}\mathbf{X})\mathbf{A} (\mathbf{I}-\frac{\beta}{n_1}\mathbf{X}^{\top}\mathbf{X})\bSigma\right) (\mathbf{I}-\frac{\beta}{n_1}\mathbf{X}^{\top}\mathbf{X})\bSigma(\mathbf{I}-\frac{\beta}{n_1}\mathbf{X}^{\top}\mathbf{X})\right]
        \\&= c_1 \mathbb{E}\left[\operatorname{tr}\left(\mathbf{A} (\bSigma-\underbrace{\frac{\beta}{n_1}(\mathbf{X}^{\top}\mathbf{X}\bSigma +\bSigma \mathbf{X}^{\top}\mathbf{X})}_{\Jb_1}+ \frac{\beta^2}{n_1^2} \mathbf{X}^{\top}\mathbf{X}\bSigma \mathbf{X}^{\top}\mathbf{X})\right)\right.\\
        &\quad\cdot \left. (\mathbf{I}-\frac{\beta}{n_1}\mathbf{X}^{\top}\mathbf{X})\bSigma(\mathbf{I}-\frac{\beta}{n_1}\mathbf{X}^{\top}\mathbf{X})\right].
    \end{align*}
    We can bound the extra term $\Jb_1$ in the similar way as $\beta>0$.
    For any $\xb_{i}$, $1\leq i\leq n_1$, we have
    \begin{align*}
       \mathbb{E}&\left[\operatorname{tr}\left(\Ab\xb_{i}\xb_{i}^{\top}\bSigma\right)(\mathbf{I}-\frac{\beta}{n_1}\mathbf{X}^{\top}\mathbf{X})\bSigma(\mathbf{I}-\frac{\beta}{n_1}\mathbf{X}^{\top}\mathbf{X})\right]\\
       &= \mathbb{E}\left[\operatorname{tr}\left(\zb_{i}^{\top}\bSigma^{\frac{3}{2}}\Ab\bSigma^{\frac{1}{2}}\zb_{i}\right)(\mathbf{I}-\frac{\beta}{n_1}\mathbf{X}^{\top}\mathbf{X})\bSigma(\mathbf{I}-\frac{\beta}{n_1}\mathbf{X}^{\top}\mathbf{X})\right]\\
       & = \sum_{k}\iota_k \mathbb{E}\left[(\zb_{i}^{\top}\boldsymbol{\kappa}_k)^2(\mathbf{I}-\frac{\beta}{n_1}\mathbf{X}^{\top}\mathbf{X})\bSigma(\mathbf{I}-\frac{\beta}{n_1}\mathbf{X}^{\top}\mathbf{X})\right]
    \end{align*}
    where the SVD of $\bSigma^{\frac{3}{2}}\Ab\bSigma^{\frac{1}{2}}$ is $\sum_{k} \iota_{k} \boldsymbol{\kappa}_{k}\boldsymbol{\kappa}^{\top}_{k}$. Similarly, for any unit vector $\wb\in\mathbb{R}^{d}$, we can obtain
    \begin{align*}
           \wb^{\top}& \mathbb{E}\left[\Hb^{-\frac{1}{2}}_{n_1,\beta}\operatorname{tr}\left(\Ab\xb_{i}\xb_{i}^{\top}\bSigma)(\mathbf{I}-\frac{\beta}{n_1}\mathbf{X}^{\top}\mathbf{X}\right)\bSigma(\mathbf{I}-\frac{\beta}{n_1}\mathbf{X}^{\top}\mathbf{X})\Hb^{-\frac{1}{2}}_{n_1,\beta}\right]\wb\\
           &\leq \sum_{k}\iota_k \sqrt{\mathbb{E}[(\zb_{i}^{\top}\boldsymbol{\kappa}_k)^4]}\sqrt{\mathbb{E}[\|\wb^{\top}\Hb^{-\frac{1}{2}}_{n_1,\beta}(\mathbf{I}-\frac{\beta}{n_1}\mathbf{X}^{\top}\mathbf{X})\bSigma(\mathbf{I}-\frac{\beta}{n_1}\mathbf{X}^{\top}\mathbf{X})\Hb^{-\frac{1}{2}}_{n_1,\beta}\wb\|^2]}\\
           &\leq 4 \sqrt{C(\beta,\bSigma)}\sigma^2_x\operatorname{tr}(A\bSigma^2)
    \end{align*}
    which implies:
    \begin{align*}
        \mathbb{E}\left[\operatorname{tr}\left(\Ab\xb_{i}\xb_{i}^{\top}\bSigma)(\mathbf{I}-\frac{\beta}{n_1}\mathbf{X}^{\top}\mathbf{X}\right)\bSigma(\mathbf{I}-\frac{\beta}{n_1}\mathbf{X}^{\top}\mathbf{X})\right]\preceq 4\sqrt{C(\beta,\bSigma)}\sigma^2_x\operatorname{tr}(\Ab\bSigma^2) \Hb_{n_1,\beta}. 
    \end{align*}
    Hence, 
\begin{align*}
   & \mathbb{E}\left[\Bb^{\top}\Bb \mathbf{A} \Bb^{\top}\Bb \right]\\ 
   &\preceq c_1\mathbb{E}\left[\operatorname{tr}\left(\mathbf{A} (\bSigma-8\beta\sqrt{C}\sigma_x^2\bSigma^2+  64\sqrt{C}\sigma_x^4\beta^2 \bSigma\operatorname{tr}(\bSigma^2))\right) (\mathbf{I}-\frac{\beta}{n_1}\mathbf{X}^{\top}\mathbf{X})\bSigma(\mathbf{I}-\frac{\beta}{n_1}\mathbf{X}^{\top}\mathbf{X})\right]\\
   &\preceq c_1\left(1-8\beta\lambda_1\sqrt{C(\beta,\bSigma)}\sigma_x^2+ 64\sqrt{C(\beta,\bSigma)}\sigma_x^4\beta^2\operatorname{tr}(\bSigma^2)\right)\mathbb{E}\left[\operatorname{tr}(\mathbf{A} \bSigma) \right] \mathbf{H}_{n_1,\beta}.
\end{align*}
Together with the discussions for $\beta>0$, we have $$c(\beta,\bSigma
)=c_1(1+8|\beta|\lambda_1\sqrt{C(\beta,\bSigma)}\sigma_x^2+ 64\sqrt{C(\beta,\bSigma)}\sigma_x^4\beta^2\operatorname{tr}(\bSigma^2)),$$ which completes the proof.
\end{proof}

\subsection{Bias-Variance Decomposition}
We will use the bias-variance decomposition similar to theoretical studies of classic linear regression~\cite{jain2017markov,dieuleveut2017harder,zou2021benign}. Consider the error at each iteration: $\brho_t=\bomega_t-\btheta^{*}$, where $\bomega_t$ is the SGD output at each iteration $t$. Then the update rule can be written as:
$$
    \brho_{t}:= (\Ib-\alpha\Bb^{\top}_t\Bb_t)\brho_{t-1}+\alpha \Bb^{\top}_{t}\boldsymbol{\xi}_{t}$$
    where $\Bb_t,\bxi_t$ are the meta data and noise at iteration $t$ (see \cref{ap-eq-data,ap-eq-xi}). It is helpful to consider $\brho_{t}$ as the sum of the following two random processes:
    \begin{itemize}
        \item If there is no meta noise, the error comes from the bias:
        $$
    \brho^{\text{bias}}_{t}:= (\Ib-\alpha\Bb^{\top}_t\Bb_t)\brho^{\text{bias}}_{t-1}\quad \brho^{\text{bias}}_{t}=\brho_{0}.$$
    \item If the SGD trajectory starts from $\btheta^{*}$, the error  originates from the variance:
    $$
       \brho^{\text{var}}_{t}:= (\Ib-\alpha\Bb^{\top}_t\Bb_t)\brho^{\text{var}}_{t-1}+\alpha \Bb^{\top}_{t}\boldsymbol{\xi}_{t}\quad \brho^{\text{var}}=\mathbf{0}
    $$
    and $\mathbb{E}[\brho^{\text{var}}_{t}]=0$.
    \end{itemize}
    With slightly abused notations, we have:
    $$
   \brho_{t}= \brho^{\text{bias}}_{t}+\brho^{\text{var}}_{t}.
    $$
    Define the averaged output of $\brho^{\text{bias}}_{t}$, $\brho^{\text{var}}_{t}$ and $\brho_t$ after $T$ iterations as:
    \begin{align}\label{eq-rho}
        \rhob^{\text{bias}}_T=\frac{1}{T}\sum_{t=1}^{T} \brho^{\text{bias}}_{t},\quad
      \rhob^{\text{var}}_T=\frac{1}{T}\sum_{t=1}^{T} \brho^{\text{var}}_{t},\quad       \rhob_T=\frac{1}{T}\sum_{t=1}^{T} \brho_{t}.
    \end{align}
    Similarly, we have 
        $$
   \rhob_{T}= \rhob^{\text{bias}}_{T}+\rhob^{\text{var}}_{T}.
    $$
Now we are ready to introduce the bias-variance decomposition for the excess risk.
\begin{lemma}[Bias-variance decomposition]\label{lemma-bv}
  Following the notations in \cref{eq-rho}, then the excess risk can be decomposed as
 \begin{align*}
     R(\wl_T, \bte)\leq 2\mathcal{E}_\text{bias}+2\mathcal{E}_\text{var}
 \end{align*}
 where
 \begin{align}
     \mathcal{E}_\text{bias}=\frac{1}{2} \langle\Hte, \mathbb{E}[\rhob^{\text{bias}}_T\otimes\rhob^{\text{bias}}_T] \rangle, \quad \mathcal{E}_\text{var} =   \frac{1}{2} \langle\Hte, \mathbb{E}[\rhob^{\text{var}}_T\otimes\rhob^{\text{var}}_T] \rangle.
 \end{align}
\end{lemma}
\begin{proof} 
By \Cref{ap-lemma-excess}, we have
\begin{align*}
    R(\wl_T, \bte)&=\frac{1}{2} \langle\Hte, \mathbb{E}[\rhob_T\otimes\rhob_T] \rangle \\
    &=\frac{1}{2} \langle\Hte, \mathbb{E}[(\rhob^{\text{bias}}_T+\rhob^{\text{var}}_T)\otimes(\rhob^{\text{bias}}_T+\rhob^{\text{var}}_T)] \rangle\\
    &\leq 2\left( \frac{1}{2} \langle\Hte, \mathbb{E}[\rhob^{\text{bias}}_T\otimes\rhob^{\text{bias}}_T] \rangle + \frac{1}{2} \langle\Hte, \mathbb{E}[\rhob^{\text{var}}_T\otimes\rhob^{\text{var}}_T] \rangle\right)
\end{align*}
where the last inequality follows because for vector-valued random variables $\ub$ and $\vb$, $\mathbb{E}\|\ub+\vb\|_{H}^{2} \leq\left(\sqrt{\mathbb{E}\|\ub\|_{H}^{2}}+\sqrt{\mathbb{E}\|\vb\|_{H}^{2}}\right)^{2}$ and from Cauchy-Schwarz inequality.
\end{proof}
For $t=0,1,\cdots,T-1$, 
consider the following  bias and variance iterates:
\begin{align}
    \mathbf{D} _{t}=(\mathcal{I}-\alpha\mathcal{T}) \circ \Db _{t-1} \quad& \text { and } \quad \Db _{0}= (\boldsymbol{\omega}_t-\btheta^{*})(\boldsymbol{\omega}_t-\btheta^{*})^{\top}\nonumber\\   \mathbf{V} _{t}=(\mathcal{I}-\alpha\mathcal{T}) \circ \Vb _{t-1}+\alpha^{2} \Pi \quad &\text { and } \quad \Vb _{0}=\mathbf{0}\label{eq-bv}
    \end{align}
    where $\Pi=\mathbb{E}[\Bb^{\top}\boldsymbol{\xi}\boldsymbol{\xi}^{\top}\Bb]$. One can verify that 
    $$
\mathbf{D}_{t}=\mathbb{E}\left[\brho_{t}^{\text {bias }} \otimes \brho_{t}^{\text {bias }}\right], \quad \mathbf{V}_{t}=\mathbb{E}\left[\brho_{t}^{\text {var }} \otimes \brho_{t}^{\text {var }}\right].
$$
With such notations, we can further bound the bias and variance terms.
\begin{lemma}\label{lemma-further-bv}
   Following the notations in \cref{eq-bv}, we  have
   \begin{align}
    \mathcal{E}_\text { bias }& \leq \frac{1}{\alpha T^{2}} \left\langle \left(\Ib-(\mathbf{I}-\alpha\Htr)^{T}\right )\Htr^{-1}\Hte, \sum_{t=0}^{T-1}\Db _{t}\right\rangle,\\
    \mathcal{E}_\text { var } &\leq \frac{1}{T^{2}} \sum_{t=0}^{T-1} \sum_{k=t}^{T-1}\left\langle(\mathbf{I}-\alpha \Htr)^{k-t} \Hte, \mathbf{V}_{t}\right\rangle.
    \end{align}
\end{lemma}
\begin{proof}
Similar calculations have appeared in the prior works~\cite{jain2017markov,zou2021benign}. However, our meta linear model contains additional terms, and hence we provide a proof here for completeness. We first have
\begin{align*}
    \mathbb{E}[\rhob^{\text{var}}_T\otimes\rhob^{\text{var}}_T]& =\frac{1}{T^{2}} \sum_{t=0}^{T-1} \sum_{k=0}^{T-1}  \mathbb{E}[\brho^{\text{var}}_t\otimes\brho^{\text{var}}_k]\\
   & \preceq \frac{1}{T^{2}} \sum_{t=0}^{T-1} \sum_{k=t}^{T-1}  \mathbb{E}[\brho^{\text{var}}_t\otimes\brho^{\text{var}}_k]+\mathbb{E}[\brho^{\text{var}}_k\otimes\brho^{\text{var}}_t]
\end{align*}
where the last inequality follows because we double count the diagonal terms $t=k$.

For $t\leq k$, $\mathbb{E}[\brho^{\text{var}}_k|\brho^{\text{var}}_t]=(\mathbf{I}-\alpha\Htr)^{k-t} \brho^{\text{var}}_t$, since $\mathbb{E}[\Bb_t^{\top}\bxi_t|\brho_{t-1}]=\mathbf{0}$. From this, we have
\begin{align*}
    \mathbb{E}[\rhob^{\text{var}}_T\otimes\rhob^{\text{var}}_T]
   & \preceq \frac{1}{T^{2}} \sum_{t=0}^{T-1} \sum_{k=t}^{T-1}  \Vb_t(\mathbf{I}-\alpha\Htr)^{k-t}+\Vb_t (\mathbf{I}-\alpha\Htr)^{k-t}.
\end{align*}
Substituting the above inequality into  $\frac{1}{2} \langle\Hte, \mathbb{E}[\rhob^{\text{var}}_T\otimes\rhob^{\text{var}}_T] \rangle$, we obtain:
\begin{align*}
    \mathcal{E}_\text{var} &=\frac{1}{2} \langle\Hte, \mathbb{E}[\rhob^{\text{var}}_T\otimes\rhob^{\text{var}}_T] \rangle\\
    &\leq \frac{1}{2T^{2}} \sum_{t=0}^{T-1} \sum_{k=t}^{T-1} \langle \Hte, \Vb_t(\mathbf{I}-\alpha\Htr)^{k-t}\rangle + \langle \Hte,\Vb_t (\mathbf{I}-\alpha\Htr)^{k-t}\rangle\\
    &=\frac{1}{T^{2}} \sum_{t=0}^{T-1} \sum_{k=t}^{T-1} \langle (\mathbf{I}-\alpha\Htr)^{k-t}\Hte, \Vb_t\rangle
\end{align*}
where the last inequality follows from \Cref{ass-comm} that $F$ and $\bSigma
$ commute, and hence $\Hte$ and $\mathbf{I}-\alpha\Htr$ commute. 

For the bias term, similarly we have:
\begin{align}
    \mathcal{E}_\text{bias}&\leq \frac{1}{T^{2}} \sum_{t=0}^{T-1} \sum_{k=t}^{T-1} \langle (\mathbf{I}-\alpha\Htr)^{k-t}\Hte, \Db_t\rangle\\
    &=\frac{1}{\alpha T^2} \sum_{t=0}^{T-1}  \langle \left(\Ib-(\mathbf{I}-\alpha\Htr)^{T-t}\right)\Htr^{-1} \Hte, \Db_t\rangle\\
    &\leq \frac{1}{\alpha T^2}   \langle \left(\Ib-(\mathbf{I}-\alpha\Htr)^{T}\right)\Htr^{-1} \Hte,\sum_{t=0}^{T-1} \Db_t\rangle
\end{align}
which completes the proof.
\end{proof}

\subsection{Bounding the Bias}
Now we start to bound the bias term.  By \Cref{lemma-further-bv}, we focus on bounding the summation of $\Db_t$, i.e. $\sum_{t=0}^{T-1} \Db_t$.  Consider  $\mathbf{S}_{t}:=\sum_{k=0}^{t-1} \Db _{k}$, and the following lemma shows the properties of $\mathbf{S}_{t}$

    \begin{lemma}
        $\mathbf{S}_{t}$ satisfies the recursion form:
$$
\mathbf{S}_{t}=(\mathcal{I}-\alpha \mathcal{T}) \circ \mathbf{S}_{t-1}+\Db_{0}.
$$
Moreover, if $\alpha<\frac{1}{c(\btr,\bSigma)\operatorname{tr}(\bSigma
)}$, then we have:
$$
\Db_{0}=\mathbf{S}_{0} \preceq \mathbf{S}_{1} \preceq \cdots \preceq \mathbf{S}_{\infty}
$$
where $\mathbf{S}_{\infty}:=\sum_{k=0}^{\infty}(\mathcal{I}-\alpha\mathcal{T})^k \circ \Db_{0}=\alpha^{-1} \mathcal{T}^{-1} \circ \Db_{0}$.
    \end{lemma}
    \begin{proof}
    By \cref{eq-bv}, we have
    \begin{align*}
        \mathbf{S}_{t}&=\sum_{k=0}^{t-1} \Db _{k}= \sum_{k=0}^{t-1}(\mathcal{I}-\alpha\mathcal{T})^{k} \circ \Db _{0}\\
        &= \Db _{0}+(\mathcal{I}-\alpha\mathcal{T})\circ \left(\sum_{k=0}^{t-2}(\mathcal{I}-\alpha\mathcal{T})^{k} \circ \Db _{0}\right)\\
        &= \Db _{0}+(\mathcal{I}-\alpha \mathcal{T}) \circ \mathbf{S}_{t-1}.
    \end{align*}
    By \Cref{lemma-linearop}, $(\mathcal{I}-\alpha \mathcal{T})$ is PSD mapping, and hence  $\Db_t=(\mathcal{I}-\alpha \mathcal{T})\circ \Db_{t-1}$ is a PSD matirx for every $t$, which implies $\Sb_{t-1}\preceq \Sb_{t-1}+\Db_t=\Sb_{t}$.
    The form of $\Sb_{\infty}$ can be directly obtained by \Cref{lemma-linearop}.
    \end{proof}
    Then we can decompose $\Sb_t$ as follows:
    \begin{align}
        \Sb_t& = \Db _{0}+(\mathcal{I}-\alpha \widetilde{\mathcal{T}}) \circ \mathbf{S}_{t-1}+ \alpha(\widetilde{\mathcal{T}}-\mathcal{T})\circ \mathbf{S}_{t-1}\nonumber\\
        &=\Db _{0}+(\mathcal{I}-\alpha \widetilde{\mathcal{T}}) \circ \mathbf{S}_{t-1}+ \alpha^2(\mathcal{M}-\widetilde{\mathcal{M}})\circ \mathbf{S}_{t-1}\nonumber\\
        &\preceq  \Db _{0}+(\mathcal{I}-\alpha \widetilde{\mathcal{T}}) \circ \mathbf{S}_{t-1}+ \alpha^2\mathcal{M}\circ \mathbf{S}_{T}\nonumber\\
        &=\sum^{t-1}_{k=0} (\mathcal{I}-\alpha \widetilde{\mathcal{T}})^{k} \circ (\Db_0+\alpha^2\mathcal{M}\circ \mathbf{S}_{T})\label{eq-st}
    \end{align}
    where the inequality follows because  $\Sb_t\preceq \Sb_{T}$ for any $t\leq T$. Therefore, it is crucial to understand $\mathcal{M}\circ \mathbf{S}_{T}$.
        \begin{lemma}\label{lemma-ms-pre}
        For any symmetric matrix $\mathbf{A}$, if $\alpha<\frac{1}{c(\btr,\bSigma)\operatorname{tr}(\bSigma
)}$, it holds that
$$
\mathcal{M} \circ \mathcal{T}^{-1} \circ \mathbf{A} \preceq \frac{c(\btr,\bSigma)\operatorname{tr}\left(\bSigma \Htr^{-1} \mathbf{A}\right)}{1-\alpha c(\btr,\bSigma) \operatorname{tr}(\bSigma)} \cdot \Htr .
$$
    \end{lemma}

    \begin{proof}
        Denote $\mathbf{C}=\mathcal{T}^{-1} \circ \mathbf{A} $. Recalling $\tilde{\mathcal{T}}=\mathcal{T}+\alpha \mathcal{M}-\alpha \widetilde{\mathcal{M}}$, we have
        $$
        \begin{aligned}
        \widetilde{\mathcal{T}} \circ \mathbf{C}&=\mathcal{T} \circ \mathbf{C}+\alpha \mathcal{M} \circ \mathbf{C}-\alpha\widetilde{\mathcal{M}} \circ \mathbf{C}\\
       & \preceq \mathbf{A}+\alpha \mathcal{M} \circ \mathbf{C}.
        \end{aligned}
        $$
        Recalling that $\widetilde{\mathcal{T}}^{-1}$ exists and is a PSD mapping, we then have
\begin{align}
    \mathcal{M} \circ \mathbf{C} &\preceq \alpha \mathcal{M} \circ \tilde{\mathcal{T}}^{-1} \circ \mathcal{M} \circ \mathbf{C}+\mathcal{M} \circ \tilde{\mathcal{T}}^{-1} \circ \mathbf{A}\nonumber\\
&\preceq \sum^{\infty}_{k=0} (\alpha \mathcal{M} \circ \tilde{\mathcal{T}}^{-1})^{k} \circ (\mathcal{M} \circ \tilde{\mathcal{T}}^{-1} \circ \mathbf{A}).\label{eq-mta}
\end{align}

By \Cref{prop-4}, we have $\mathcal{M} \circ \widetilde{\mathcal{T}}^{-1} \circ \mathbf{A}\preceq \underbrace{c(\btr,\bSigma)\operatorname{tr}( \bSigma\widetilde{\mathcal{T}}^{-1} \circ \mathbf{A})}_{J_2}\Htr$. Substituting back into \cref{eq-mta}, we obtain:
\begin{align*}
    \sum^{\infty}_{k=0}& (\alpha \mathcal{M} \circ \tilde{\mathcal{T}}^{-1})^{k} \circ (\mathcal{M} \circ \tilde{\mathcal{T}}^{-1} \circ \mathbf{A})\preceq \sum^{\infty}_{k=0} (\alpha \mathcal{M} \circ \tilde{\mathcal{T}}^{-1})^{k} \circ (J_2\Htr)\\
&\preceq J_2\sum^{\infty}_{k=0} (\alpha c(\btr,\bSigma)\operatorname{tr}(\bSigma))^{k} \Htr\preceq  \frac{J_2}{1-\alpha c(\btr,\bSigma) \operatorname{tr}(\bSigma)}\Htr
\end{align*}
where the second inequality follows since $\tilde{\mathcal{T}}^{-1}\circ \Htr \preceq \mathbf{I}$ (\Cref{lemma-linearop}) and $\mathcal{M} \circ \mathbf{I} \preceq c(\btr,\bSigma)\operatorname{tr}(\bSigma) \Htr$ (\Cref{prop-4}).

Finally, we bound $J_2$ as follows:
$$
\begin{aligned}
\operatorname{tr}\left(\bSigma \widetilde{\mathcal{T}}^{-1} \circ \mathbf{A}\right) &=\alpha\operatorname{tr}\left(\sum_{k=0}^{\infty} \bSigma(\mathbf{I}-\alpha \Htr)^{k} \mathbf{A}(\mathbf{I}-\alpha \Htr)^{k}\right) \\
&=\alpha \operatorname{tr}\left(\sum_{k=0}^{\infty} \bSigma(\mathbf{I}-\alpha \Htr)^{2 k} \mathbf{A}\right) \\
&=\operatorname{tr}\left(\bSigma\left(2 \Htr-\alpha \Htr^{2}\right)^{-1} \mathbf{A}\right)\\
&\leq \operatorname{tr}\left(\bSigma \Htr^{-1} \mathbf{A}\right)
\end{aligned}
$$
where the second equality follows because $\bSigma$ and $\Htr$ commute, and the last inequality holds since $\alpha<\frac{1} { \max_{i}\{\mu_{i}(\Htr)\}}$. Putting all these results together completes the proof. 
    \end{proof}
\begin{lemma}[Bounding $\mathcal{M}\circ \mathbf{S}_{T}$]\label{lemma-ms}
    $$
\mathcal{M} \circ \mathbf{S}_{T}  \preceq \frac{c(\btr,\bSigma) \cdot \operatorname{tr}\left(\bSigma \Htr^{-1}\left[\mathcal{I}-(\mathcal{I}-\alpha \tilde{\mathcal{T}})^{T}\right] \circ \Db _{0}\right)}{\alpha(1-c(\btr,\bSigma) \alpha \operatorname{tr}(\bSigma))} \cdot \Htr.
$$
\end{lemma}
\begin{proof}
$\mathbf{S}_{T}$  can be further derived as follows:
$$
\mathbf{S}_{T}=\sum_{k=0}^{T-1}(\mathcal{I}-\alpha \mathcal{T})^{k} \circ \Db _{0}=\alpha^{-1} \mathcal{T}^{-1} \circ\left[\mathcal{I}-(\mathcal{I}-\alpha \mathcal{T})^{T}\right]\circ \Db _{0}.
$$
Since $\tilde{\mathcal{T}}-\mathcal{T}$ is a PSD mapping by \Cref{lemma-linearop}, we have
 $\mathcal{I}-\alpha \tilde{\mathcal{T}} \leq \mathcal{I}-\alpha \mathcal{T}$.  Hence $\mathcal{I}-(\mathcal{I}-\alpha \mathcal{T})^{T} \preceq \mathcal{I}-(\mathcal{I}-\alpha \tilde{\mathcal{T}})^{T}$. Combining with the fact that  $\mathcal{T}^{-1}$ is also a PSD mapping, we have:
$$
\mathbf{S}_{T} \preceq \alpha^{-1} \mathcal{T}^{-1} \circ\left[\mathcal{I}-(\mathcal{I}-\alpha \widetilde{\mathcal{T}})^{T}\right] \circ \Db _{0}.
$$
Letting $\Ab=\left[\mathcal{I}-(\mathcal{I}- \alpha\widetilde{\mathcal{T}})^{T}\right] \circ \Db _{0}$ in \Cref{lemma-ms-pre}, we obtain:
\begin{align*}
\mathcal{M} \circ \mathbf{S}_{T}& \preceq \alpha^{-1} \mathcal{M} \circ \mathcal{T}^{-1} \circ\left[\mathcal{I}-(\mathcal{I}-\alpha \tilde{\mathcal{T}})^{T}\right] \circ \Db _{0} \\
&\preceq \frac{c(\btr,\bSigma
) \cdot \operatorname{tr}\left(\bSigma \Htr^{-1}\left[\mathcal{I}-(\mathcal{I}-\alpha \tilde{\mathcal{T}})^{T}\right] \circ \Db _{0}\right)}{\alpha(1-c(\btr,\bSigma) \alpha \operatorname{tr}(\bSigma))} \cdot \Htr .
\end{align*}
\end{proof}
Now we are ready to derive the upper bound on the bias term.
\begin{lemma}[Bounding the bias]\label{lemma-bias}
    If $\alpha<\frac{1}{c(\btr,\bSigma)\operatorname{tr}(\bSigma)}$, for sufficiently large $n_1$, s.t. $\mu_i(\Htr)>0$, $\forall i$, then we have 
    \begin{align*}
       \mathcal{E}_\text{bias}&\leq  \sum_{i}\left(\frac{1}{\alpha^{2} T^{2}}\mathbf{1}_{\mu_i(\Htr)\geq \frac{1}{\alpha T}}+ \mu_i^2(\Htr)\mathbf{1}_{\mu_i(\Htr)< \frac{1}{\alpha T} } \right)\frac{\omega_i^2\mu_i(\Hte)}{\mu_i(\Htr)^2}\\
       &+ \frac{2 c(\btr,\bSigma)}{T \alpha\left(1-c(\btr,\bSigma)\alpha \operatorname{tr}(\bSigma)\right)}
        \sum_{i}{\left(\frac{1}{\mu_i(\Htr)}\mathbf{1}_{\mu_i(\Htr)\geq \frac{1}{\alpha T}}+T\alpha \mathbf{1}_{\mu_i(\Htr)< \frac{1}{\alpha T} } \right) \cdot \lambda_i\omega_i^2} 
 \\&\times 
\sum_{i}\left(\frac{1}{T}\mathbf{1}_{\mu_i(\Htr)\geq \frac{1}{\alpha T}}+T\alpha^2 \mu_i(\Htr)^2\mathbf{1}_{\mu_i(\Htr)< \frac{1}{\alpha T} } \right)\cdot\frac{\mu_i(\Htr)}{\mu_i(\Hte)}.
    \end{align*}
\end{lemma}
\begin{proof} Applying \Cref{lemma-ms} to \cref{eq-st}, we can obtain:
   $$
    \begin{aligned}
    \mathbf{S}_{t} 
    &\preceq\sum_{k=0}^{t-1}(\mathcal{I}-\alpha \widetilde{\mathcal{T}})^{k} \circ
    \left(\frac{\alpha c(\btr,\bSigma)\cdot \operatorname{tr}\left(\bSigma \Htr^{-1}\left[\mathcal{I}-(\mathcal{I}-\alpha \tilde{\mathcal{T}})^{T}\right] \circ \Db _{0}\right)}{1-c(\beta,\Sigma) \alpha \operatorname{tr}(\Sigma)} \cdot \Htr + \Db_{0}\right) \\
    &=\sum_{k=0}^{t-1}(\mathbf{I}-\alpha \Htr)^{k}\cdot \\
    &\left(\underbrace{\frac{\alpha c(\btr,\bSigma) \cdot \operatorname{tr}\left(\bSigma\Htr^{-1}(\Db_0-(\mathbf{I}-\alpha \Htr)^{T}\Db _{0}(\mathbf{I}-\alpha \mathbf{H}_n)^{T})\right)}{1-c(\btr,\bSigma) \alpha \operatorname{tr}(\bSigma)} \cdot \Htr}_{\mathbf{G}_1} +\underbrace{\Db_{0}}_{\Gb_2}\right)\\
    &\cdot (\mathbf{I}-\alpha \Htr)^{k}.
    \end{aligned}
    $$
  Letting $t=T$, and substituting the upper bound of $\mathbf{S}_{T}$ into the bias term in \Cref{lemma-further-bv}, we obtain:
  $$
\begin{aligned}
\mathcal{E}_\text { bias } & \leq \frac{1}{ \alpha T^{2}} \sum_{k=0}^{T-1}\left\langle ((\mathbf{I}- \alpha  \Htr)^{2 k}-(\mathbf{I}-  \alpha  \mathbf{H}_{n,\beta})^{T+2 k})\mathbf{H}_{n,\beta}^{-1}\Hte,\mathbf{G}_1+\Gb_2\right\rangle\\
&\leq \frac{1}{\alpha T^{2}} \sum_{k=0}^{T-1}\left\langle ((\mathbf{I}- \alpha  \Htr)^{ k}-(\mathbf{I}-  \alpha  \Htr)^{T+ k})\Htr^{-1}\Hte,\mathbf{G}_1+\Gb_2\right\rangle.
\end{aligned}
$$
We first consider
\begin{align*}
    d_1=  \frac{1}{\alpha T^{2}} \sum_{k=0}^{T-1}\left\langle ((\mathbf{I}- \alpha  \Htr)^{ k}-(\mathbf{I}-  \alpha  \Htr)^{T+ k})\Htr^{-1}\Hte,\mathbf{G}_1\right\rangle.
\end{align*}
Since $\Htr$, $\Hte$ and $\mathbf{I}-\alpha \Htr$ commute, 
we have
\begin{align*}
  d_1=& \frac{c(\btr,\bSigma) \cdot \operatorname{tr}\left(\bSigma\Htr^{-1}(\Db_0-(\mathbf{I}-\alpha \Htr)^{T}\Db _{0}(\mathbf{I}-\alpha \Htr)^{T})\right)}{(1-c(\btr,\bSigma) \alpha \operatorname{tr}(\bSigma)) T^2}\\
  &\times \sum_{k=0}^{T-1}\left\langle\left( (\mathbf{I}-\alpha  \Htr)^{k}-(\mathbf{I}-\alpha  \Htr)^{T+k} \right),\Hte\right\rangle.
    \end{align*}
    For the first term, since $\bSigma$, $\Htr$ and $\mathbf{I}-\alpha \Htr$ can be diagonalized simultaneously, considering the  eigen-decompositions under the basis of $\bSigma$ and recalling $\bSigma
    =\Vb\bLambda\Vb^{\top}$, we have:
       $$
    \begin{aligned}
        &\operatorname{tr}\left(\bSigma
        \Htr^{-1}[\Db _{0}-(\mathbf{I}-\alpha \Htr)^{T} \Db _{0}(\mathbf{I}-\alpha\Htr)^{T}]\right)\\
        &=\sum_{i}{\left(1-\left(1-\alpha \mu_{i}(\Htr)\right)^{2 T}\right) \cdot\left(\left\langle\mathbf{w}_{0}-\mathbf{w}^{*}, \mathbf{v}_{i}\right\rangle\right)^{2}}\frac{\lambda_i}{\mu_i(\Htr)}\\
        &\leq 2\sum_{i}\left(\mathbf{1}_{\lambda_{i}(\Hb_{n,\beta})\geq \frac{1}{\alpha T}}+T\alpha \mu_{i}(\Hb_{n,\beta})\mathbf{1}_{\mu_{i}(\Htr)< \frac{1}{\alpha T}} \right) \cdot \left(\left\langle\mathbf{w}_{0}-\mathbf{w}^{*}, \mathbf{v}_{i}\right\rangle\right)^{2}\frac{\lambda_i}{\mu_i(\Htr)} 
    \end{aligned}$$
    where the last inequality holds since $1-(1-\alpha x)^{2T}\leq\min\{2, 2T\alpha x\}$.
    
    For the second term, similarly, $\Hte$ and $\mathbf{I}-\alpha \Htr$ can be diagonalized simultaneously. We then have
    \begin{align*}
        \sum_{k=0}^{T-1}&\left\langle\left( (\mathbf{I}-\alpha  \Htr)^{k}-(\mathbf{I}-\alpha  \Htr)^{T+k} \right),\Hte\right\rangle\\
        \leq & \sum^{T-1}_{k=0}\sum_{i} [(1-\alpha\mu_{i}(\Htr))^{k}-(1-\alpha\mu_{i}(\Htr))^{T+k}] \mu_{i}(\Hte)\\
        =&\frac{1}{\alpha} \sum_{i} [1-(1-\alpha\mu_{i}(\Htr))^{T}]^2 \frac{\mu_{i}(\Hte)}{\mu_{i}(\Htr)}\\
        \leq &\frac{1}{\alpha} \sum_i \left(\mathbf{1}_{\lambda_{i}(\Hb_{n,\beta})\geq \frac{1}{\alpha T}}+T^2\alpha^2 \lambda_{i}(\Hb_{n,\beta})\mathbf{1}_{\lambda_{i}(\Hb_{n,\beta})< \frac{1}{\alpha T}} \right)\frac{\mu_{i}(\Hte)}{\mu_{i}(\Htr)}.
    \end{align*}
    Now we turn to:
    \begin{align*}
    d_2&=  \frac{1}{\alpha T^{2}} \sum_{k=0}^{T-1}\left\langle ((\mathbf{I}- \alpha  \Htr)^{ k}-(\mathbf{I}-  \alpha  \Htr)^{T+ k})\Htr^{-1}\Hte,\mathbf{G}_2\right\rangle.
\end{align*}
Considering the orthogonal decompositions of  $\Hte$ and $\Htr$ under $\Vb$, $\Htr=\Vb\bLambda_1\Vb^{\top}$, $\Hte=\Vb\bLambda_2\Vb^{\top}$, where the diagonal entries of $\bLambda_1$ are $\mu_i(\Htr)$ (and $\mu_i(\Hte)$ for $\bLambda_2$). Then we have:
\begin{align*}
       d_2&=\frac{1}{\alpha T^{2}} \sum_{k=0}^{T-1}\left\langle\underbrace{\left( (\mathbf{I}- \alpha \bLambda_1)^{ k}-(\mathbf{I}-  \alpha  \bLambda_1)^{T+ k}\right)\bLambda_1^{-1}\bLambda_2}_{\Jb_3}, \Vb^{\top}\Db_0\Vb\right\rangle \\
       &=\frac{1}{\alpha T^{2}} \sum_{k=0}^{T-1} \sum_{i}\left[\left(1-\alpha \mu_{i}(\Htr)\right)^{k}-\left(1-\alpha \mu_{i}(\Htr)\right)^{T+k}\right] 
       \frac{\omega_i^2\mu_{i}(\Hte)}{\mu_{i}(\Htr)}\\
 &=\frac{1}{\alpha^{2} T^{2}} \sum_{i}\left[1-\left(1-\alpha \mu_{i}(\Htr)\right)^{T}\right]^{2}\frac{\omega_i^2\mu_{i}(\Hte)}{\mu_{i}^2(\Htr)}\\
 &\leq \frac{1}{\alpha^{2} T^{2}}\sum_{i}\left(\mathbf{1}_{\mu_i(\Htr)\geq \frac{1}{\alpha T}}+ \alpha^{2} T^{2}\mu_i^2(\Htr)\mathbf{1}_{\mu_i(\Htr)< \frac{1}{\alpha T} } \right)\frac{\omega_i^2\mu_i(\Hte)}{\mu^2_i(\Htr)}
\end{align*}
where $\omega_i=\langle\bomega_0-\btheta^{*},\vb_i\rangle$ is the diagonal entry of $\Vb^{\top}\Db_0\Vb$ and the second equality holds since $\Jb_3$ is a diagonal matrix.
\end{proof}
\subsection{Bounding the Variance}
Note that the noisy part $\Pi=\mathbb{E}[\Bb^{\top}\boldsymbol{\xi}\boldsymbol{\xi}^{\top}\Bb]$ in  \cref{eq-bv} is important in the variance iterates. In order to analyze the variance term,  we first understand the role of $\Pi$ by the following lemma.
\begin{lemma}[Bounding the noise]\label{lemma-noise}
     \begin{align*}
        & \Pi=\mathbb{E}[\Bb^{\top}\boldsymbol{\xi}\boldsymbol{\xi}^{\top}\Bb]\preceq f(\btr,n_2,\sigma,\bSigma,\bSigma_{\btheta}) \Htr
     \end{align*}
     where $f(\beta,n,\sigma,\bSigma,\bSigma_{\btheta})=[c(\beta,\bSigma)\operatorname{tr}({\bSigma_{\btheta}\bSigma})+4c_1\sigma^2\sigma_x^2\beta^2\sqrt{C(\beta,\bSigma)}\operatorname{tr}(\bSigma^2)+{\sigma^2}/n]$.
\end{lemma}
\begin{proof} With a slight abuse of notations, we write $\btr$ as $\beta$ in this proof. By definition of meta data and noise, we have
\begin{align*}
    & \Pi=\mathbb{E}[\Bb^{\top}\boldsymbol{\xi}\boldsymbol{\xi}^{\top}\Bb]\\
    &=\frac{\sigma^2}{n_2}\Hb_{n_1,\beta}+\mathbb{E}[\Bb^{\top}\Bb\Sigma_{\btheta}\Bb^{\top}\Bb]+\sigma^2\cdot \frac{\beta^{2}}{n_2 n_1^2}\mathbb{E}[\Bb^{\top}\Xb^{\text{out}}{\Xb^{\text{in}}}^{\top}\Xb^{\text{in}}{\Xb^{\text{out}}}^{\top}\Bb].
\end{align*}
The second term can be directly bounded by \Cref{prop-4}:
$$
\mathbb{E}[\Bb^{\top}\Bb\Sigma_{\btheta}\Bb^{\top}\Bb]\preceq c(\beta,\bSigma) \operatorname{tr}({\bSigma_{\btheta}\bSigma})\Hb_{n_1,\beta}.
$$
For the third term, we utilize the technique similar to \Cref{prop-4}, and by Assumption 1, we have:
\begin{align*}
   & \sigma^2\cdot \frac{\beta^{2}}{n_2 n_1^2}\mathbb{E}\left[\Bb^{\top}\Xb^{\text{out}}{\Xb^{\text{in}}}^{\top}\Xb^{\text{in}}{\Xb^{\text{out}}}^{\top}\Bb\right]\\
    &\preceq \sigma^2c_1\cdot \frac{\beta^{2}}{n_1^2}\mathbb{E}\left[\operatorname{tr}({\Xb^{\text{in}}}^{\top}\Xb^{\text{in}}\Sigma)(\mathbf{I}-\frac{\beta}{n_1}{\Xb^{\text{in}}}^{\top}\Xb^{\text{in}})\bSigma (\mathbf{I}-\frac{\beta}{n_1}{\Xb^{\text{in}}}^{\top}\Xb^{\text{in}})\right].
\end{align*}
Following the analysis for $\Jb_1$ in the proof of \Cref{prop-4}, and letting $\Ab=\mathbf{I}$, we obtain:
$$
\frac{1}{n_1^2}\mathbb{E}\left[\operatorname{tr}({\Xb^{\text{in}}}^{\top}\Xb^{\text{in}}\Sigma)(\mathbf{I}-\frac{\beta}{n_1}{\Xb^{\text{in}}}^{\top}\Xb^{\text{in}})\bSigma (\mathbf{I}-\frac{\beta}{n_1}{\Xb^{\text{in}}}^{\top}\Xb^{\text{in}})\right] \preceq 4\sqrt{C(\beta,\bSigma)}\sigma^2_x\operatorname{tr}(\bSigma^2) \Hb_{n_1,\beta}.
$$
Putting all these results together completes the proof.
\end{proof}
\begin{lemma}[Property of $\Vb_t$]\label{lemma-v-prop}
If the stepsize satisfies $\alpha<\frac{1} {c(\btr,\bSigma)\operatorname{tr}(\bSigma)}$, it holds that
$$
\mathbf{0}=\mathbf{V}_{0} \preceq \mathbf{V}_{1} \preceq \cdots \preceq \mathbf{V}_{\infty} \preceq \frac{\alpha f(\btr,n_2,\sigma,\bSigma,\bSigma_{\btheta}) }{1-\alpha c(\btr,\bSigma)\operatorname{tr}(\bSigma)}\mathbf{I}.
$$
\end{lemma}
\begin{proof}
Similar calculations has appeared in prior works~\cite{jain2017markov,zou2021benign}. However, our analysis of the meta linear model needs to handle the complicated meta noise, and hence we provide a proof here for completeness.  

We first show that $\mathbf{V}_{t-1}\preceq\Vb_t$. By recursion:
      $$
\begin{aligned}
\mathbf{V}_{t} &=(\mathcal{I}-\alpha \mathcal{T}) \circ \mathbf{V}_{t-1}+\alpha^{2} \Pi \\
&\overset{(a)}{=}\alpha^{2} \sum_{k=0}^{t-1}(\mathcal{I}-\alpha \mathcal{T})^{k} \circ \Pi \\
&=\mathbf{V}_{t-1}+\alpha^{2}(\mathcal{I}-\alpha \mathcal{T})^{t-1} \circ \Pi \\
& \overset{(b)}{\succeq} \mathbf{V}_{t-1} 
\end{aligned}
$$
where $(a)$ holds by solving the recursion and  $(b)$ follows because $\mathcal{I}-\alpha \mathcal{T}$ is a PSD mapping.

The existence of $\mathbf{V}_{\infty}$ can be shown in the way similar to the proof of \Cref{lemma-linearop}. We first have
$$
\mathbf{V}_{t}=\alpha^{2} \sum_{k=0}^{t-1}(\mathcal{I}-\alpha \mathcal{T})^{k} \circ \Pi \preceq \alpha^{2} \sum_{k=0}^{\infty}\underbrace{(\mathcal{I}-\alpha \mathcal{T})^{k} \circ \Pi}_{\mathbf{A}_k}.
$$
By previous analysis in \Cref{lemma-linearop} , if $\alpha<\frac{1}{c(\btr,\bSigma)\operatorname{tr}(\bSigma)}$, we have
$$
\begin{aligned}
\operatorname{tr}\left(\mathbf{A}_{k}\right) 
& \leq\left(1-\alpha\min_{i}\{\mu_{i}(\Htr)\}\right) \operatorname{tr}\left(\mathbf{A}_{t-1}\right).
\end{aligned}
$$
Therefore, 
$$
\operatorname{tr}\left(\mathbf{V}_{t}\right) \leq \alpha^{2} \sum_{k=0}^{\infty} \operatorname{tr}\left(\mathbf{A}_{k}\right) \leq \frac{\alpha \operatorname{tr}(\Pi)}{\min_{i}\{\mu_{i}(\Htr)\}}<\infty.
$$
 The trace of $\Vb_t$ is uniformly bounded from above, which indicates that  $\Vb_{\infty}$ exists. 
 
 Finally,  we bound $\mathbf{V}_{\infty}$. Note that $\Vb_{\infty}$ is the solution to:
 $$
 \mathbf{V}_{\infty} =(\mathcal{I}-\alpha \mathcal{T}) \circ \mathbf{V}_{\infty}+\alpha^{2} \Pi.
 $$
 Then we can write $\mathbf{V}_{\infty}$ as $\Vb_{\infty} = \mathcal{T}^{-1}\circ \alpha\Pi $. Following the analysis in the proof of \Cref{lemma-ms-pre}, we have:
\begin{align*}
    \tilde{\mathcal{T}} \circ \mathbf{V}_{\infty}&=\tilde{\mathcal{T}}\circ  \mathcal{T}^{-1}\circ \alpha\Pi\\
    &\preceq \alpha \Pi+\alpha\mathcal{M} \circ \Vb_{\infty}\\
    &\preceq \alpha f(\btr,n_2,\sigma,\bSigma,\bSigma_{\btheta}) \Htr+\alpha\mathcal{M} \circ \mathbf{V}_{\infty}
\end{align*}
where the last inequality follows from \Cref{lemma-noise}.
Applying $\tilde{\mathcal{T}}^{-1}$, which exists and is a PSD mapping, to the both sides, we have
$$
\begin{aligned}
\mathbf{V}_{\infty} & \preceq \alpha  f(\btr,n_2,\sigma,\bSigma,\bSigma_{\btheta})\cdot \tilde{\mathcal{T}}^{-1} \circ \Htr+\alpha  \tilde{\mathcal{T}}^{-1} \circ \mathcal{M} \circ \mathbf{V}_{\infty} \\
& \overset{(a)}{\preceq} \alpha  f(\btr,n_2,\sigma,\bSigma,\bSigma_{\btheta}) \cdot \sum_{t=0}^{\infty}\left(\alpha  \tilde{\mathcal{T}}^{-1} \circ \mathcal{M}\right)^{t} \circ \tilde{\mathcal{T}}^{-1} \circ \Htr\\
&\overset{(b)}{\preceq} \alpha f(\btr,n_2,\sigma,\bSigma,\bSigma_{\btheta})\sum^{\infty}_{t=0} (\alpha c(\btr,\bSigma)\operatorname{tr}(\bSigma))^{t}\mathbf{I}\\
&= \frac{\alpha f(\btr,n_2,\sigma,\bSigma,\bSigma_{\btheta})}{1-\alpha c(\btr,\bSigma) \operatorname{tr}(\bSigma)}\mathbf{I}
\end{aligned}
$$ 
where $(a)$ holds by directly solving the recursion; $(b)$ follows from the fact that $\tilde{\mathcal{T}}^{-1} \circ \Htr\preceq \mathbf{I}$ from \Cref{lemma-linearop} and $\mathcal{M} \circ\mathbf{I}\preceq c(\btr,\bSigma)\operatorname{tr}(\bSigma
)\Htr$ by letting $\Ab=\mathbf{I}$ in \Cref{prop-4}.
\end{proof}
Now we are ready to provide the upper bound on the variance term.
\begin{lemma}[Bounding the Variance]\label{lemma-var} If $\alpha<\frac{1}{c(\btr,\bSigma)\operatorname{tr}(\bSigma)}$, for sufficiently large $n_1$, s.t. $\mu_i(\Htr)>0$, $\forall i$, then we have 
   \begin{align*}
    \mathcal{E}_\text{var}\leq& \frac{ f(\btr,n_2,\sigma,\bSigma,\bSigma_{\btheta})}{(1-\alpha c(\btr,\bSigma) \operatorname{tr}(\bSigma))}
    \\
    &\times\sum_{i}\left(\frac{1}{T}\mathbf{1}_{\mu_{i}(\Htr)\geq \frac{1}{\alpha T}}+T\alpha^2\mu_{i}^2(\Htr)\mathbf{1}_{\mu_{i}(\Htr)< \frac{1}{\alpha T} } \right)\frac{\mu_i(\Hte)}{\mu_i(\Htr)}.
\end{align*}
\end{lemma}
  \begin{proof}
  Recall
    \begin{align}
    \mathbf{V}_{t} &=(\mathcal{I}-\alpha  \mathcal{T}) \circ \mathbf{V}_{t-1}+\alpha ^{2} \Pi \nonumber \\
    &=(\mathcal{I}-\alpha  \widetilde{\mathcal{T}}) \circ \mathbf{V}_{t-1}+\alpha ^{2}(\mathcal{M}-\widetilde{\mathcal{M}}) \circ \mathbf{V}_{t-1}+\alpha ^{2} \Pi \nonumber \\
    & \preceq(\mathcal{I}-\alpha  \widetilde{\mathcal{T}}) \circ \mathbf{V}_{t-1}+\alpha ^{2} \mathcal{M} \circ \mathbf{V}_{t-1}+\alpha ^{2} \Pi. \label{eq-v}
    \end{align}
    By the uniform bound on $\mathbf{V}_t$ and $\mathcal{M}$ is a PSD mapping, we have:
\begin{align*}
    \mathcal{M} \circ \mathbf{V}_{t} &\preceq \mathcal{M} \circ \mathbf{V}_{\infty}\\
    &\overset{(a)}{\preceq} \mathcal{M} \circ\frac{\alpha f(\btr,n_2,\sigma,\bSigma,\bSigma_{\btheta})}{1-\alpha c(\btr,\bSigma)  \operatorname{tr}(\bSigma)}\mathbf{I}\\
&\overset{(b)}{\preceq} \frac{\alpha f(\btr,n_2,\sigma,\bSigma,\bSigma_{\btheta})c(\btr,\bSigma) \operatorname{tr}(\bSigma)}{1-\alpha c(\btr,\bSigma)  \operatorname{tr}(\bSigma)}\cdot \Htr
\end{align*}
where $(a)$ directly follows from  \Cref{lemma-v-prop}; $(b)$ holds because $\mathcal{M} \circ\mathbf{I}\preceq c(\btr,\bSigma)\operatorname{tr}(\bSigma
)\Htr$ (letting $\Ab=\mathbf{I}$ in \Cref{prop-4}). Substituting it back into \cref{eq-v}, we have:
\begin{align*}
    \mathbf{V}_t&\preceq (\mathcal{I}-\alpha \tilde{\mathcal{T}}) \circ \mathbf{V}_{t-1}+\alpha^2 \frac{\alpha 
     f c(\btr,\bSigma) \operatorname{tr}(\bSigma)}{1-\alpha c(\btr,\bSigma) \operatorname{tr}(\bSigma)}\cdot \Htr+\alpha^2f \Htr\\
&=(\mathcal{I}-\alpha \tilde{\mathcal{T}}) \circ \mathbf{V}_{t-1}+\frac{\alpha^2 f(\btr,n_2,\sigma,\bSigma,\bSigma_{\btheta})
}{1-\alpha c(\btr,\bSigma)  \operatorname{tr}(\bSigma)}\Htr\\
&\overset{(a)}{=} \frac{\alpha^2 f(\btr,n_2,\sigma,\bSigma,\bSigma_{\btheta})}{1-\alpha c(\btr,\bSigma)  \operatorname{tr}(\bSigma)}\sum^{t-1}_{k=0} (\mathbf{I}-\alpha \tilde{\mathcal{T}})^{k}\circ \Htr\\
&\overset{(b)}{\preceq} \frac{\alpha f(\btr,n_2,\sigma,\bSigma,\bSigma_{\btheta})}{1-\alpha c(\btr,\bSigma)  \operatorname{tr}(\bSigma)}(\mathbf{I}-(\mathbf{I}-\alpha \mathbf{H}_{n,\beta})^{t})
\end{align*}
where $(a)$ holds by solving the recursion and $(b)$ is due to the fact that 
\begin{align*}
    \sum^{t-1}_{k=0} (\mathbf{I}-\alpha \tilde{\mathcal{T}})^{k}\circ \Htr&=  \sum^{t-1}_{k=0} (\mathbf{I}-\alpha\Htr)^{k}\Htr (\mathbf{I}-\alpha\Htr)^{k}\\
    &\preceq  \sum^{t-1}_{k=0} (\mathbf{I}-\alpha\Htr)^{k}\Htr\\
    &=\frac{1}{\alpha}[\mathbf{I}-(\mathbf{I}-\alpha\Htr)^{t}].
\end{align*}
Substituting the bound for $\Vb_t$ back into the variance term in \Cref{lemma-further-bv}, we have
 $$
  \begin{aligned}
  \mathcal{E}_\text { var } & \leq \frac{1}{T^{2}} \sum_{t=0}^{T-1} \sum_{k=t}^{T-1}\left\langle(\mathbf{I}-\alpha \Htr)^{k-t} \Hte, \mathbf{V}_{t}\right\rangle \\
  &=\frac{1}{\alpha T^{2}} \sum_{t=0}^{T-1}\left\langle (\mathbf{I}-(\mathbf{I}-\alpha \Htr)^{T-t})\Htr^{-1}\Hte, \mathbf{V}_{t}\right\rangle \\
  & \leq \frac{ f(\btr,n_2,\sigma,\bSigma,\bSigma_{\btheta})}{(1-\alpha c(\btr,\bSigma) \operatorname{tr}(\bSigma))T^2}\sum_{t=0}^{T-1}\left\langle\mathbf{I}-(\mathbf{I}-\alpha \mathbf{H}_{n,\beta})^{T-t},\left(\mathbf{I}-(\mathbf{I}-\alpha \mathbf{H}_{n,\beta})^{t}\right)\mathbf{H}_{n,\beta}^{-1}\mathbf{H}_{m,\eta}\right\rangle .
  \end{aligned}
  $$
Simultaneously diagonalizing $\Htr$ and $\Hte$ as the analysis in \Cref{lemma-bias}, we have
  $$ 
  \begin{aligned}
  \mathcal{E}_\text{var}\leq&\frac{ f(\btr,n_2,\sigma,\bSigma,\bSigma_{\btheta})}{(1-\alpha c(\btr,\bSigma) \operatorname{tr}(\bSigma))T^2}\\
  &\cdot\sum_{i} \sum_{t=0}^{T-1}\left(1-\left(1-\alpha \mu_{i}(\Htr)\right)^{T-t}\right)\left(1-\left(1-\alpha \mu_{i}(\Htr)\right)^{t}\right)\frac{\mu_{i}(\Hte)}{\mu_{i}(\Htr)} \\
  \leq &\frac{ f(\btr,n_2,\sigma,\bSigma,\bSigma_{\btheta})}{(1-\alpha c(\btr,\bSigma) \operatorname{tr}(\bSigma))T^2} \\
  &\cdot\sum_{i} \sum_{t=0}^{T-1}\left(1-\left(1-\alpha \mu_{i}(\Htr)\right)^{T}\right)\left(1-\left(1-\alpha \mu_{i}(\Htr)\right)^{T}\right)\frac{\mu_{i}(\Hte)}{\mu_{i}(\Htr)} \\
  =& \frac{ f(\btr,n_2,\sigma,\bSigma,\bSigma_{\btheta})}{(1-\alpha c(\btr,\bSigma) \operatorname{tr}(\bSigma))T} \sum_{i} \left(1-\left(1-\alpha \mu_{i}(\Htr)\right)^{T}\right)^{2}\frac{\mu_{i}(\Hte)}{\mu_{i}(\Htr)}\\
  \leq& \frac{ f(\btr,n_2,\sigma,\bSigma,\bSigma_{\btheta})}{(1-\alpha c(\btr,\bSigma) \operatorname{tr}(\bSigma))T} \sum_{i} \left(\min \left\{1, \alpha T \mu_{i}(\Htr)\right\}\right)^2\frac{\mu_{i}(\Hte)}{\mu_{i}(\Htr)}\\
  \leq &\frac{ f(\btr,n_2,\sigma,\bSigma,\bSigma_{\btheta})}{(1-\alpha c(\btr,\bSigma) \operatorname{tr}(\bSigma))}\\
  &\cdot\sum_{i}\left(\frac{1}{T}\mathbf{1}_{\mu_{i}(\Htr)\geq \frac{1}{\alpha T}}+T\alpha^2\mu_{i}^2(\Htr)\mathbf{1}_{\mu{i}(\Hb_{n,\beta})< \frac{1}{\alpha T} } \right)\frac{\mu_i(\Hte)}{\mu_i(\Htr)},
  \end{aligned}
  $$
  which completes the proof.
  \end{proof}
\subsection{Proof of Theorem~\ref{thm-upper}}
\begin{theorem}[\Cref{thm-upper} Restated]\label{ap-thm-upper} Let $\omega_i=\left\langle\boldsymbol{\omega}_{0}-\boldsymbol{\theta}^{*}, \mathbf{v}_{i}\right\rangle$. If $|\btr|,|\bte|<1/\lambda_1$, $n_1$ is large ensuring that $\mu_i(\Hb_{n_1,\beta^{\text{tr}}})>0$, $\forall i$ and
$\alpha<1/\left(c(\btr,\bSigma) \operatorname{tr}(\bSigma)\right)$, then the meta excess risk $R(\overline{\boldsymbol{\omega}}_T,\bte)$ is bounded above as follows
\[R(\wl_{T},\bte)\leq \text{Bias}+ \text{Var} \]
where
  \begin{align*}
     \text{Bias}   & =  \frac{2}{\alpha^2 T} \sum_{i}\Xi_i \frac{\omega_i^2}{\mu_i(\Hb_{n_1,\beta^{\text{tr}}})} \\
      \text{Var}   &= \frac{2}{(1-\alpha c(\btr,\bSigma) \operatorname{tr}(\bSigma))}\left(\sum_{i}\Xi_{i} \right)
   \\
 \quad  \times & [{f(\btr,n_2,\sigma,\bSigma_{\boldsymbol{\theta}},\bSigma)}+{\textstyle 2c(\btr,\bSigma)
 \sum_{i}\left( \frac{\mathbf{1}_{\mu_{i}(\Hb_{n_1,\beta^{\text{tr}}})\geq \frac{1}{\alpha T}}}{T\alpha \mu_i(\Hb_{n_1,\beta^{\text{tr}}})}+\mathbf{1}_{\mu_{i}(\Hb_{n_1,\beta^{\text{tr}}})< \frac{1}{\alpha T} } \right) \lambda_{i}\omega^2_i}]. 
\end{align*}
\end{theorem}
\begin{proof}
By \Cref{lemma-bv}, we have 
 \begin{align*}
     R(\wl_T, \bte)\leq 2\mathcal{E}_\text{bias}+2\mathcal{E}_\text{var}.
 \end{align*}
 Using \Cref{lemma-bias} to bound $\mathcal{E}_\text{bias}$, and \Cref{lemma-var} to bound $\mathcal{E}_\text{var}$, we have
 \begin{align*}
     & R(\wl_T, \bte)\\&\leq \frac{2 f(\btr,n_2,\sigma,\bSigma,\bSigma_{\btheta})}{(1-\alpha c(\btr,\bSigma) \operatorname{tr}(\bSigma))}\\
  &\times \sum_{i}\left(\frac{1}{T}\mathbf{1}_{\mu_{i}(\Htr)\geq \frac{1}{\alpha T}}+T\alpha^2\mu_{i}^2(\Htr)\mathbf{1}_{\mu_{i}(\Htr)< \frac{1}{\alpha T} } \right)\frac{\mu_i(\Hte)}{\mu_i(\Htr)}\\
        &+\frac{4 c(\btr,\bSigma)}{T \alpha(1-c(\btr,\bSigma)\alpha \operatorname{tr}(\bSigma))}
        \sum_{i}\left(\frac{1}{T}\mathbf{1}_{\mu_i(\Htr)\geq \frac{1}{\alpha T}}+T\alpha^2 \mu_i(\Htr)^2\mathbf{1}_{\mu_i(\Htr)< \frac{1}{\alpha T} } \right)
 \\&\times 
\sum_{i}{\left(\frac{1}{\mu_i(\Htr)}\mathbf{1}_{\mu_i(\Htr)\geq \frac{1}{\alpha T}}+T\alpha \mathbf{1}_{\mu_i(\Htr)< \frac{1}{\alpha T} } \right) \cdot \lambda_i\left(\left\langle\bomega_{0}-\btheta^{*}, \mathbf{v}_{i}\right\rangle\right)^{2}}\\
  &+ 2 \sum_{i}\left(\frac{1}{\alpha^{2} T^{2}}\mathbf{1}_{\mu_i(\Htr)\geq \frac{1}{\alpha T}}+ \mu_i^2(\Htr)\mathbf{1}_{\mu_i(\Htr)< \frac{1}{\alpha T} } \right)\frac{\omega_i^2\mu_i(\Hte)}{\mu_i(\Htr)^2}.
 \end{align*}
 Incorporating with the definition of effective meta weight
  \begin{equation}
      \Xi_i (\bSigma
      ,\alpha,T)=\begin{cases}
      \mu_i(\Hb_{m,\beta^{\text{te}}})/\left(T \mu_i(\Hb_{n_1,\beta^{\text{tr}}})\right) &  \mu_{i}(\Hb_{n_1,\beta^{\text{tr}}})\geq \frac{1}{\alpha T}; \\
      T\alpha^2 \mu_i(\Hb_{n_1,\btr})\mu_i(\Hb_{m,\beta^{\text{te}}})&  \mu_{i}(\Hb_{n_1,\beta^{\text{tr}}})< \frac{1}{\alpha T},
      \end{cases}
  \end{equation}
  we obtain
  $$
\left(\frac{1}{T}\mathbf{1}_{\mu_{i}(\Htr)\geq \frac{1}{\alpha T}}+T\alpha^2\mu_{i}^2(\Htr)\mathbf{1}_{\mu_{i}(\Htr)< \frac{1}{\alpha T} } \right)\frac{\mu_i(\Hte)}{\mu_i(\Htr)}=\Xi_i(\bSigma
      ,\alpha,T).
  $$
  Therefore,
\[ R(\wl_T, \bte)\leq \text{Bias}+\text{Var}\]
where
  \begin{align*}
     \text{Bias}   & =  \frac{2}{\alpha^2 T} \sum_{i}\Xi_i \frac{\omega_i^2}{\mu_i(\Htr)} \\
      \text{Var}   &= \frac{2}{(1-\alpha c(\btr,\bSigma) \operatorname{tr}(\bSigma))}\left(\sum_{i}\Xi_{i} \right)
   \\
 \quad  \times & [{f(\btr,n_2,\sigma,\bSigma_{\boldsymbol{\theta}},\bSigma)}+\underbrace{ 2c(\btr,\bSigma)
 \sum_{i}\left( \frac{\mathbf{1}_{\mu_{i}(\Hb_{n_1,\beta^{\text{tr}}})\geq \frac{1}{\alpha T}}}{T\alpha \mu_i(\Hb_{n_1,\beta^{\text{tr}}})}+\mathbf{1}_{\mu_{i}(\Hb_{n_1,\beta^{\text{tr}}})< \frac{1}{\alpha T} } \right) \lambda_{i}\omega^2_i}_{V_2}].
\end{align*}
Note that the term $V_2$ is obtained by our analysis for $\mathcal{E}_{\text{bias}}$. However, it originates from the stochasticity of SGD, and hence we treat this term as the variance in our final results.
\end{proof}

\section{Analysis for Lower Bound (Theorem~\ref{thm-lower}) }
\subsection{Fourth Moment Lower Bound for Meta Nosie}
Similarly to upper bound, we need some technical results for the fourth moment of meta data $\Bb$ and noise $\bxi$ to proceed the lower bound analysis.
\begin{lemma}\label{lemma-4l}
Suppose Assumption 1-3 hold. Given $|\btr|<\frac{1}{\lambda_1}$, for any PSD matrix $\Ab$,  we have 
   \begin{align}\mathbb{E}[\Bb^{\top}\Bb\Ab\Bb^{\top}\Bb] &\succeq \Htr \mathbf{A }\Htr+\frac{b_1}{n_2} \operatorname{tr}(\Htr \mathbf{A}) \Htr\label{eq-lower-data}\\
   \Pi&\succeq \frac{1}{n_2}g(\btr,n_1,\sigma,\bSigma_{\boldsymbol{\theta}},\bSigma)\Htr
   \end{align}
   where $g(\beta,n, \sigma,\bSigma, \bSigma_{\btheta})  :={\sigma^2+b_1\operatorname{tr}(\bSigma_{\btheta}\Hb_{n,\beta})+\beta^2 \mathbf{1}_{\beta\leq 0} b_1 \operatorname{tr}(\bSigma^2)/{n}}$.
\end{lemma}
\begin{proof}
With a slight abuse of notations, we write $\btr$ as $\beta$, $\mathbf{X}^{\text{in}}$ as $\mathbf{X}$ in this proof. Note that $\xb\in\mathbb{R}^{d}\sim\mathcal{P}_{\xb}$ is independent of $\Xb^{\text{in}}$. We first derive
\begin{align*}
    \mathbb{E}&[\Bb^{\top}\Bb\Ab\Bb^{\top}\Bb]\\
    &= \frac{1}{n_2} \mathbb{E}\left[(\mathbf{I}-\frac{\beta}{n_1}\Xb^{\top}\Xb)\mathbf{x x}^{\top} (\mathbf{I}-\frac{\beta}{n_1}\Xb^{\top}\Xb) \mathbf{A}(\mathbf{I}-\frac{\beta}{n_1}\Xb^{\top}\Xb) \mathbf{x} \mathbf{x}^{\top}(\mathbf{I}-\frac{\beta}{n_1}\Xb^{\top}\Xb)\right]\\
    &+\frac{n_2-1}{n_2}\mathbb{E}\left[(\mathbf{I}-\frac{\beta}{n_2}\Xb^{\top}\Xb)\bSigma (\mathbf{I}-\frac{\beta}{n_1}\Xb^{\top}\Xb) \mathbf{A}(\mathbf{I}-\frac{\beta}{n_1}\Xb^{\top}\Xb)\bSigma(\mathbf{I}-\frac{\beta}{n_1}\Xb^{\top}\Xb)\right]\\
    &\overset{(a)}{\succeq}  \frac{b_1}{n_2} \mathbb{E}\left[\operatorname{tr}(\mathbf{A}(\mathbf{I}-\frac{\beta}{n}\Xb^{\top}\Xb) \Sigma(\mathbf{I}-\frac{\beta}{n}\Xb^{\top}\Xb))(\mathbf{I}-\frac{\beta}{n}\Xb^{\top}\Xb)\Sigma(\mathbf{I}-\frac{\beta}{n}\Xb^{\top}\Xb)\right]\\
    &+\Hb_{n_1,\beta} \mathbf{A }\Hb_{n_1,\beta}\\
    &{\succeq} \frac{b_1}{n_2} \operatorname{tr}(\Hb_{n_1,\beta} \mathbf{A}) \Hb_{n_1,\beta}+\Hb_{n_1,\beta} \mathbf{A }\Hb_{n_1,\beta}
\end{align*}
where $(a)$ is implied by Assumption 1.

Recall that $\Pi$ takes the following form:
\begin{align*}
 \Pi&=\frac{\sigma^2}{n_2}\Hb_{n_1,\beta}+\mathbb{E}[\Bb^{\top}\Bb\bSigma_{\btheta}\Bb^{\top}\Bb]+\sigma^2\cdot \frac{\beta^{2}}{n_2 n_1^2}\mathbb{E}[\Bb^{\top}\Xb^{\text{out}}{\Xb}^{\top}\Xb{\Xb^{\text{out}}}^{\top}\Bb].
 \end{align*}
 The second term can be directly bounded by letting $\Ab=\bSigma_{\btheta}$ in \cref{eq-lower-data}, and we have:
 $$\mathbb{E}[\Bb^{\top}\Bb\bSigma_{\btheta}\Bb^{\top}\Bb]\succeq \frac{b_1}{n_2} \operatorname{tr}(\Hb_{n_1,\beta} \bSigma_{\btheta}) \Hb_{n_1,\beta}.$$
 For the third term:
  \begin{align*}
   & \frac{1}{n_2}\mathbb{E}[\Bb^{\top}\Xb^{\text{out}}{\Xb}^{\top}\Xb{\Xb^{\text{out}}}^{\top}\Bb]\\
    &=\frac{1}{n_2}\mathbb{E}\left[(\mathbf{I}-\frac{\beta}{n_1}\Xb^{\top}\Xb)\mathbf{x x}^{\top}\Xb^{\top}\Xb  \mathbf{x} \mathbf{x}^{\top}(\mathbf{I}-\frac{\beta}{n_1}\Xb^{\top}\Xb)\right]\\
   &+\frac{n_2-1}{n_2}\mathbb{E}\left[(\mathbf{I}-\frac{\beta}{n_1}\Xb^{\top}\Xb)\bSigma\Xb^{\top}\Xb  \bSigma(\mathbf{I}-\frac{\beta}{n_1}\Xb^{\top}\Xb)\right]\\
    &\succeq \frac{n_1 b_1\operatorname{tr}(\bSigma^2)}{n_2}\Hb_{n_1,\beta}\mathbf{1}_{\beta\leq 0}
 \end{align*}
Putting these results together completes the proof.
\end{proof}

\subsection{Bias-Variance Decomposition}
For the lower bound analysis, we also decompose the excess risk into bias and variance terms.
\begin{lemma}[Bias-variance decomposition, lower bound]\label{lemma-bv-lower}
Following the notations in \cref{eq-bv}, the excess risk can be decomposed as follows:
 $$
\begin{aligned}
    R(\wl_{T},\bte)\geq \underline{\mathcal{E}_{bias}}+\underline{\mathcal{E}_{var}} 
\end{aligned}
$$
where 
\begin{align*}
  \underline{\mathcal{E}_{bias}}= & \frac{1}{2 T^{2}} \cdot \sum_{t=0}^{T-1} \sum_{k=t}^{T-1}\left\langle(\mathbf{I}-\alpha  \Htr)^{k-t} \Hte, \mathbf{D}_{t}\right\rangle, \\
\underline{\mathcal{E}_{var}} =&\frac{1}{2 T^{2}} \cdot \sum_{t=0}^{T-1} \sum_{k=t}^{T-1}\left\langle(\mathbf{I}-\alpha  \Htr)^{k-t} \Hte, \mathbf{V}_{t}\right\rangle.
\end{align*}

\end{lemma}
\begin{proof}
The proof is similar to that for \Cref{lemma-further-bv}, and the inequality sign is reversed since we only calculate the half of summation. In particular,
\begin{align*}
    \mathbb{E}[\rhob^{\text{var}}_T\otimes\rhob^{\text{var}}_T]& =\frac{1}{T^{2}} \sum_{1\leq t<k\leq  T-1} \mathbb{E}[\brho^{\text{var}}_t\otimes\brho^{\text{var}}_k]+\frac{1}{T^{2}} \sum_{1\leq k<t\leq  T-1} \mathbb{E}[\brho^{\text{var}}_t\otimes\brho^{\text{var}}_k]\\
   & \succeq\frac{1}{T^{2}} \sum_{1\leq t<k\leq  T-1} \mathbb{E}[\brho^{\text{var}}_t\otimes\brho^{\text{var}}_k].
\end{align*}
For $t\leq k$, $\mathbb{E}[\brho^{\text{var}}_k|\brho^{\text{var}}_t]=(\mathbf{I}-\alpha\Htr)^{k-t} \brho^{\text{var}}_t$, since $\mathbb{E}[\Bb_t^{\top}\bxi_t|\brho_{t-1}]=\mathbf{0}$. From this 
\begin{align*}
    \mathbb{E}[\rhob^{\text{var}}_T\otimes\rhob^{\text{var}}_T]
   & \succeq \frac{1}{T^{2}} \sum_{t=0}^{T-1} \sum_{k=t}^{T-1}  \Vb_t(\mathbf{I}-\alpha\Htr)^{k-t}.
\end{align*}
Plugging this into  $\frac{1}{2} \langle\Hte, \mathbb{E}[\rhob^{\text{var}}_T\otimes\rhob^{\text{var}}_T] \rangle$, we obtain:
\begin{align*}
     &\frac{1}{2} \langle\Hte, \mathbb{E}[\rhob^{\text{var}}_T\otimes\rhob^{\text{var}}_T] \rangle\\
    &\geq \frac{1}{2T^{2}} \sum_{t=0}^{T-1} \sum_{k=t+1}^{T-1} \langle \Hte, \Vb_t(\mathbf{I}-\alpha\Htr)^{k-t}\rangle \\
    &=\frac{1}{2T^{2}} \sum_{t=0}^{T-1} \sum_{k=t}^{T-1} \langle (\mathbf{I}-\alpha\Htr)^{k-t}\Hte, \Vb_t\rangle
    \\&=\underline{\mathcal{E}_\text{var}}.
\end{align*}
The proof is the same for the term $\underline{\mathcal{E}_\text{bias}}$.
\end{proof}
\subsection{Bounding the Bias}
We first bound the summation of $\Db_t$, i.e. $\Sb_k= \sum^{k-1}_{t=0}\Db_t$.
\begin{lemma}[Bounding $\mathbf{S}_t$]\label{lemma-sk}
 If the stepsize satisfies $\alpha <1 / (2\max_{i}\{\mu_{i}(\Htr)\})$, then for any $k \geq 2$, it holds that
    \begin{align*}
           \mathbf{S}_{k} &\succeq \frac{b_1}{4n_2} \operatorname{tr}\left(\left(\mathbf{I}-(\mathbf{I}-\alpha  \Htr)^{k / 2}\right) \mathbf{D}_{0}\right) \cdot\left(\mathbf{I}-(\mathbf{I}-\alpha  \Htr)^{k / 2}\right)\\
           &+\sum_{t=0}^{k-1}(\mathbf{I}-\alpha  \Htr)^{t} \cdot \mathbf{D}_{0} \cdot(\mathbf{I}-\alpha  \Htr)^{t}.
    \end{align*}

\end{lemma}
\begin{proof} By \cref{eq-st}, since $\tilde{\mathcal{M}}-\mathcal{M}$ is a PSD mapping, we have
\begin{align}
  \Sb_k &=\Db _{0}+(\mathcal{I}-\alpha \widetilde{\mathcal{T}}) \circ \mathbf{S}_{k-1}+ \alpha^2(\mathcal{M}-\widetilde{\mathcal{M}})\circ \mathbf{S}_{k-1}\label{eq-sk}\\
  & \succeq \sum^{k-1}_{t=0} (\mathcal{I}-\alpha \widetilde{\mathcal{T}})^{t}
  \circ \Db_0\nonumber\\
  &=\sum_{t=0}^{k-1}(\mathbf{I}-\alpha \Htr)^{t} \cdot \mathbf{D}_{0} \cdot(\mathbf{I}-\alpha \Htr)^{t}.\nonumber
\end{align}
Note that for PSD $\Ab$,  $$(\mathcal{M}-\widetilde{\mathcal{M}}) \circ\Ab= \mathbb{E}[\Bb^{\top}\Bb\Ab\Bb^{\top}\Bb] - \Htr \mathbf{A }\Htr
$$ By \Cref{lemma-4l}, we have
    \begin{align}
        (\mathcal{M}-\widetilde{\mathcal{M}}) \circ \mathbf{S}_{k} & \succeq \frac{b_1}{n_2} \operatorname{tr}\left(\Htr \mathbf{S}_{k}\right) \Htr\nonumber \\
        & \succeq \frac{b_1}{n_2} \operatorname{tr}\left(\sum_{t=0}^{k-1}(\mathbf{I}-\alpha \Htr)^{2 t} \Htr \cdot \mathbf{D}_{0}\right) \Htr\nonumber \\
        & \succeq \frac{b_1}{n_2} \operatorname{tr}\left(\sum_{t=0}^{k-1}(\mathbf{I}-2 \alpha\Htr)^{t} \Htr\cdot \mathbf{D}_{0}\right) \Htr\nonumber \\
        & \succeq \frac{b_1}{2 n_2 \alpha} \operatorname{tr}\left(\left(\mathbf{I}-(\mathbf{I}-\alpha \Htr)^{k}\right) \mathbf{D}_{0}\right) \Htr.\label{eq-crude}
        \end{align}
        Substituting \cref{eq-crude} back into \cref{eq-sk}, and solving the recursion, we obtain
        $$
\begin{aligned}
\mathbf{S}_{k} \succeq & \sum_{t=0}^{k-1}(\mathcal{I}-\alpha \tilde{\mathcal{T}})^{t} \circ\left\{\frac{b_1 \alpha}{2n_2} \operatorname{tr}\left(\left(\mathbf{I}-(\mathbf{I}-\alpha\Htr)^{k-1-t}\right) \mathbf{D}_{0}\right) \mathbf{H}+\mathbf{D}_{0}\right\} \\
=& \frac{b_1 \alpha}{2n_2} \underbrace{\sum_{t=0}^{k-1} \operatorname{tr}\left(\left(\mathbf{I}-(\mathbf{I}-\alpha \Htr)^{k-1-t}\right) \mathbf{D}_{0}\right) \cdot(\mathbf{I}-\alpha\Htr)^{2 t} \Htr}_{\Jb_4} \\
&+\sum_{t=0}^{k-1}(\mathbf{I}-\alpha\Htr)^{t} \cdot \mathbf{D}_{0} \cdot(\mathbf{I}-\alpha \Htr)^{t}.
\end{aligned}
$$
The term $\Jb_4$ can be further bounded by the following:
\begin{align*}
    \Jb_4&\succeq \sum_{t=0}^{k-1} \operatorname{tr}\left(\left(\mathbf{I}-(\mathbf{I}-\alpha \Htr)^{k-1-t}\right) \mathbf{D}_{0}\right) \cdot(\mathbf{I}-2\alpha\Htr)^{t}\Htr\\
    &\succeq  \operatorname{tr}\left(\left(\mathbf{I}-(\mathbf{I}-\alpha \Htr)^{k/2}\right) \mathbf{D}_{0}\right) \cdot\sum_{t=0}^{k/2-1}(\mathbf{I}-2\alpha\Htr)^{t}\Htr\\
        &\succeq  \frac{1}{2\alpha}\operatorname{tr}\left(\left(\mathbf{I}-(\mathbf{I}-\alpha \Htr)^{k/2}\right) \mathbf{D}_{0}\right) \cdot\left(\mathbf{I}-(\mathbf{I}-\alpha\Htr)^{k/2}\right)
\end{align*}
which completes  the proof.
\end{proof}
Then we can bound the bias term.
\begin{lemma}[Bounding the bias]\label{lemma-biasl}
 Let $\omega_i=\left\langle\boldsymbol{\omega}_{0}-\boldsymbol{\theta}^{*}, \mathbf{v}_{i}\right\rangle$. If $\alpha<\frac{1}{c(\btr,\bSigma)\operatorname{tr}(\bSigma)}$, for sufficiently large $n_1$, s.t. $\mu_i(\Htr)>0$, $\forall i$, then we have 
 \begin{align*}
     \underline{\mathcal{E}_{\text{bias}}}\ge&   \frac{1}{100\alpha^2 T} \sum_{i}\Xi_i \frac{\omega_i^2}{\mu_i(\Hb_{n_1,\beta^{\text{tr}}})} + \frac{b_1}{1000n_2(1-\alpha c(\btr,\bSigma) \operatorname{tr}(\bSigma))}\sum_{i}\Xi_{i}   \\
   &\times
  \sum_{i}\Big( \frac{\mathbf{1}_{\mu_{i}(\Hb_{n_1,\beta^{\text{tr}}})\geq \frac{1}{\alpha T}}}{T\alpha \mu_i(\Hb_{n_1,\beta^{\text{tr}}})}+\mathbf{1}_{\mu_{i}(\Hb_{n_1,\beta^{\text{tr}}})< \frac{1}{\alpha T} } \Big) \lambda_{i}\omega^2_i.
 \end{align*}
\end{lemma}
\begin{proof}
From \Cref{lemma-bv-lower}, we have
\begin{align*}
      \underline{\mathcal{E}_{bias}}&=  \frac{1}{2 T^{2}} \cdot \sum_{t=0}^{T-1} \sum_{k=t}^{T-1}\left\langle(\mathbf{I}-\alpha  \Htr)^{k-t} \Hte, \mathbf{D}_{t}\right\rangle \\
      &=\frac{1}{2\alpha T^{2}} \cdot \sum_{t=0}^{T-1} \left\langle\left(\mathbf{I}-(\mathbf{I}-\alpha  \Htr)^{T-t}\right)\Htr^{-1} \Hte, \mathbf{D}_{t}\right\rangle\\
      &\ge \frac{1}{2 \alpha T^{2}}  \left\langle\left(\mathbf{I}-(\mathbf{I}-\alpha  \Htr)^{T/2}\right)\Htr^{-1} \Hte, \sum_{t=0}^{T/2}\mathbf{D}_{t}\right\rangle\\
      &\ge \frac{1}{2\alpha T^{2}}  \left\langle\left(\mathbf{I}-(\mathbf{I}-\alpha  \Htr)^{T/2}\right)\Htr^{-1} \Hte, \mathbf{S}_{\frac{T}{2}}\right\rangle.
\end{align*}
Applying \Cref{lemma-sk} to $\mathbf{S}_{\frac{T}{2}}$, we obtain:
\begin{align*}
    \underline{\mathcal{E}_{bias}}
    \ge& \underline{d_1}+\underline{d_2}
    \end{align*}
    where
    \begin{align*}
   \underline{d_1}&=\frac{b_1}{8\alpha n_2T^2} \operatorname{tr}\left(\left(\mathbf{I}-(\mathbf{I}-\alpha  \Htr)^{T / 4}\right) \mathbf{D}_{0}\right)\\ &\times\left\langle\left(\mathbf{I}-(\mathbf{I}-\alpha  \Htr)^{T/2}\right)\Htr^{-1} \Hte,\right.
    \left.  \left(\mathbf{I}-(\mathbf{I}-\alpha  \Htr)^{T /4}\right)\right\rangle\\
    \underline{d_2}&=\frac{1}{2\alpha T^2 }\left\langle\left(\mathbf{I}-(\mathbf{I}-\alpha  \Htr)^{T/2}\right)\Htr^{-1} \Hte,\right.
    \\&\left. \sum_{t=0}^{T/2-1}(\mathbf{I}-\alpha  \Htr)^{t} \cdot \mathbf{D}_{0} \cdot(\mathbf{I}-\alpha  \Htr)^{t}\right\rangle.
\end{align*}
Moreover,
\begin{align*}
\underline{d_2}&\ge\frac{1}{2\alpha T^2 }\left\langle\left(\mathbf{I}-(\mathbf{I}-\alpha  \Htr)^{T/2}\right)\Htr^{-1} \Hte,
    \sum_{t=0}^{T/2-1}(\mathbf{I}-2\alpha  \Htr)^{t} \mathbf{D}_{0} \right\rangle;\\
    &\ge\frac{1}{4\alpha^2 T^2 }\left\langle\left(\mathbf{I}-(\mathbf{I}-\alpha  \Htr)^{T/2}\right)^2\Htr^{-2} \Hte,
\mathbf{D}_{0} \right\rangle.
\end{align*}
Using the diagonalizing technique similar to the proof for \Cref{lemma-bias}, we have
\begin{align}
    \underline{d_1}&\ge \frac{b_1}{8 \alpha n_2 T^{2}}\left(\sum_{i}\left(1-\left(1-\alpha  \mu_{i}(\Htr)\right)^{T / 4}\right) \omega_{i}^{2}\right)\\
    &\times\left(\sum_{i}\left(1-\left(1-\alpha \mu_{i}(\Htr)\right)^{T / 4}\right)^{2} \frac{\mu_i(\Hte)}{\mu_i(\Htr)}\right)\label{b1},\\
       \underline{d_2}&\geq \frac{1}{4 \alpha^{2} T^{2}} \sum_{i}\left(1-\left(1-\alpha \mu_{i}(\Htr)\right)^{T / 4}\right)^{2} \frac{\mu_i(\Hte)}{\mu_i^2(\Htr) } \omega_{i}^{2}.\label{b2}
\end{align}
We use the following fact to bound the polynomial term. For $h_1(x)=1-(1-x)^{\frac{T}{4}}$, we have
$$
h_1(x)\ge\begin{cases}
\frac{1}{5}& x\ge 1/T\\
\frac{T}{5} x& x< 1/T
\end{cases}
$$
i.e., $1-\left(1-\alpha \mu_{i}(\Htr)\right)^{T / 4}\geq \left(\frac{1}{5}\mathbf{1}_{\alpha \mu_{i}(\Htr)\ge \frac{1}{T}}+\frac{\alpha \mu_{i}(\Htr)}{5}\mathbf{1}_{\alpha \mu_{i}(\Htr)< \frac{1}{T}}\right)$. Substituting this back into \cref{b1,b2}, and using the definition of effective meta weight $\Xi_i$ complete the proof.
\end{proof}

\subsection{Bounding the Variance}
We first bound the term $\Vb_t$.
\begin{lemma}[Bounding $\mathbf{V}_{t}$]\label{lemma-vt}
   If the stepsize satisfies $\alpha <1 / (\max_{i}\{\mu_{i}(\Htr)\})$, it holds that
$$
\mathbf{V}_{t} \succeq \frac{\alpha g(\btr,n_1, \bSigma, \bSigma_{\btheta})}{2} \cdot\left(\mathbf{I}-(\mathbf{I}-\alpha  \Htr)^{2 t}\right).
$$
\end{lemma}
\begin{proof}
With a slight abuse of notations, we write $g(\btr,n_1, \bSigma, \bSigma_{\btheta})$ as $g$. By definition,
    $$
\begin{aligned}
\mathbf{V}_{t} &=(\mathcal{I}-\alpha  \mathcal{T}) \circ \mathbf{V}_{t-1}+\alpha ^{2} \Pi \\
&=(\mathcal{I}-\alpha  \widetilde{\mathcal{T}}) \circ \mathbf{V}_{t-1}+(\mathcal{M}-\widetilde{\mathcal{M}}) \circ \mathbf{V}_{t-1}+\alpha ^{2} \Pi \\
&\overset{(a)}{\succeq}(\mathcal{I}-\alpha  \widetilde{\mathcal{T}}) \circ \mathbf{V}_{t-1}+\alpha ^{2} g \Htr \\
&\overset{(b)}{=}\alpha ^{2}  g \cdot \sum_{k=0}^{t-1}(\mathcal{I}-\alpha  \widetilde{\mathcal{T}})^{k} \circ \Htr \\
&=\alpha ^{2}g \cdot \sum_{k=0}^{t-1}(\mathbf{I}-\alpha  \Htr)^{k} \Htr(\mathbf{I}-\alpha  \Htr)^{k} \quad \text { (by the definition of } \mathcal{I}-\alpha  \widetilde{\mathcal{T}} ) \\
&=\alpha   g \cdot\left(\mathbf{I}-(\mathbf{I}-\alpha  \Htr)^{2 t}\right) \cdot\left(2   \mathbf{I}-\alpha  \Htr\right)^{-1} \\
& \overset{(c)}{\succeq}  \frac{\alpha g}{2} \cdot\left(\mathbf{I}-(\mathbf{I}-\alpha  \Htr)^{2 t}\right)
\end{aligned}
$$
where $(a)$ follows from the \Cref{lemma-4l}, $(b)$ follows by solving the recursion and $(c)$ holds since we directly replace $\left(2   \mathbf{I}-\alpha  \Htr\right)^{-1} $  by $(2\Ib)^{-1}$.
\end{proof}
\begin{lemma}[Bounding the variance]\label{lemma-varl}
 Let $\omega_i=\left\langle\boldsymbol{\omega}_{0}-\boldsymbol{\theta}^{*}, \mathbf{v}_{i}\right\rangle$. If $\alpha<\frac{1}{c(\btr,\bSigma)\operatorname{tr}(\bSigma)}$, for sufficiently large $n_1$, s.t. $\mu_i(\Htr)>0$, $\forall i$, for $T>10$, then we have 
 \begin{align*}
     \underline{\mathcal{E}_{\text{var}}}&\ge     \frac{g(\btr,n_1, \bSigma, \bSigma_{\btheta})}{100n_2(1-\alpha c(\btr,\bSigma) \operatorname{tr}(\bSigma))}\sum_{i}\Xi_{i}.  
 \end{align*}
\end{lemma}
\begin{proof}
From \Cref{lemma-bv-lower}, we have
\begin{align*}
      \underline{\mathcal{E}_{var}}&=  \frac{1}{2 T^{2}} \cdot \sum_{t=0}^{T-1} \sum_{k=t}^{T-1}\left\langle(\mathbf{I}-\alpha  \Htr)^{k-t} \Hte, \mathbf{V}_{t}\right\rangle \\
      &=\frac{1}{2\alpha T^{2}} \cdot \sum_{t=0}^{T-1} \left\langle\left(\mathbf{I}-(\mathbf{I}-\alpha  \Htr)^{T-t}\right)\Htr^{-1} \Hte, \mathbf{V}_{t}\right\rangle.
\end{align*}
Then applying \Cref{lemma-vt}, and writting $g(\btr,n_1, \bSigma, \bSigma_{\btheta})$ as $g$,  we obtain
\begin{align}
      \underline{\mathcal{E}_{var}}&\ge \frac{g}{4 T^{2}} \cdot \sum_{t=0}^{T-1} \left\langle\left(\mathbf{I}-(\mathbf{I}-\alpha  \Htr)^{T-t}\right)\Htr^{-1} \Hte, \left(\mathbf{I}-(\mathbf{I}-\alpha  \Htr)^{2 t}\right)\right\rangle\nonumber\\
      &=\frac{g}{4 T^{2}} \sum_{i}\frac{\mu_i(\Hte)}{\mu_i(\Htr)}\sum_{t=0}^{T-1}(1-(1-\alpha \mu_i(\Htr)^{T-t}))(1-(1-\alpha \mu_i(\Htr)^{2t}))\nonumber\\
      &\ge \frac{g}{4 T^{2}} \sum_{i}\frac{\mu_i(\Hte)}{\mu_i(\Htr)}\sum_{t=0}^{T-1}(1-(1-\alpha \mu_i(\Htr)^{T-t-1}))(1-(1-\alpha \mu_i(\Htr)^{t})) \label{eq-var}
\end{align}
where the equality holds by applying the diagonalizing technique again. Following the trick similar to that in \cite{zou2021benign} to lower bound the function $h_2(x):=\sum_{t=0}^{T-1}\left(1-(1-x)^{T-t-1}\right)\left(1-(1-x)^{t}\right)$ defined on $x\in(0,1)$, for $T>10$, we have 
$$
f(x) \geq \begin{cases}\frac{T}{20}, & \frac{1}{T} \leq x<1 \\ \frac{3 T^{3}}{50} x^{2}, & 0<x<\frac{1}{T}\end{cases}
$$
Substituting this back into \cref{eq-var}, and using the definition of effective meta weight $\Xi_i$ completes the proof.
\end{proof}

\subsection{Proof of Theorem~\ref{thm-lower}}
\begin{theorem}[\Cref{thm-lower} Restated]\label{ap-thm-lower}
Let $\omega_i=\left\langle\boldsymbol{\omega}_{0}-\boldsymbol{\theta}^{*}, \mathbf{v}_{i}\right\rangle$. If $|\btr|,|\bte|<1/\lambda_1$, $n_1$ is large ensuring that $\mu_i(\Hb_{n_1,\beta^{\text{tr}}})>0$, $\forall i$ and
$\alpha<1/\left(c(\btr,\bSigma) \operatorname{tr}(\bSigma)\right)$. For $T>10$, the meta excess risk $R(\overline{\boldsymbol{\omega}}_T,\bte)$ is bounded below as follows 
  \begin{align*}
 R(\overline{\boldsymbol{\omega}}_T,\bte)  \ge &\frac{1}{100\alpha^2 T} \sum_{i}\Xi_i \frac{\omega_i^2}{\mu_i(\Hb_{n_1,\beta^{\text{tr}}})} +\frac{1}{n_2}\cdot \frac{1}{(1-\alpha c(\btr,\bSigma) \operatorname{tr}(\bSigma))}\sum_{i}\Xi_{i}   \\
   \times & [\frac{1}{100} g(\btr,n_1, \bSigma, \bSigma_{\btheta})+\frac{b_1}{1000}
  \sum_{i}\Big( \frac{\mathbf{1}_{\mu_{i}(\Hb_{n_1,\beta^{\text{tr}}})\geq \frac{1}{\alpha T}}}{T\alpha \mu_i(\Hb_{n_1,\beta^{\text{tr}}})}+\mathbf{1}_{\mu_{i}(\Hb_{n_1,\beta^{\text{tr}}})< \frac{1}{\alpha T} } \Big) \lambda_{i}\omega^2_i].
 \end{align*}
\end{theorem}
\begin{proof}
The proof can be completed by combining \Cref{lemma-biasl,lemma-varl}.
\end{proof}
\section{Proofs for Section~\ref{sec-main-task}}
\subsection{Proof of Lemma~\ref{lem-single}}
\begin{proof}[Proof of \Cref{lem-single}]
For the single task setting, we first simplify our notations in \Cref{thm-upper} as follows.
\begin{align*} c(0,\bSigma) = c_1,\quad
  f(0,n_2,\sigma,\bSigma,\mathbf{0})=\sigma^2/n_2,\quad \Htr = \bSigma.
\end{align*}
By \Cref{thm-upper}, we have 
\begin{align*}
        \text{Bias}  & =  \frac{2}{\alpha^2 T} \sum_{i}\left(\frac{1}{T}\mathbf{1}_{\lambda_i\geq \frac{1}{\alpha T}}+T\alpha^2\lambda^2_i\mathbf{1}_{\lambda_i< \frac{1}{\alpha T} } \right) \frac{\omega_i^2\mu_i(\Hte)}{\lambda^2_i}\\
        &\leq \frac{2}{\alpha^2 T} \sum_{i}(\alpha \lambda_i\mathbf{1}_{\lambda_i\geq \frac{1}{\alpha T}}+\alpha\lambda_i\mathbf{1}_{\lambda_i< \frac{1}{\alpha T} } )\frac{\omega_i^2\mu_i(\Hte)}{\lambda^2_i}.
\end{align*}
For large $m$, we have $\mu_i(\Hte)=(1-\bte\lambda_i)^2\lambda_i+o(1)$. Therefore, 
\begin{align*}
        \text{Bias}  
        \leq \frac{2(1-\bte\lambda_d)^2}{\alpha^2 T} \sum_{i} {\omega_i^2}\leq \mathcal{O}(\frac{1}{T}).
\end{align*}
For the variance term, 
\begin{align*}
    \text{Var}   &= \frac{2}{(1-\alpha c_1 \operatorname{tr}(\bSigma))}\underbrace{\sum_{i}\left(\frac{1}{T}\mathbf{1}_{\lambda_i\geq \frac{1}{\alpha T}}+T\alpha^2\lambda^2_i\mathbf{1}_{\lambda_i< \frac{1}{\alpha T} } \right) \frac{\mu_i(\Hte)}{\lambda_i}}_{J_5}
   \\
 \quad \times & [\frac{\sigma^2}{n_2}+ 2c_1
 \sum_{i}\left( \frac{\mathbf{1}_{\lambda_i\geq \frac{1}{\alpha T}}}{T\alpha \lambda_i}+\mathbf{1}_{\lambda_i< \frac{1}{\alpha T} } \right) \lambda_{i}\omega^2_i].
\end{align*}
It is easy to check that
$$
\sum_{i}\left( \frac{\mathbf{1}_{\lambda_i\geq \frac{1}{\alpha T}}}{T\alpha \lambda_i }+\mathbf{1}_{\lambda_i< \frac{1}{\alpha T} } \right) \lambda_{i}\omega^2_i\leq \sum_{i}\left( \frac{\mathbf{1}_{\lambda_i\geq \frac{1}{\alpha T}}}{T\alpha }+\frac{1}{\alpha T}\mathbf{1}_{\lambda_i< \frac{1}{\alpha T} } \right) \omega^2_i \leq \mathcal{O}(1/T).
$$
Moreover, 
$$
J_5\leq (1-\bte\lambda_d)^2 \sum_{i}\left(\frac{1}{T}\mathbf{1}_{\lambda_i\geq \frac{1}{\alpha T}}+T\alpha^2\lambda^2_i\mathbf{1}_{\lambda_i< \frac{1}{\alpha T} } \right).
$$
The term $\sum_{i}\left(\frac{1}{T}\mathbf{1}_{\lambda_i\geq \frac{1}{\alpha T}}+T\alpha^2\lambda^2_i\mathbf{1}_{\lambda_i< \frac{1}{\alpha T} } \right)$ has the form similar to Corollary 2.3 in \cite{zou2021benign}
and  we directly have
$J_5=\mathcal{O}\left(\log^{-p}(T)\right)$, which implies 
\begin{align*}
        \text{Var} =\mathcal{O}\left(\log^{-p}(T)\right).
\end{align*}
Thus we complete the proof.
\end{proof}

\subsection{Proof of Proposition~\ref{prop-hard}}
\begin{proof}[Proof of \Cref{prop-hard}]\label{proof-prop2}
We first consider the bias term in \Cref{thm-upper,thm-lower} (up to absolute constants):
\begin{align*}
  \text{Bias}&= \frac{2}{\alpha^{2} T} \sum_{i}\left(\frac{1}{ T}\mathbf{1}_{\mu_i(\Htr)\geq \frac{1}{\alpha T}}+\alpha^{2} T \mu_i^2(\Htr)\mathbf{1}_{\mu_i(\Htr)< \frac{1}{\alpha T} } \right)\frac{\omega_i^2\mu_i(\Hte)}{\mu_i(\Htr)^2}.
\end{align*}
If $\mu_i(\Htr)\geq \frac{1}{\alpha T}$, $\frac{1}{T}\leq \alpha\mu_i(\Htr) $; and if $\mu_i(\Htr)< \frac{1}{\alpha T}$, then $\alpha^{2} T \mu_i^2(\Htr)<\alpha \mu_i(\Htr) $. Hence
\begin{align*}
  \text{Bias}\le \frac{1}{\alpha^{2} T} \sum_{i}\frac{\omega_i^2\mu_i(\Hte)}{\mu_i(\Htr)}\le  \frac{2}{\alpha^{2} T}\cdot\max_{i}\frac{\mu_i(\Hte)}{\mu_i(\Htr)} \|\bomega_0-\btheta^{*}\|^2=\mathcal{O}(\frac{1}{T}).
\end{align*}
Moreover, 
\begin{align*}
    \sum_{i}&\left( \frac{\mathbf{1}_{\mu_{i}(\Hb_{n_1,\beta^{\text{tr}}})\geq \frac{1}{\alpha T}}}{T\alpha \mu_i(\Hb_{n_1,\beta^{\text{tr}}})}+\mathbf{1}_{\mu_{i}(\Hb_{n_1,\beta^{\text{tr}}})< \frac{1}{\alpha T} } \right) \lambda_{i}\omega^2_i
    \\
     &\overset{(a)}{\le} \frac{1}{\alpha T}\sum_{i} \frac{\lambda_{i}}{\mu_{i}(\Hb_{n_1,\beta^{\text{tr}}})}\omega^2_i\\ &\le \frac{1}{\alpha T}\max_{i} \frac{\lambda_{i}}{\mu_{i}(\Hb_{n_1,\beta^{\text{tr}}})}\|\bomega_0-\btheta^{*}\|^2=\mathcal{O}(\frac{1}{T})
\end{align*}
where $(a)$ holds since we directly upper bound $\mu_{i}(\Htr)$ by $\frac{1}{\alpha T}$ when $\mu_{i}(\Hb_{n_1,\beta^{\text{tr}}})< \frac{1}{\alpha T}$. Therefore, it is essential to analyze $  f(\btr,n_2,\sigma,\bSigma,\bSigma_{\boldsymbol{\theta}})\left(\sum_{i}\Xi_{i} \right)$ and $  g(\btr,n_1,\sigma,\bSigma,\bSigma_{\boldsymbol{\theta}})\left(\sum_{i}\Xi_{i} \right)$ from variance term in the upper and lower bounds respectively.

Then we calculate some rates of interesting in \Cref{thm-upper,thm-lower} under the specific data and task distributions in \Cref{prop-hard}.

If the spectrum of $\bSigma$ satisfies $\lambda_{k}=k^{-1} \log ^{-p}(k+1)$, then it is easily verified that  $\operatorname{tr}(\bSigma^{s})=O(1)$ for $s=1,\cdots,4$. By discussions on \Cref{ass:higherorder} in \Cref{sec-diss}, we have $C(\beta,\bSigma
)=\Theta(1)$ for given $\beta$. Hence,
\begin{align*} c(\btr,\bSigma) &= \Theta(1)\\
  f(\btr,n_2,\sigma,\bSigma,\bSigma_{\boldsymbol{\theta}})&=c(\btr,\bSigma)\operatorname{tr}({\bSigma_{\boldsymbol{\theta}}\bSigma})+ \Theta(1)\\
  g(\btr,n_1, \sigma,\bSigma, \bSigma_{\btheta}) & =b_1\operatorname{tr}(\bSigma_{\btheta}\Htr)+\Theta(1).
\end{align*}

If $r\ge  2p-1$, then we have $g(\btr,n_1, \sigma,\bSigma, \bSigma_{\btheta})\ge \Omega\left(\log^{r-p+1}(d)\right)\ge \Omega\left(\log^{r-p+1}(T)\right)$.

Let $k^{\dagger}:= \operatorname{card}\{i: \mu_{i}(\Htr)\ge 1/\alpha T\}$. For large $n_1$, we have $\mu_i(\Htr)=(1-\btr\lambda_i)^2\lambda_i+o(1)$.  If $k^{\dagger}=\mathcal{O}\left(T/\log^{p}(T+1)\right)$, then
    $$\min_{1 \le i\le k^{\dagger}+1 }\mu_{i}(\Htr)=\omega\left(\frac{\log^{p}(T)}{T[\log(T)-p\log(\log(T))]^p}\right)=\omega\left(\frac{1}{T}\right)$$
which contradicts the definition of $k^{\dagger}$. Hence $k^{\dagger}=\Omega\left(T/\log^{p}(T+1)\right)$. Then 
\begin{align*}
    \sum_{i}\Xi_{i}\ge  \Omega\left(k^{\dagger}\cdot \frac{1}{T} \right)=\Omega\left(\frac{1}{\log^{p}(T)}\right).
\end{align*}
Therefore, by \Cref{thm-lower}, $R(\wl,\bte)=\Omega\left(\log^{r-2p+1}(T)\right)$.

For $r< 2p-1$, if $d=T^l$, where $l$ can be sufficiently large ($d\gg T$) but still finite, then
\begin{itemize}
\item If $p-1<r< 2p-1$,  $f(\btr,n_2,\sigma,\bSigma,\bSigma_{\boldsymbol{\theta}})\le \mathcal{O}(\log^{r-p+1} T)$;
\item If $r\leq p-1$, $f(\btr,n_2,\sigma,\bSigma,\bSigma_{\boldsymbol{\theta}})\le \mathcal{O}\Big(\log\left(\log(T)\right)\Big)$. 
\end{itemize}
Following the analysis similar to that for Corollary 2.3 in \cite{zou2021benign}, we have $ \sum_{i}\Xi_{i}= \mathcal{O}(\frac{1}{\log^{p}(T)})$. Then by \Cref{thm-upper}
$$
R(\overline{\boldsymbol{\omega}}_T,\beta^{\text{te}})= \mathcal{O}\left(\frac{1}{\log^{p-(r-p+1)^{+}}(T)}\right).
$$
\end{proof}

\subsection{Proof of Proposition~\ref{prop-fast}}
\begin{proof}[Proof of \Cref{prop-fast}]
Following the analysis in \Cref{proof-prop2}, it is essential to analyze $  f(\btr,n_2,\sigma,\bSigma,\bSigma_{\boldsymbol{\theta}})\left(\sum_{i}\Xi_{i} \right)$.
If $d=T^l$, where $l$ can be sufficiently large but still finite, then $$f(\btr,n_2,\sigma,\bSigma,\bSigma_{\boldsymbol{\theta}})=\tilde{\Theta}(1)$$ for $\lambda_{k}=k^{q}$ $(q>1)$ or $\lambda_k=e^{-k}$.

Following the analysis similar to that for Corollary 2.3 in \cite{zou2021benign}, we have
\begin{itemize}
    \item If $\lambda_{k}=k^{q}$ $(q>1)$, then  $\sum_{i}\Xi_{i}=\mathcal{O}\left(\frac{1}{T^{\frac{q-1}{q}}}\right)$;
    \item If $\lambda_k=e^{-k}$, then $\sum_{i}\Xi_{i}=\mathcal{O}\left(\frac{\log(T)}{T}\right)$.
\end{itemize}
Substituting these results back into \Cref{thm-upper}, we obtain
\begin{itemize}
    \item If $\lambda_{k}=k^{q}$ $(q>1)$, then  $ R(\overline{\boldsymbol{\omega}}_T,\bte)=\tilde{\mathcal{O}}\left(\frac{1}{T^{\frac{q-1}{q}}}\right)$;
    \item If $\lambda_k=e^{-k}$, then $ R(\overline{\boldsymbol{\omega}}_T,\bte)=\tilde{\mathcal{O}}\left(\frac{1}{T}\right)$.
\end{itemize}
\end{proof}

\section{Proofs for Section~\ref{sec-main-stopping}}

\subsection{Proof of Proposition~\ref{prop-tradeoff}}\label{sec-prop4}
\begin{proof}[Proof of \Cref{prop-tradeoff}]
Following the analysis in \Cref{proof-prop2}, it is crucial to analyze $  f(\btr,n_2,\sigma,\bSigma,\bSigma_{\boldsymbol{\theta}})\left(\sum_{i}\Xi_{i} \right)$. 

Then we calculate the rate of interest in \Cref{thm-upper,thm-lower} under some specific data and task distributions in Proposition 4. We have $\operatorname{tr}(\bSigma^2)=\frac{1}{s}+\frac{1}{d-s}=\Theta(\frac{\log^{p}(T)}{T})$. Moreover, by discussions on Assumption 3 in \Cref{sec-diss}, $C(\beta,\bSigma)=\Theta(1)$. 
Hence 
\begin{align*} c(\beta,\bSigma) &:= c_1+\tilde{\mathcal{O}}(\frac{1}{T});\\
  f(\beta,n,\sigma,\bSigma,\bSigma_{\boldsymbol{\theta}})&:=2c_1\mathcal{O}(1)+\frac{\sigma^2}{n}+ \tilde{\mathcal{O}}\left(\frac{1}{T}\right).
\end{align*}
By the definition of $\Xi_i$, we have
\begin{align*}
    \sum_{i}\Xi_i&=\mathcal{O}\left(s\cdot \frac{\mu_1(\Hte)}{T\mu_1(\Htr)}+\frac{1}{d-s}\cdot T \frac{\mu_d(\Htr)\mu_d(\Hte)}{\lambda^2_d}\right)\\
    &= \mathcal{O}\left(\frac{1}{\log^{p}(T)} \right)\frac{\mu_1(\Hte)}{\mu_1(\Htr)}+\mathcal{O}\Big(\frac{1}{\log^{q}(T)} \Big)\frac{\mu_d(\Htr)\mu_d(\Hte)}{\lambda^2_d}\\
    &=\mathcal{O}\left(\frac{1}{\log^{p}(T)} \right)\frac{(1-\bte\lambda_1)^2}{(1-\btr\lambda_1)^2}+\mathcal{O}\Big(\frac{1}{\log^{q}(T)} \Big)(1-\bte\lambda_d)^2(1-\btr\lambda_d)^2
\end{align*}
where the last equality follows from the fact that for large $n$, we have $\mu_i(\Hb_{n,\beta})=(1-\beta\lambda_i)^2\lambda_i+o(1)$. Combining with the bias term which is $\mathcal{O}(\frac{1}{T})$, and applying \Cref{thm-upper} completes the proof.
\end{proof}

\subsection{Proof of Corollary~\ref{col-stop}}
\begin{proof}[Proof of \Cref{col-stop}]
For $t\in (s, K]$, 
by \Cref{thm-lower}, one can verify that $t=\tilde{\Theta}(K)$ for diminishing risk. Let $t=K\log^{-l}(K)$, where $p>l>0$. Following the analysis in \Cref{sec-prop4}, we have 
 \begin{align}
 &R(\overline{\boldsymbol{\omega}}^{\beta^{\text{tr}}}_t,\bte)\lesssim \widetilde{\mathcal{O}}(\frac{1}{K})
    +(2c_1\nu^2+\frac{\sigma^2}{n_2})\\
    &\times \left[\mathcal{O}\Big(\frac{1}{\log^{p-l}(K)}\Big) \frac{(1-\bte\lambda_1)^2}{(1-\btr \lambda_{1})^{2}}+\mathcal{O}\Big(\frac{1}{\log^{p+l} (K)}\Big)\Big(1-\btr \lambda_{d}\Big)^{2}\Big(1-\bte \lambda_{d}\Big)^{2}\right].
\end{align}
To clearly illustrate the trade-off in the stopping time, we let $l=0$ for convenience. If $R(\overline{\boldsymbol{\omega}}^{\beta^{\text{tr}}}_t,\bte)<\epsilon$, we have 
\begin{align*}
t_{\epsilon}\leq    \exp\Big(\epsilon^{-\frac{1}{p}}\Big[\frac{U_{l}}{(1-\btr\lambda_1)^2}+ U_{t} (1-\btr\lambda_d)^2\Big]^{\frac{1}{p}}\Big)
\end{align*}
where 
\begin{align*}
    U_{l}=\mathcal{O}\Big( (2c_1\nu^2+\frac{\sigma^2}{n_2})(1-\bte\lambda_1)^2\Big) \quad { and }\quad  U_{l} =\mathcal{O}\Big( (2c_1\nu^2+\frac{\sigma^2}{n_2})(1-\bte\lambda_d)^2\Big).
\end{align*}
The arguments are similar for the lower bound, and we can obtain:
\begin{align*}
    L_{l}=\mathcal{O}\Big( (2\frac{b_1\nu^2}{n_2}+\frac{\sigma^2}{n_2})(1-\bte\lambda_1)^2\Big) \quad { and }\quad  L_{l} =\mathcal{O}\Big( (2\frac{b_1\nu^2}{n_2}+\frac{\sigma^2}{n_2})(1-\bte\lambda_d)^2\Big).
\end{align*}
\end{proof}
\section{Discussions on Assumptions}\label{sec-diss}
\paragraph{Discussions on \Cref{ass-comm}}
If $\mathcal{P}_{\xb}$ is Gaussian distribution, then we have 
$$
F=\mathbb{E}[\xb\xb^{\top}\bSigma\xb\xb^{\top} ]= 2\bSigma^3+\bSigma\operatorname{tr}(\bSigma^2).
$$
This implies that $F$ and $\bSigma$ commute because $\bSigma^3$ and $\bSigma$ commute. Moreover, in this case
$$
\frac{\beta^2}{n}(F-\bSigma^3)=\frac{\beta^2}{n}( \bSigma^3+\bSigma\operatorname{tr}(\bSigma^2)).
$$
Therefore, if $n\gg \lambda_1(\lambda^2_1+\operatorname{tr}(\bSigma^2))$, then the eigen-space of $\Hb_{n,\beta}$ will be dominated by $(\Ib-\beta\bSigma)^2\bSigma$.
\paragraph{Discussions on \Cref{ass:higherorder}} \Cref{ass:higherorder} is an eighth moment condition for $\xb:=\bSigma^{\frac{1}{2}}\zb$, where $\zb$ is a $\sigma_x$ sub-Gaussian vector. Given $\beta$, for sufficiently large $n$ s.t. $\mu_i(\Hb_{n,b})>0$, $\forall i$, and if $\operatorname{tr}(\bSigma
^{k})$ are all $O(1)$ for $k=1,\cdots,4$, then by the quadratic form and the sub-Gaussian property, which has finite higher order moments, we can   conclude that $C(\beta,\Sigma)=\Theta(1)$.

The following lemma further shows that if $\mathcal{P}_{\xb}$ is a Gaussian distribution, we can derive the analytical form for $C(\beta,\bSigma
)$.
\begin{lemma}\label{lemma-C}
 Given $|\beta|<\frac{1}{\lambda_1}$, for sufficiently large $n$ s.t. $\mu_i(\Hb_{n,b})>0$, $\forall i$, and if $\mathcal{P}_{\xb}$ is a Gaussian distribution, assuming $\bSigma
 $ is diagonal, we have:
 $$ C(\beta, \bSigma)=210( 1+\frac{\beta^4\operatorname{tr}(\bSigma^2)^2}{(1-\beta\lambda_1)^4}).$$
\end{lemma}
\begin{proof} Let $\eb_i\in\mathbb{R}^{d}$ denote the vector that the $i$-th coordinate is $1$, and all other coordinates equal $0$. For $\xb\sim\mathcal{P}_{\xb}$, denote $\xb\xb^{\top}=[x_{ij}]_{1\leq i,j\leq d}$. Then we have:
  \begin{align*}
     & \mathbb{E}[\|\eb_{i}^{\top}\Hb^{-\frac{1}{2}}_{n,\beta}(\mathbf{I}-\frac{\beta}{n}\Xb^{\top}\Xb)\bSigma (\mathbf{I}-\frac{\beta}{n}\Xb^{\top}\Xb)\Hb^{-\frac{1}{2}}_{n,\beta}\eb_{i}\|^2]\\
      &\leq  \mathbb{E}[\|\eb_{i}^{\top}\Hb^{-\frac{1}{2}}_{n,\beta}(\mathbf{I}-\beta\xb\xb^{\top})\bSigma (\mathbf{I}-\beta\xb\xb^{\top})\Hb^{-\frac{1}{2}}_{n,\beta}\eb_{i}\|^2]\\
      &=\mathbb{E}\left[(\eb_{i}^{\top}\Hb^{-1}_{n,\beta}\eb_{i})^2\left(\sum_{j\neq i }\beta^2\lambda_jx^2_{ij}+ \lambda_i(1-\beta x_{ii})^2\right)^2\right]
  \end{align*}
  For Gaussian distributions, we have $$
  \mathbb{E}[x^2_{ij}x^2_{ik}]=\begin{cases}9\lambda^2_{i}\lambda^2_{j} &  j=k \text{ and } \neq i\\
  105\lambda_{i}^4& i=j=k\\
 3\lambda^{2}_i\lambda_j\lambda_k  & i\neq j\neq k
  \end{cases}$$
  We can further obtain:
  \begin{align*}
      &\mathbb{E}[\|\eb_{i}^{\top}\Hb^{-\frac{1}{2}}_{n,\beta}(\mathbf{I}-\frac{\beta}{n}\Xb^{\top}\Xb)\bSigma (\mathbf{I}-\frac{\beta}{n}\Xb^{\top}\Xb)\Hb^{-\frac{1}{2}}_{n,\beta}\eb_{i}\|^2]\\
      &\leq 105(\eb_{i}^{\top}\Hb^{-1}_{n,\beta}\eb_{i})^2\left(\sum_{j\neq i }\beta^2 \lambda^2_{j}+ (1-\beta \lambda_{i})^2  \right)^2\\
      &\overset{(a)}{\leq} 210(\eb_{i}^{\top}\Hb^{-1}_{n,\beta}\eb_{i})^2[\beta^4\operatorname{tr}(\bSigma^2
      )^2+ (1-\beta \lambda_{i})^4 ]\\
      &\overset{(b)}{\leq} 210[(\eb_{i}^{\top}\Hb^{-1}_{n,\beta}\eb_{i})^2\beta^4\operatorname{tr}(\bSigma^2
      )^2+ 1 ]
  \end{align*}
  where $(a)$ follows from the Cauchy-Schwarz inequality, and $(b)$ follows the fact that $(\eb_{i}^{\top}\Hb^{-1}_{n,\beta}\eb_{i})^2=\frac{1}{[(1-\beta\lambda_i)\lambda_i^2+\frac{\beta^2}{n}(\lambda_i^2+\operatorname{tr}(\bSigma
  ^2)\lambda_i)]^2}\leq 1/(1-\beta\lambda_i)^4$.
  
  Therefore, for any unit $\vb\in\mathbb{R}^{d}$, we have 
  \begin{align*}
     & \mathbb{E}[\|\vb^{\top}\Hb^{-\frac{1}{2}}_{n,\beta}(\mathbf{I}-\frac{\beta}{n}\Xb^{\top}\Xb)\bSigma (\mathbf{I}-\frac{\beta}{n}\Xb^{\top}\Xb)\Hb^{-\frac{1}{2}}_{n,\beta}\vb\|^2]\\
      &\leq \max_{\vb} 210[(\vb^{\top}\Hb^{-1}_{n,\beta}\vb)^2\beta^4\operatorname{tr}(\bSigma^2
      )^2+ 1 ]\leq 210\left( 1+\frac{\beta^4\operatorname{tr}(\bSigma^2)^2}{(1-\beta\lambda_1)^4}\right).
  \end{align*}
\end{proof}

\end{document}